\documentclass{article} 
\usepackage{iclr2027_preprint,times} 

\usepackage{amsmath,amsfonts,bm}

\def\eqref#1{equation~\ref{#1}}

\def\1{\bm{1}}

\DeclareMathAlphabet{\mathsfit}{\encodingdefault}{\sfdefault}{m}{sl}
\SetMathAlphabet{\mathsfit}{bold}{\encodingdefault}{\sfdefault}{bx}{n}

\renewcommand{\eqref}[1]{\textup{(\ref{#1})}} 

\usepackage{hyperref}
\hypersetup{hidelinks,pdftitle={Multi-Reward Hacking: From Aggregation-Induced Drift to Trust-Anchored Defense},pdfauthor={Yushu Li, Jiin Woo, Xiaoxiao Li, Youngsuk Park}}
\usepackage{url}
\usepackage{graphicx}
\usepackage{caption}
\usepackage{booktabs}
\usepackage{array,tabularx}
\usepackage{amsmath,amssymb,amsthm}
\usepackage{multirow}
\usepackage[skins,raster]{tcolorbox}
\usepackage{wrapfig}
\makeatletter
\newcommand{\needspace}[1]{\begingroup\setlength{\dimen@}{#1}\vskip\z@\@plus\dimen@\penalty-100\vskip\z@\@plus-\dimen@\vskip\dimen@\penalty9999\vskip-\dimen@\vskip\z@skip\endgroup}  
\makeatother
\usepackage{pifont}
\usepackage[varqu,varl]{inconsolata} 
\newcommand{\exttfont}{\fontencoding{T1}\fontfamily{zi4}\selectfont}
\definecolor{exprompt}{RGB}{92,92,92}
\definecolor{exref}{RGB}{56,92,152}
\definecolor{exsym}{RGB}{178,46,42}
\definecolor{extag}{RGB}{34,120,68}
\definecolor{exhlbad}{RGB}{251,214,210}
\definecolor{exhlgood}{RGB}{204,236,212}
\definecolor{exhlwarn}{RGB}{253,231,186}
\newtcolorbox{exbox}[2][]{enhanced,colback=#2!3!white,colframe=#2!62!black,colbacktitle=#2!76!black,coltitle=white,fonttitle=\sffamily\bfseries\fontsize{6.5}{7.5}\selectfont,boxrule=0.35pt,titlerule=0pt,arc=1.4pt,left=3pt,right=3pt,top=2pt,bottom=3pt,toptitle=1.2pt,bottomtitle=1.2pt,boxsep=0pt,fontupper=\scriptsize\exttfont,#1}
\newtcolorbox{expromptbox}[1][]{enhanced,colback=exprompt!8!white,colframe=exprompt!8!white,boxrule=0pt,arc=1.4pt,left=4pt,right=4pt,top=3pt,bottom=3pt,boxsep=0pt,fontupper=\scriptsize\fontencoding{T1}\sffamily\selectfont\raggedright,#1}
\newcommand{\exchip}[3]{{\setlength{\fboxsep}{0.6pt}\colorbox{white}{\fontsize{6}{7}\selectfont\sffamily\mdseries\textcolor{#1}{#2}\,\textcolor{black}{#3}}}}

\newcommand{\exhlline}[2]{\leavevmode\rlap{\smash{\color{#1}\rule[-0.28\baselineskip]{\linewidth}{\baselineskip}}}#2}
\newcommand{\exnote}[2]{{\par\vspace{1.5pt}\raggedright\sffamily\itshape\scriptsize\color{#1!72!black}$\triangleright$~#2\par}}

\newtheorem{theorem}{Theorem}

\newtheorem{proposition}{Proposition}
\newtheorem{lemma}{Lemma}
\newtheorem{corollary}{Corollary}
\newtheorem{definition}{Definition}

\newcommand{\Rrel}{R_{\mathrm{P}}} 
\newcommand{\Rexp}{R_{\mathrm{A}}} 
\makeatletter\@ifpackageloaded{xcolor}{}{\usepackage{xcolor}}\makeatother
\usepackage{colortbl}
\makeatletter
\def\tightparagraph{\@startsection{paragraph}{4}{\z@}{0.1ex plus 0.3ex minus .1ex}{-1em}{\normalsize\bf}}
\makeatother
\newcommand{\newtext}[1]{#1} 
\usepackage{xspace}
\makeatletter
\newenvironment{fixedtable}{\par\medskip\noindent\begin{minipage}{\linewidth}\def\@captype{table}\centering}{\end{minipage}\par\medskip}
\makeatother
\newcommand{\TAG}{\textsc{TAG}\xspace}

\newcommand{\YL}[1]{#1}

\title{Multi-Reward Hacking: From Aggregation-Induced Drift to Trust-Anchored Defense}

\author{Yushu Li$^{1}$\thanks{Work done during an internship at Amazon Annapurna Labs.}\quad Jiin Woo$^{2}$\quad Xiaoxiao Li$^{1}$\quad Youngsuk Park$^{2}$ \\[0pt] $^{1}$University of British Columbia\qquad $^{2}$Amazon Annapurna Labs \\[0pt] {\footnotesize\texttt{yushul@student.ubc.ca}\quad\texttt{woojiin@amazon.com}} \\[0pt] {\footnotesize\texttt{xiaoxiao.li@ece.ubc.ca}\quad\texttt{pyoungsu@amazon.com}}} 

\iclrfinalcopy 
\begin{document}

\maketitle

\vspace{-18pt}
\begin{abstract}
\newtext{

\YL{To improve task performance while encouraging desirable response properties, LLM post-training often combines a primary reward (e.g., correctness) with auxiliary rewards (e.g., brevity or efficiency). However, how these rewards are aggregated can critically affect training outcomes.} We show that symmetric aggregation, \YL{ which treats reward channels equally, can induce a failure where aggregate reward improves while primary-task performance degrades, even when individual rewards are meaningful.} \YL{
We investigate this aggregation-induced failure through a one-step multi-reward decision formulation, characterizing when it constitutes multi-reward hacking.
We prove that, under symmetric aggregation, deterministic conditional-mean GRPO converges to auxiliary-favoring behavior when auxiliary rewards are more reliably obtained, regardless of initialization.} We propose Trust-Anchored Gating (TAG), which first elects an anchor channel based on the policy drift induced by optimizing each channel alone, then gates auxiliary credit on anchor satisfaction. Unlike calibration-sensitive static reweighting, TAG structurally eliminates credit for auxiliary-only outcomes while preserving additional credit for joint success. In the same analytical setting, TAG provably reverses this selection, converging to the primary-favoring action from any initialization \YL{when the primary channel is anchored}. Experiments across Qwen and Llama models on mathematics and program synthesis reveal substantial degradation under symmetric, constrained, and normalized aggregation, whereas TAG mitigates this failure. On Qwen3-4B with a brevity reward, TAG achieves 0.581 pass@4 versus 0.220 for additive aggregation, while shortening responses by 44\% relative to correctness-only training.

}
\end{abstract}
\vspace{-10pt}

\section{Introduction}
\label{sec:intro}


\setlength{\intextsep}{2pt}%
\begin{wrapfigure}{r}{0.52\linewidth}
\vspace{-2pt}
\centering\scriptsize
\captionsetup{font=footnotesize,skip=1pt}
\begin{tcbraster}[raster columns=1,raster row skip=1.2pt,raster before skip=0pt,raster after skip=0pt,top=1pt,bottom=1pt]
\begin{expromptbox}[]
{\bfseries\color{exprompt!75!black}Prompt}\enspace{}What is the remainder when $46^{925}$ is divided by $21?$\par
\end{expromptbox}
\begin{exbox}[title={Single Reward\,{\mdseries($\Rrel$ only)}\hfill\hspace*{4pt}\exchip{extag}{\ding{51}}{correct}\,\exchip{exsym}{\ding{55}}{long}\;{\mdseries 205 tok}}]{exref}
We need to find the remainder when 46\char94{}925 is divided\par
by 21. First, note that 46 mod 21 is 46 - 2*21 = 46 - 42\par
= 4. So the problem reduces to finding 4\char94{}925 mod 21.\par
We can use Euler\textquotesingle{}s theorem. Since 4 and 21 are coprime,\par
{\rmfamily\itshape [\,\ldots\,]}\par
Answer: \$\char92{}boxed\char123{}4\char125{}\$\par
\end{exbox}
\begin{exbox}[title={Symmetric Aggregation\,{\mdseries($\Rrel{+}\Rexp$)}\hfill\hspace*{4pt}\exchip{exsym}{\ding{55}}{wrong}\,\exchip{extag}{\ding{51}}{short}\;{\mdseries 8 tok}}]{exsym}
\exhlline{exhlbad}{Answer: \char92{}boxed\char123{}1\char125{}}\par
\end{exbox}
\begin{exbox}[title={Trust-Anchored Gating (TAG)\hfill\hspace*{4pt}\exchip{extag}{\ding{51}}{correct}\,\exchip{extag}{\ding{51}}{short}\;{\mdseries 146 tok}}]{extag}
We need to find 46\char94{}925 mod 21. Since 46 \ensuremath{\equiv} 46-2*21=46-42=4\par
mod 21. So 46\char94{}925 \ensuremath{\equiv} 4\char94{}925 mod 21. Now compute 4\char94{}\ensuremath{\varphi}(21)=4\char94{}12\par
\ensuremath{\equiv} 1 mod 21 (by Euler\textquotesingle{}s theorem). 925 divided by 12: 925 =\par
12*77 + 1. So 4\char94{}925 \ensuremath{\equiv} 4\char94{}1 = 4 mod 21. Answer: \$\char92{}boxed\char123{}4\char125{}\$\par
\end{exbox}
\end{tcbraster}
\caption{\textbf{Multi-reward hacking} (Qwen3-1.7B; $\Rrel$: correctness, \mbox{$\Rexp=\mathbf 1[\text{len}<200]$}: brevity). Summed, the two rewards yield the bare template, short and wrong; \TAG credits $\Rexp$ only when $\Rrel=1$; its response is short and right.}
\label{fig:teaser}
\vspace{-2pt}
\end{wrapfigure}%
\newtext{Post-training of large language models increasingly relies on more than one reward. A verifier certifies correctness while a second signal scores format, length or style \citep{deepseekr1_2025,dapo2025,gdpo2026}.
Unit tests certify a program while a latency term rewards efficiency; rubric dimensions and reward models are stacked in open-ended tasks \citep{armorm2024,gunjal2025rubrics,papo2025}. No single signal captures everything one wants, so several are collected and combined into the one scalar the optimizer sees. In current practice that step is an engineering choice, made by weighted sums \citep{deepseekr1_2025,drw2025}, constraints \citep{moskovitz2023confronting,dai2024saferlhf,cgpo2024} or per-channel normalization \citep{mogrpo2025,gdpo2026}, and tuned per project on the assumption that some choice works. \YL{However, these approaches largely focus on how to combine reward signals, leaving open when and why aggregation systematically favors auxiliary behaviors over the primary objective.} This paper asks the question that this practice leaves open: \YL{Can reward aggregation itself undermine a primary objective, even when that objective is explicitly rewarded?}}

\newtext{It can, and the failure has a characteristic shape (Figure~\ref{fig:teaser}). \YL{Across models, tasks and aggregation rules the auxiliary rate climbs toward one, so the aggregated objective keeps rising, while task correctness collapses (Figures~\ref{fig:teaser} and~\ref{fig:main}). In most of these runs} the policy stops solving the task and locks onto a single behavior that games the auxiliary signal \YL{(Tables~\ref{tab:matrix}--\ref{tab:coding-main}; Appendix~\ref{app:examples})}: it negates its final answer to pass a sign check, shrinks its solution to a bare few-token guess to earn a brevity reward, and emits programs that exit at once and solve nothing to earn a latency reward. Length exploitation of this kind is documented in RLHF \citep{singhal2024long,chen2024odin}, and reasoning models learn to exploit unit tests in agentic coding \citep{baker2025monitoring}. We call this failure \emph{multi-reward hacking} (Definition~\ref{def:hackable}), an instance of reward hacking \citep{amodei2016concrete,skalse2022defining} with a distinctive origin: unlike failures of misspecified or incomplete proxies \citep{pan2022effects,zhuang2020consequences}, it occurs with legitimate channels as well as flawed ones; aggregation itself induces it. \YL{Existing reward hacking studies largely focus on imperfect proxies \citep{gao2023scaling,karwowski2024goodhart}, while detection methods typically assess individual reward signals or behaviors. Neither fully explains how combining rewards can systematically favor an auxiliary objective over a primary one.} \YL{Recent multi-reward methods that rebalance the channels within each group, such as MO-GRPO \citep{mogrpo2025} and PG-OT \citep{ba2026pareto}, remain symmetric in the channels and, as we show, collapse in the same way (Section~\ref{sec:experiments}).}}

\newtext{We first establish a fundamental failure mode of \emph{symmetric reward aggregation}, where no reward channel is given priority over another. Such rules assign equal credit to primary-only and auxiliary-only success, despite differences in how easily these rewards can be obtained. We formalize this in a stochastic decision model that captures a common feature of real-world post-training: primary rewards are \YL{often} harder to obtain than auxiliary rewards. \YL{For example, solving a math problem or passing unit tests requires successful reasoning, whereas producing a short response or a program that exits immediately is comparatively easy.} We prove that, \YL{under this imbalance, every symmetric rule that rewards single-channel success drives the conditional-mean GRPO dynamics toward the auxiliary-favoring behavior, regardless of initialization, even when the policy initially favors the primary objective.} This result reveals that symmetric aggregation can systematically \YL{favor} an easier auxiliary objective at the expense of the primary one.}

The failure motivates breaking the symmetry in reward aggregation. \newtext{\YL{While the analysis assumes a designated primary channel, identifying which reward to prioritize in practice is not always straightforward.} This raises two questions: which channel should receive priority, and how should that priority be enforced? We propose \emph{Trust-Anchored Gating} (\TAG), which first elects an anchor by briefly optimizing each channel in isolation and selecting the one that induces the least policy drift. \YL{The intuition is that exploitable channels may induce large policy shifts toward shortcut behaviors, making isolated drift a useful signal for anchor selection.} \YL{Our experiments show that lowering the auxiliary weight delays the collapse without preventing it, and that a weight which is safe at one model scale fails at another (\S\ref{sec:exp-weights}).} Rather than relying on static reweighting, \TAG enforces priority through a structural gate: auxiliary-only success earns no more credit than failure on both channels, while joint success remains more rewarding than primary-only success. We prove that, with the primary channel as the anchor, \TAG drives the conditional-mean GRPO dynamics to the primary-favoring action from any initialization, for any positive gating weight, whenever primary success remains possible.}

\newtext{We validate these findings on Qwen and Llama models (1.7B--30B), on mathematics and program synthesis, with compatible and incompatible auxiliary rewards. In nearly every setting, previous aggregation trades task correctness for the auxiliary reward, whereas \TAG substantially mitigates this degradation while retaining auxiliary benefits. On Qwen3-4B, for instance, \TAG reaches 0.581 pass@4 under a brevity reward (0.220 for additive aggregation) with responses 44\% shorter than correctness-only training. Our contributions are:}
\vspace{-2pt}
\begin{list}{$\bullet$}{\setlength{\leftmargin}{0pt}\setlength{\itemindent}{1em}\setlength{\labelwidth}{0.6em}\setlength{\labelsep}{0.4em}\setlength{\itemsep}{1pt}\setlength{\parsep}{0pt}\setlength{\topsep}{2pt}\setlength{\partopsep}{0pt}}
\item \newtext{
\emph{A provable failure of symmetric aggregation.}
We formalize multi-reward hacking and prove that, when
auxiliary rewards are more reliably obtained, symmetric
aggregation drives the conditional-mean GRPO dynamics toward
auxiliary-favoring behavior from any initialization.
We further characterize when this selection constitutes
reward hacking under the unknown true reward.}

\item \newtext{
\emph{Trust-anchored aggregation with convergence guarantees.} We propose \TAG, which elects an anchor by isolated policy drift and structurally conditions auxiliary credit on its satisfaction. With the primary channel as anchor, \TAG provably reverses the selection, converging to the primary-favoring behavior from any initialization for any positive gating weight\YL{, while still rewarding joint success, so that the auxiliary objective is pursued where it is compatible with the primary one}.}

\item \newtext{
\emph{Empirical evidence in LLM post-training.}
Across Qwen and Llama models, mathematics and program
synthesis, we demonstrate severe correctness degradation under conventional reward aggregation,
including additive, normalized, and constrained methods. In all evaluated settings, anchor election identifies correctness without prior knowledge of channel roles, and \TAG substantially mitigates this degradation while
retaining the benefits of auxiliary rewards.}
\end{list}
\vspace{-6pt}
\section{Related Work}
\label{sec:related}

\newtext{\textbf{Reward hacking and over-optimization in LLM post-training.} Optimizing a proxy reward beyond the region where it tracks the intended objective is a well-studied failure mode \citep{amodei2016concrete,pan2022effects,zhuang2020consequences}. \citet{skalse2022defining} formalize hackability as a reversal of the policy ordering induced by proxy and true reward, and \citet{laidlaw2025correlated} refine the notion to correlated proxies. \YL{Over-optimization} has been measured along the KL path \citep{gao2023scaling} and in occupancy-measure geometry \citep{karwowski2024goodhart}, and addressed by regularizing the occupancy measure toward the reference policy's \citep{laidlaw2025correlated} or by aligning the policy update direction \citep{deng2026directional}; action-level KL regularization alone fails under heavy-tailed misspecification \citep{kwa2024catastrophic}. Length exploitation is the best-documented instance in RLHF \citep{singhal2024long}\YL{, and grader loopholes let a wrong answer pass the verifier \citep{trace2026}}; ODIN \citep{chen2024odin} separates length and quality heads and optimizes only the latter. These lines concern one proxy against the true reward; here the optimized scalar is itself a combination of channels, and the failure arises from their interaction.}

\newtext{\textbf{Multi-reward and rubric-based RL for LLMs.} Practice combines channels by weighted sums with fixed or learned weights \citep{deepseekr1_2025,drw2025}, by constraints and mixtures of judges \citep{moskovitz2023confronting,dai2024saferlhf,cgpo2024}, by log-sigmoid conjunction of transformed rewards \citep{wang2024transforming}, by Pareto-guided optimal transport\YL{, proposed for text-to-image models,} \citep{ba2026pareto}, by per-channel normalization before summation \citep{mogrpo2025,gdpo2026}, by interpolating policies fine-tuned on separate rewards \citep{rame2023rewarded} or folding several objectives into one preference objective \citep{zhou2024modpo}, by decoupled advantages for rubric integration \citep{papo2025}, by structured reward models \citep{armorm2024}, and by rubric-based rewards \citep{gunjal2025rubrics} with randomly dropped rubric items \citep{rubricdropout2026}. These methods balance or combine reward signals; we ask whether the aggregation rule itself drives optimization toward the auxiliary behavior and how to remove that incentive structurally. Classical multi-objective RL \citep{roijers2013survey,vamplew2008limitations,gabor1998multi,skalse2022lexicographic} treats the specified objectives as legitimate quantities to trade off; here degradation is measured against an unknown true reward, so Pareto optimality among the channels is not sufficient.}

\newtext{\textbf{Reward auditing and reward selection.} Flawed reward behavior is mitigated or audited with reward-model ensembles \citep{coste2023reward} or weight-averaged reward models \citep{warm2024}, which share systematic errors and so mitigate but do not eliminate hacking \citep{eisenstein2023helping}, with adversarial auditing during training \citep{ara2026}, or by monitoring the chain-of-thought during training \citep{baker2025monitoring}. None of these attributes a behavior to a channel. Anchor election sits earlier in the pipeline: it probes each channel alone before joint training and ranks anchor candidates by the drift a channel induces on its own.}
\section{\texorpdfstring{\newtext{Problem Formulation}}{Problem Formulation}}
\label{sec:setting}

\subsection{\texorpdfstring{\newtext{Preliminaries}}{Preliminaries}}
\label{sec:prelim}

\paragraph{Group Relative Policy Optimization (GRPO).}
GRPO \citep{shao2024deepseekmath} replaces the value model of PPO by the mean reward of a group of responses to the same prompt. We write the objective without its clipping and KL-penalty terms. Let $\pi_\theta(\cdot\mid x)$ be the language model as a policy with parameters $\theta$ and $\pi_{\theta_{\mathrm{old}}}$ the policy that generated the current group: for a prompt $x$ and a group $y_1,\ldots,y_n\sim\pi_{\theta_{\mathrm{old}}}(\cdot\mid x)$ with rewards $r_1,\ldots,r_n$ it maximizes
\begin{equation}\label{eq:grpo}
\mathcal J_{\mathrm{GRPO}}(\theta)=\mathbb E\Bigl[\frac1n\sum_{i=1}^{n}\frac{\pi_\theta(y_i\mid x)}{\pi_{\theta_{\mathrm{old}}}(y_i\mid x)}\,A_i^{\mathrm{GRPO}}\Bigr],
\qquad
A_i^{\mathrm{GRPO}}=\frac{r_i-\bar r}{s_r},
\end{equation}
where $\bar r,s_r$ are the group mean and standard deviation. Besides the GRPO advantage, Dr.GRPO \citep{drgrpo2025} introduced an advantage without the standard deviation, $A_i^{\text{Dr.GRPO}}=r_i-\bar r$.

\tightparagraph{Reward hacking.}
\citet{skalse2022defining} define reward hacking as a relation between a proxy reward $R$ and the true reward $R^\star$ on a policy set $\Pi$, where the corresponding expected return with any policy $\pi \in \Pi$ is $J(\pi)=\mathbb E_{x\sim\mathcal D,\,y\sim\pi(\cdot\mid x)}[R(x,y)]$ and $J^\star(\pi):=\mathbb E_{x\sim\mathcal D,\,y\sim\pi(\cdot\mid x)}[R^\star(x,y)]$, respectively.

\begin{definition}[Hackable reward pair]
\label{def:skalse}
A pair of reward functions $(R,R^\star)$ is hackable relative to a policy set $\Pi$ if there exist $\pi,\pi'\in\Pi$ such that
$J(\pi)<J(\pi')$ and $J^\star(\pi)>J^\star(\pi')$; else it is unhackable.
\end{definition}

\subsection{\texorpdfstring{\newtext{Multi-Reward RL for Post-Training}}{Multi-Reward Post-Training}}
\label{sec:setting-cells}

We model response-level post-training as a one-step multi-reward decision
problem: a prompt $x\sim\mathcal D$ is followed by a response
$y\sim\pi_\theta(\cdot\mid x)$ from a policy in the class $\Pi$ of all policies at $x$, which receives the
reward vector $\mathbf R(x,y)=(R_1(x,y),\ldots,R_K(x,y))$ with
$R_k:\mathcal X\times\mathcal Y\to\mathbb R$. An aggregation rule $f:\mathbb R^K\to\mathbb R$ maps the reward vector to a scalar reward as
\begin{equation}\label{eq:rf}
R_f(x,y):=f\bigl(R_1(x,y),\ldots,R_K(x,y)\bigr),
\qquad
J_f(\pi):=\mathbb E_{x\sim\mathcal D,\,y\sim\pi(\cdot\mid x)}[R_f(x,y)].
\end{equation}


The \emph{true reward} $R^\star:\mathcal X\times\mathcal Y\to\mathbb R$ is the criterion against which hacking is judged, with $R_f$ as the proxy of Definition~\ref{def:skalse}. In the analysis we represent it as
$R^\star(x,y)=f^\star(R_1(x,y),\ldots,R_K(x,y))$ with
$f^\star:\mathbb R^K\to\mathbb R$ unknown to training.

\begin{definition}[Multi-reward hacking]
\label{def:hackable}
The aggregate $R_f$ exhibits \emph{multi-reward hacking} on $\Pi$ if for every $R_f$-optimal policy $\pi_f\in\arg\max_{\pi\in\Pi}J_f(\pi)$ some $\pi\in\Pi$ satisfies
\[
J_f(\pi_f) > J_f(\pi),
\qquad
J^\star(\pi_f) < J^\star(\pi):
\]
the proxy prefers its optimum, the true reward prefers the
alternative. Both orderings are strict, as in Definition~\ref{def:skalse}: two $R_f$-optimal policies that differ in $J^\star$ do not by themselves constitute hacking.
\end{definition}
Multi-reward hacking is thus a hackable pair (Definition~\ref{def:skalse}) realized at the $R_f$-optimal policy, and it is a property of the aggregate, not of any individual channel.

We study GRPO and its variants trained on the aggregated reward $R_f$, and MO-GRPO \citep{mogrpo2025}, which standardizes each channel separately: $A_i^{\text{MO-GRPO}}=\sum_{k=1}^K\frac{r_{ik}-\bar r_k}{s_k}$ with $r_{ik}:=R_k(x,y_i)$ and $\bar r_k$, $s_k$ the group mean and standard deviation of channel $k$.

\subsection{\texorpdfstring{\newtext{A Two-Channel MDP for Analysis}}{A Two-Channel MDP for Analysis}}
\label{sec:prompt_dynamics}
\label{sec:cross_prompt}
\label{sec:partial_compatibility}
We analyze the dynamics in a two-channel MDP. Let $K=2$ with two binary channels, a \emph{primary} channel $\Rrel\in\{0,1\}$ and an \emph{auxiliary} channel $\Rexp\in\{0,1\}$, for example $\Rrel$ correctness and $\Rexp$ brevity as in Figure~\ref{fig:teaser}. Fix a prompt $x$ and let the policy be a tabular softmax over three actions $a_0,a_1,a_2$, $\pi_\theta(a_j\mid x)=e^{\theta_j}/\sum_{j'}e^{\theta_{j'}}$ with $\theta\in\mathbb R^3$. Each action $a_j$ has an outcome distribution $P_{a_j}$ on $\{0,1\}^2$, $P_{a_j}(r):=\Pr[(\Rrel,\Rexp)=r\mid a=a_j]$, with parameters $\eta,\rho\in[0,1]$; each action draws its outcome $(\Rrel,\Rexp)$ with the probability written on the corresponding arrow in Figure~\ref{fig:mdp}, which omits $a_0$ since it yields $(0,0)$ with probability one.
Outcomes are independent across the $n$ responses of a group given their actions, $\mathbb E_{a_j}$ denotes expectation under $P_{a_j}$, and at this prompt $J_f(\pi)=\sum_j\pi(a_j\mid x)\,\mathbb E_{a_j}[f(\Rrel,\Rexp)]$, likewise $J^\star$ with $f^\star$; the response $y$ is abstracted away, so $a\sim\pi_\theta(\cdot\mid x)$ followed by $(\Rrel,\Rexp)\sim P_a$ is a one-step MDP.

\paragraph{Training dynamics.}
One on-policy GRPO step with step size $\alpha>0$ changes the logit of $a_j$ by $\Delta\theta_j=\frac{\alpha}{n}\sum_{i:\,a_i=a_j}A_i$ (Appendix~\ref{app:proofs}); we analyze the deterministic \emph{conditional-mean} recursion $\theta_{t+1}=\theta_t+\mathbb E[\Delta\theta\mid\theta_t]$ and write $\to$ for $t\to\infty$.
\vspace{-5pt}

\begin{wrapfigure}{r}{0.42\textwidth}
\centering
\begin{tikzpicture}[>=stealth,
  act/.style={circle,draw,thick,minimum size=6mm,font=\small},
  res/.style={rounded corners=2pt,draw,minimum width=11mm,minimum height=5mm,font=\scriptsize,inner sep=1.5pt},
  p/.style={font=\tiny,fill=white,inner sep=0.8pt},
  tag/.style={font=\scriptsize,text=black!70,align=center}]
\node[act] (a1) at (0,0.8)  {$a_1$};
\node[act] (a2) at (0,-0.4) {$a_2$};
\node[res] (o01) at (3.6,1.2)  {$(0,1)$};
\node[res] (o11) at (3.6,0.4)  {$(1,1)$};
\node[res] (o10) at (3.6,-0.4) {$(1,0)$};
\node[res] (o00) at (3.6,-1.2) {$(0,0)$};
\draw[->] (a1) -- node[p,pos=0.5]{$1-\eta$} (o01.west);
\draw[->] (a1) -- node[p,pos=0.6]{$\eta$} (o11.west);
\draw[->] (a2) -- node[p,pos=0.38]{$\eta$} (o11.west);
\draw[->] (a2) -- node[p,pos=0.5]{$(1-\eta)\rho$} (o10.west);
\draw[->] (a2) -- node[p,pos=0.5]{$(1-\eta)(1-\rho)$} (o00.west);
\node[tag,anchor=south] at (0,1.55)   {Action\\[-1pt]$a_j$};
\node[tag,anchor=south] at (1.7,1.55) {Probability\\[-1pt]$P_{a_j}(r)$};
\node[tag,anchor=south] at (3.6,1.55) {Reward\\[-1pt]$r=(\Rrel,\Rexp)$\\[-2pt]{\tiny e.g.\ (correctness, brevity)}};
\end{tikzpicture}
\captionsetup{font=footnotesize,skip=2pt}
\caption{\textbf{The two-channel MDP.} Outcome distributions of the auxiliary-favoring action $a_1$ and the primary-favoring action $a_2$.}
\label{fig:mdp}
\end{wrapfigure}
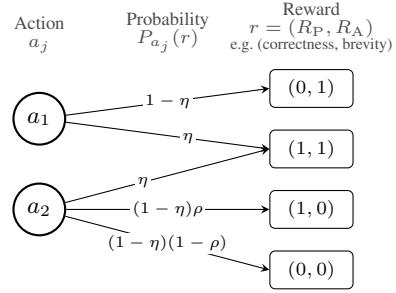
\paragraph{Compatibility $\eta$ and primary success rate $\rho$.}
The two actions differ only through $\eta$ and $\rho$: $a_1$ earns $\Rexp$ with certainty and $\Rrel$ with probability $\eta$, whereas $a_2$ earns $\Rrel$ with probability $\eta+(1-\eta)\rho$ and $\Rexp$ with probability $\eta$. Thus $\eta$ is the \emph{compatibility} of the channels, the chance that a response satisfies both at once: $\eta=0$ is the \emph{incompatible} regime, in which no response satisfies both channels, and $\eta>0$ the \emph{compatible} regime, with $\eta$ near one meaning that a single behavior serves both objectives. $\rho$ is the \emph{primary success rate}, the chance that a response pursuing $\Rrel$ actually earns it  outside the joint-success component, that is, how hard the primary task is.
\vspace{-5pt}

\section{\texorpdfstring{\newtext{Failure of Symmetric Reward Aggregation}}{Failure of Symmetric Reward Aggregation}}
\label{sec:failure}

\subsection{\texorpdfstring{\newtext{Action Selection under Symmetric Aggregation}}{Action Selection under Symmetric Aggregation}}
\label{sec:main-result}

\begin{definition}[Symmetric aggregation]
\label{def:symmetric}
An aggregation rule $f:\mathbb R^K\to\mathbb R$ is \emph{symmetric} if $f(r_{\sigma(1)},\ldots,r_{\sigma(K)})=f(r_1,\ldots,r_K)$ for every permutation $\sigma$ of $\{1,\ldots,K\}$ and every $r\in\mathbb R^K$, so that no channel is singled out; it is \emph{asymmetric} otherwise.
\end{definition}


\begin{theorem}[Symmetric aggregation selects the auxiliary action $a_1$]
\label{thm:overlap}
\label{thm:main-selection}
Consider the MDP of Section~\ref{sec:cross_prompt}. Let $f$ be symmetric and monotone, that is, nondecreasing in each channel and strictly increasing from no satisfied channel to one, $f(0,0)<f(0,1)\le f(1,1)$, and $\eta,\rho\in[0,1)$, $n\ge2$, $\alpha>0$, and $\theta_0\in\mathbb R^3$. Under the deterministic
conditional-mean GRPO recursion,
\begin{equation}\label{eq:thm1}
\pi_{\theta_t}(a_1\mid x)\to1,\qquad
\pi_{\theta_t}(a_0\mid x)\to0,\qquad
\pi_{\theta_t}(a_2\mid x)\to0
\qquad\text{as }t\to\infty,
\end{equation}
regardless of the initial ordering of $\pi_{\theta_0}(a_1\mid x)$ and
$\pi_{\theta_0}(a_2\mid x)$.
\vspace{-6pt}
\end{theorem}

\begin{proof}[Proof sketch]
Outside the shared joint-success component, on which both actions receive $f(1,1)$, $a_1$ receives $f(0,1)$ whereas $a_2$ receives the equal level $f(1,0)$ only with probability $\rho$ and $f(0,0)$ otherwise, so that in expectation
\begin{equation}\label{eq:gap}
\mathbb E_{a_1}[R_f]-\mathbb E_{a_2}[R_f]=(1-\eta)(1-\rho)\,\bigl(f(1,0)-f(0,0)\bigr),
\end{equation}
positive iff $\eta,\rho<1$. Since the standardized GRPO advantage is nondecreasing in the own reward, the conditional mean advantage of $a_1$ dominates that of $a_2$ at every state, strictly when $\eta,\rho<1$, and both dominate $a_0$; the log-odds dynamics then eliminate $a_0$ and drive $a_1$ to dominate $a_2$ (Appendix~\ref{app:partial_compatibility}). Figure~\ref{fig:sym-partial} shows the recursion from an initialization in which the primary action leads the auxiliary one.
\end{proof}
\vspace{-5pt}

\begin{corollary}[Incompatible channels with $\rho<1$]
\label{cor:reliability}
\label{thm:reliability}
If $\eta=0$ and $\rho<1$, then under the assumptions of
Theorem~\ref{thm:overlap}, $\pi_{\theta_t}(a_1\mid x)\to1$ as $t\to\infty$ for
every $\theta_0\in\mathbb R^3$.  
\end{corollary}

\vspace{-5pt}


\subsection{\texorpdfstring{\newtext{When Symmetric Aggregation Hacks the True Reward}}{When Symmetric Aggregation Hacks the True Reward}}
\label{sec:interpretation}

\tightparagraph{Reward hacking conditions for various true objectives.}
By Eq.~\eqref{eq:gap} and $f(0,0)<f(0,1)$, for a symmetric, monotone $R_f$ with $\eta,\rho<1$ the unique optimal policy $\pi_f$ plays $a_1$, and by Theorem~\ref{thm:overlap}
the conditional-mean recursion converges to it. Let $f^\star$ be
nondecreasing in each channel, so that $R^\star$ never prefers the null
action; $J^\star$ is then maximized at $a_1$ or at $a_2$, and
Definition~\ref{def:hackable} reduces to comparing $\pi_f$ with the policy
$\pi_{a_2}$ that plays $a_2$: since $J_f(\pi_f)>J_f(\pi_{a_2})$ by
Eq.~\eqref{eq:gap}, $R_f$ exhibits multi-reward hacking exactly when
$J^\star(\pi_f)<J^\star(\pi_{a_2})$.
For two representative true rewards, the primary channel alone, as in lexicographic multi-objective RL \citep{skalse2022lexicographic}, and a linear scalarization of the two channels, the standard multi-objective and multi-reward practice \citep{roijers2013survey,hayes2022practical,rame2023rewarded,zhou2024modpo}, the outcome distributions $P_{a_j}$ of Section~\ref{sec:prompt_dynamics} give:
\begin{center}
\begin{tabular}{l c l}
\toprule
True reward $R^\star$ & $J^\star(\pi_{a_2})-J^\star(\pi_f)$ & Hacking iff ($\eta,\rho<1$) \\
\midrule
$\Rrel$ & $(1-\eta)\rho$ & $\rho>0$ \\
$w_{\mathrm P}\Rrel+w_{\mathrm A}\Rexp$, $w_{\mathrm P}>w_{\mathrm A}\ge0$ & $(1-\eta)(\rho\,w_{\mathrm P}-w_{\mathrm A})$ & $\rho\,w_{\mathrm P}>w_{\mathrm A}$ \\
\bottomrule
\end{tabular}
\end{center}
The right panel of Figure~\ref{fig:sym-partial} traces $J^\star(\pi_{\theta_t})$ for the second row with $(w_{\mathrm P},w_{\mathrm A})=(0.9,0.1)$: the true reward first rises with $\pi_{\theta_t}(a_2\mid x)$ and then falls towards $\mathbb E_{a_1}[R^\star]=0.19$, below the $0.83$ that $\pi_{a_2}$ attains (Definition~\ref{def:hackable}).

\tightparagraph{Why symmetric rules fail.}
A symmetric rule reinforces \YL{the auxiliary-favoring action $a_1$ faster than the primary-favoring action $a_2$}, and on-policy sampling compounds its lead, an on-policy analogue of shortcut learning \citep{geirhos2020shortcut}. The policy therefore collapses onto the easier channel, which in our settings is never correctness, the harder and slower criterion; whenever the true reward prefers the primary action (the table above), the collapse is reward hacking. The normalized and constrained variants of Section~\ref{sec:experiments} fare no better, so a remedy that keeps the auxiliary signal has to break the symmetry (Section~\ref{sec:method}).

\begin{figure}[t]
\centering
\includegraphics[width=\linewidth]{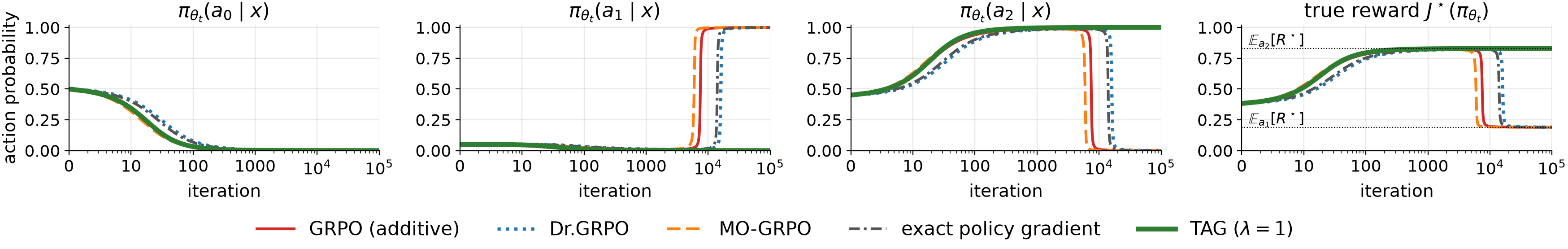}
\vspace{-10pt}
\caption{\textbf{Symmetric aggregation selects \YL{the auxiliary-favoring action $a_1$ over the primary-favoring action $a_2$.}} Exact conditional-mean recursions ($n=8$, $\alpha=0.1$, $\rho=0.9$, $\eta=0.1$) from $\pi_{\theta_0}=(0.5,0.05,0.45)$. GRPO, Dr.GRPO and exact policy gradient first raise $\pi_{\theta_t}(a_2\mid x)$ and then hand the policy to $a_1$, as does MO-GRPO (Appendix~\ref{app:partial_compatibility}); \TAG keeps $a_2$. Right: for an assumed true reward that favors the primary channel, $R^\star=0.9\Rrel+0.1\Rexp$, the symmetric rules first gain and then lose true reward as $a_1$ takes over, while \TAG retains it. }
\label{fig:sym-partial}
\vspace{-12pt}
\end{figure} 
\section{\texorpdfstring{\newtext{Breaking the Symmetry with Trust-Anchored Gating}}{Breaking the Symmetry with Trust-Anchored Gating}}
\label{sec:method}


The preceding analysis establishes a fundamental failure mode of symmetric aggregation: when auxiliary rewards are easier to obtain, GRPO converges to the auxiliary-favoring behavior, even when the primary objective is preferred (Theorem~\ref{thm:overlap}). Addressing this failure raises two questions: which reward channel should take priority, and how should that priority be enforced? We propose \emph{Trust-Anchored Gating} (\TAG), which first elects an anchor based on the policy drift induced by optimizing each channel independently (\S\ref{sec:pre}), then conditions auxiliary credit on anchor satisfaction (\S\ref{sec:tag}). We prove that, when the primary channel is selected as the anchor, \TAG reverses the selection and guarantees convergence to the primary-favoring action from any initialization (Theorem~\ref{thm:tag-conflict}).

\subsection{Anchor Election: Which Channel to Trust}
\label{sec:pre}

Preventing the collapse of Theorem~\ref{thm:overlap} means singling out one channel, and existing detectors cannot attribute a behavior to a channel.

We therefore elect the anchor before joint training with multiple rewards: starting from the same base policy $\pi_0$, each reward channel $R_k$ is trained with GRPO alone for $T$ steps, producing a probe policy $\pi_{T,k}$ that is used only for the election, after which joint training restarts from $\pi_0$; this costs one short probe run per channel. We measure the movement of each probe by its per-token reverse KL from $\pi_0$ \citep{minka2005divergence}:

\vspace{-6pt}
\begin{equation}\label{eq:drev}
\mathrm{KL}^{\mathrm{rev}}_{T,k}
=
\frac{\mathbb E_{x\sim\mathcal D}\bigl[D_{\mathrm{KL}}\bigl(\pi_{T,k}(\cdot\mid x)\,\big\Vert\,\pi_0(\cdot\mid x)\bigr)\bigr]}{\mathbb E_{x\sim\mathcal D,\;y\sim\pi_{T,k}(\cdot\mid x)}\bigl[\,|y|\,\bigr]},
\end{equation}
the sequence-level reverse KL per generated token, estimated on held-out probe prompts as the mean of the per-token log-ratios pooled over all generated tokens; the normalization keeps probes with different response lengths comparable: an auxiliary probe may end at nine tokens where a correctness probe keeps hundreds. 

The reward channel with the smallest $\mathrm{KL}^{\mathrm{rev}}_{T,k}$ is elected as the anchor, written $R_{k^\star}$: the channel that moves the policy least when optimized alone is the one least likely to be optimized through an invented mode (Appendix~\ref{app:pre_details}).

\tightparagraph{Reverse KL detects mode invention.}
\begin{wrapfigure}{r}{0.47\linewidth}
\vspace{-4pt}
\centering
\includegraphics[width=\linewidth]{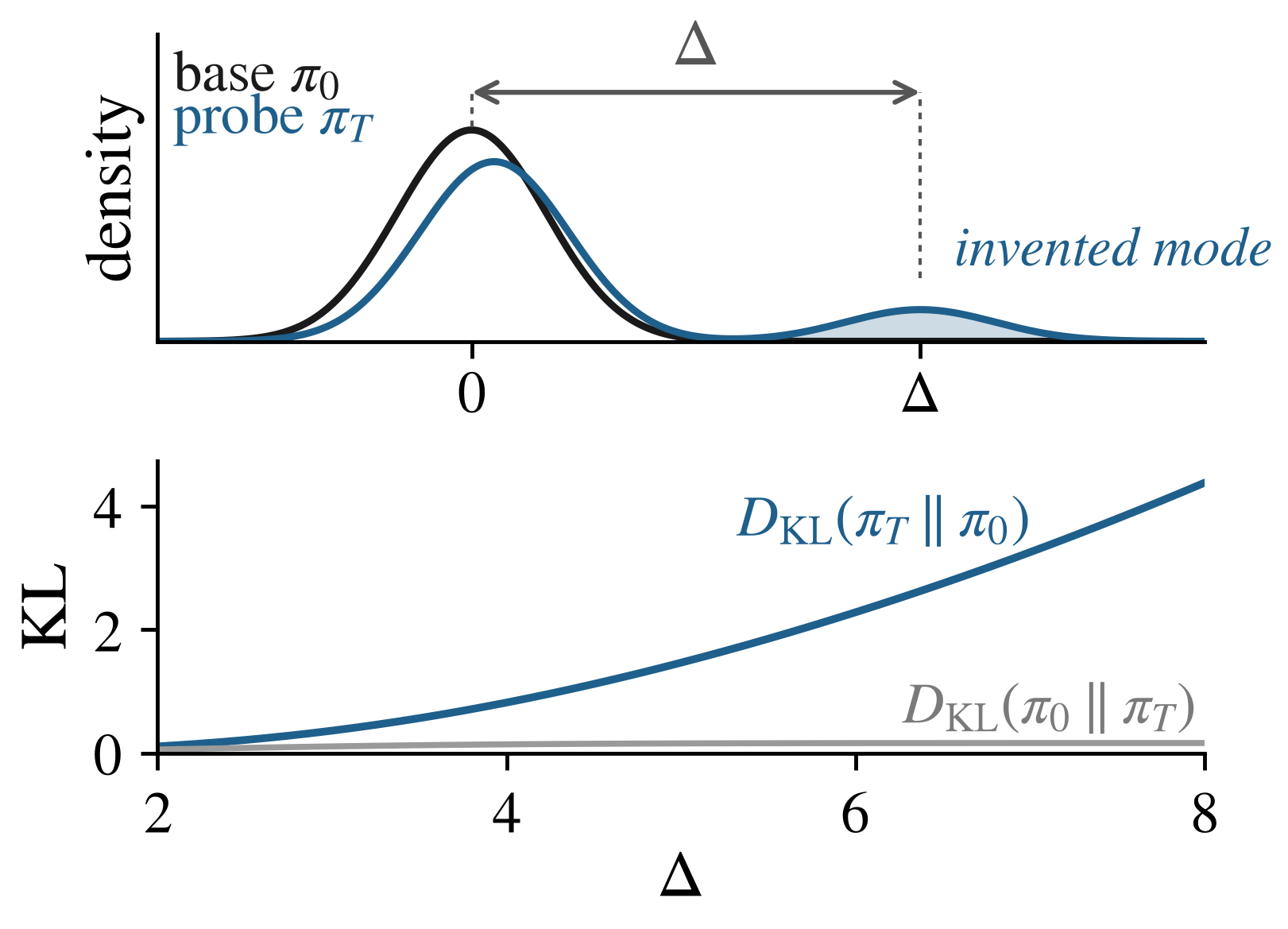}
\captionsetup{font=footnotesize,skip=2pt}
\caption{\textbf{The reverse KL sees how far an invented mode lies.} Top: base $\pi_0$ and a probe $\pi_T$ that adds a mode at distance $\Delta$. Bottom: $D_{\mathrm{KL}}(\pi_T\Vert\pi_0)$ grows with $\Delta^2$; $D_{\mathrm{KL}}(\pi_0\Vert\pi_T)$ stays at the added mass (Lemma~\ref{lem:kl-direction}).}
\label{fig:forwardvsreverse}
\end{wrapfigure}%

The hacking behaviors we observe exist under the base policy with very small probability: a response that always returns one hard-coded answer is almost never sampled from $\pi_0$ and looks nothing like its ordinary samples. Once such a behavior receives positive feedback from an auxiliary reward its probability grows sharply into an atypical new mode, and that mode comes to dominate the policy, so the statistic that elects the anchor has to be sensitive to mode invention. We therefore adopt the reverse KL, which measures the divergence under the probe policy $\pi_{T,k}$: it samples the invented mode and charges it its log-ratio against the base, which grows the further that mode lies from the modes of the base policy. The forward KL measures the divergence under the base policy, whose samples never visit the invented mode, so it registers that the policy has moved but not how far (Lemma~\ref{lem:kl-direction}). A Gaussian pair shows how sensitively the two directions register a new mode (Figure~\ref{fig:forwardvsreverse}): a probe policy invents a mode at distance $\Delta$ from the mode of the base policy $\pi_0$, and as $\Delta$ grows the forward KL saturates while the reverse KL keeps growing.

\vspace{-8pt}

\subsection{Trust-Anchored Gating: Conditioning Credit on the Anchor}
\label{sec:tag}

With the anchor $R_{k^\star}$ elected and the other channels $k\neq k^\star$ auxiliary (we assume that at least one channel admits no low-base-probability shortcut; Appendix~\ref{app:pre_details}), we need an aggregation $f$ that credits the auxiliary channels without reopening the selection of Section~\ref{sec:failure}.

\tightparagraph{Static reweighting is calibration-sensitive.}\label{sec:reweight} The simplest asymmetric aggregation is to linearly combine the channels with a smaller weight on auxiliary channels, $f_{\mathrm{weighted}}(R_1,\ldots,R_K)=w_{k^\star}R_{k^\star}+\sum_{k\neq k^\star}w_kR_k$ with $0\le w_k<w_{k^\star}$. Under GRPO no single weight is safe across a prompt population that contains hard prompts (small $\rho$): on the harder prompts, where a group realizes no primary-successful response, its rewards take only two levels and the group standardization erases the auxiliary weight, so an auxiliary-only response earns the full two-level advantage however small that weight is (Appendix~\ref{app:reliability}, Table~\ref{tab:weight-thresholds}). Sweeping the weight of the two channels bears this out: a smaller weight postpones the collapse without preventing it (Figure~\ref{fig:reweight-curves}).

\vspace{-6pt}

\paragraph{Anchor-gated aggregation.}
Since no weight is safe for every $\rho$, we gate auxiliary credit on the anchor instead; the election of \S\ref{sec:pre} followed by this gate is our method, \emph{Trust-Anchored Gating} (\TAG). Its aggregation function is
\vspace{-6pt}
\begin{equation}
\label{eq:tag}
f_{\TAG}(R_1,\ldots,R_K)
=
R_{k^\star}\,
\frac{1+\sum_{k\neq k^\star}\lambda_kR_k}{1+\sum_{k\neq k^\star}\lambda_k},
\qquad\lambda_k\ge0,
\end{equation}
with training reward $R_{\TAG}:=R_{f_{\TAG}}$.
Under \TAG, auxiliary reward cannot raise the training signal when the anchor is unsatisfied, while joint success remains distinguishable from primary-only success. \YL{Training on the anchor channel $R_{k^\star}$ alone avoids auxiliary-induced interference but sacrifices the potential benefits of auxiliary rewards, such as brevity and efficiency. By rewarding joint success over primary-only success, \TAG preserves these incentives while prioritizing the anchor. Empirically, \TAG shortens mathematical solutions by 33--44\% and reduces program latency while largely preserving correctness (\S\ref{sec:exp-math}--\ref{sec:exp-code}).}


In the two-channel model of Section~\ref{sec:cross_prompt} we take the anchor to be the primary channel, $R_{k^\star}=\Rrel$, and there is a single gating weight, $\lambda:=\lambda_{\mathrm A}>0$, so that $f_{\TAG}(\Rrel,\Rexp)=\Rrel\,(1+\lambda\Rexp)/(1+\lambda)$: $f_{\TAG}(0,0)=f_{\TAG}(0,1)=0<f_{\TAG}(1,0)=1/(1+\lambda)<f_{\TAG}(1,1)=1$.
\begin{theorem}[\TAG selects the primary action $a_2$]
\label{thm:tag-conflict}
Consider the two-channel reward MDP in Section~\ref{sec:cross_prompt}. Under the \TAG reward with $\lambda>0$ and the deterministic conditional-mean
GRPO recursion, let $\eta\in[0,1)$, $\rho\in(0,1]$, $n\ge2$, $\alpha>0$, and
let the initial logits be finite. Then
\begin{equation}\label{eq:thm2}
\pi_{\theta_t}(a_2\mid x)\to1,\qquad
\pi_{\theta_t}(a_0\mid x)\to0,\qquad
\pi_{\theta_t}(a_1\mid x)\to0
\qquad\text{as }t\to\infty.
\end{equation}
\vspace{-15pt}
\end{theorem}

\emph{Proof sketch.} Under $f_{\TAG}$ an auxiliary-only response scores the same as one that satisfies neither channel, so its group advantage is never positive, while primary success keeps a strictly positive expected advantage for every $\rho>0$ and $\lambda>0$; the logit of $a_2$ therefore rises monotonically and the recursion converges to $\pi(a_2\mid x)=1$ (full proof in Appendix~\ref{app:tag_proofs}).

\section{Experiments}
\label{sec:experiments}

\subsection{Setup}
\label{sec:exp-setup}

\vspace{-5pt}

\tightparagraph{Reward channels.}\newtext{We test the failure of \S\ref{sec:failure} and the defense of \S\ref{sec:method} in RL post-training with two reward channels whose compatibility is known by construction. The primary channel $\Rrel$ is always correctness, a gold-answer check in mathematics and unit tests in program synthesis; the auxiliary channel $\Rexp$ is a synthetic signal. In \emph{compatible} settings some correct responses show the auxiliary behavior ($\eta>0$ in \S\ref{sec:prompt_dynamics}); in \emph{incompatible} settings none does ($\eta=0$) and the auxiliary credit is a grader loophole in the sense of \citet{trace2026}. Settings are named by the pair of rewarded behaviors: Correct \& brief and Correct \& negative on mathematics (\S\ref{sec:exp-math}), Correct \& fast and Correct \& no-op on program synthesis (\S\ref{sec:exp-code}), with $\Rexp$ defined in those sections.}
\vspace{-3pt}

\tightparagraph{Models.}\newtext{The main grid trains Qwen3-1.7B and Qwen3-4B-Instruct-2507 (Qwen3-4B below) \citep{qwen3report2025}; \S\ref{sec:exp-family} adds Llama-3.2-3B, Llama-3.1-8B \citep{llama3herd2024} and Qwen3-30B-A3B.}
\vspace{-3pt}

\tightparagraph{Data and evaluation.}\newtext{Mathematics arms train on the Big-Math hard split \citep{bigmath2025,trace2026}; program-synthesis arms train on DeepCoder-RT, a runtime-filtered pool built from DeepCoder \citep{deepcoder2025} (\S\ref{sec:exp-code}, Appendix~\ref{app:exp-coding}). We report held-out pass@4 and the rate of the auxiliary behavior at the end of training and for the base model. An arm has \emph{collapsed} when the auxiliary behavior takes over and its held-out correctness falls far below the $\Rrel$-only reference; shaded table cells mark pass@4 below the base model. Appendix~\ref{app:exp-details} (Table~\ref{tab:hyper}) lists the configuration.}
\vspace{-3pt}

\tightparagraph{Baselines.}\newtext{All arms start from the same base model with the same data and training budget and differ only in how the two channels enter training. Training on $\Rrel$ alone is the reference. Reward-level rules shape a scalar reward and train it with standard GRPO: the additive rule $(\Rrel{+}\Rexp)/2$ (labeled $\Rrel{+}\Rexp$ in the figures; the scale is immaterial under the group z-score), the weighted sum $0.9\Rrel{+}0.1\Rexp$, the AND rule $\min(\Rrel,\Rexp)$, a Lagrangian port of \citet{moskovitz2023confronting} and a reward-level port of PG-OT \citep{ba2026pareto}. Advantage-level estimators keep the additive reward and change the advantage: MO-GRPO \citep{mogrpo2025} sums per-channel group z-scores, and Dr.GRPO \citep{drgrpo2025} drops the group-standard-deviation of the GRPO loss. \TAG uses Eq.~\eqref{eq:tag} with $\lambda=1$ and the anchor elected by the probes of \S\ref{sec:pre}.}
\vspace{-8pt}

\subsection{Aggregation-Induced Degradation and \TAG}
\label{sec:exp-math}

\newtext{\textbf{Settings.} $\Rrel=1$ iff the final answer matches the gold answer. \emph{Correct \& brief}: $\Rexp=\mathbf 1[\,|y|<200\text{ tokens}\,]$, a length target a practitioner may add; short correct solutions exist ($\eta>0$). \emph{Correct \& negative}: $\Rexp=\mathbf 1[\text{final answer}<0]$, the grader flaw of \citet{trace2026}; gold answers are negative on about $5\%$ of problems ($\eta\approx0$).}

\newtext{\textbf{Compatible channels: compatibility does not prevent collapse.} Theorem~\ref{thm:overlap} predicts that symmetric aggregation selects the auxiliary behavior even when joint success is reachable. Correct \& brief bears this out: although short, correct solutions exist, conventional aggregation methods converge to a bare answer template (Figure~\ref{fig:teaser}) that earns the brevity reward without solving the task, so pass@4 falls to $0.14$--$0.22$ and the correct-and-short rate is only the chance-level correctness of a bare guess (Table~\ref{tab:matrix}). \TAG instead conditions brevity credit on correctness and shortens the solutions the policy already gets right: on Qwen3-1.7B and Qwen3-4B it reaches pass@4 of $0.453$ and $0.581$ against $0.481$ and $0.634$ for the correctness-only reference, cuts response length by $33\%$ and $44\%$, and attains the highest correct-and-short rate of any arm, so a compatible auxiliary reward is realized at a cost of $3$--$5$\,pp in pass@4. Real outputs are shown in Appendix~\ref{app:examples}; the training curves of Figure~\ref{fig:main} show the takeover in every symmetric arm and its absence under \TAG.}

\newtext{\textbf{Incompatible channels: auxiliary reward takeover.} Under incompatible rewards ($\eta=0$), Corollary~\ref{cor:reliability} predicts that monotone symmetric aggregation rewarding single-channel success inevitably selects the auxiliary-favoring action. Correct \& negative bears this out: GRPO, Dr.GRPO and MO-GRPO all drive the negative-answer rate to nearly $100\%$ while pass@4 falls to $0.02$--$0.05$, below the base model, against $0.48$ and $0.63$ under $\Rrel$-only training (Table~\ref{tab:matrix}). The AND rule lies outside the theorem because it does not reward auxiliary-only success, yet it fares no better: at 1.7B it is rewarded only on the $5\%$ of prompts with a negative gold answer and stalls at $0.39$, and at 4B it collapses as well; asymmetric reweighting avoids the collapse at 4B but not at 1.7B (Figure~\ref{fig:reweight-curves}; below). \TAG removes the credit for auxiliary-only outcomes (Theorem~\ref{thm:tag-conflict}): its negative-answer rate stays at the $\Rrel$-only level and its pass@4 matches or exceeds the reference.}

\begin{table}[!t]
\centering
\footnotesize
\setlength{\tabcolsep}{3pt}
\resizebox{\linewidth}{!}{%
\begin{tabular}{l l ccc cc ccc cc}
\toprule
 & & \multicolumn{5}{c}{Qwen3-1.7B} & \multicolumn{5}{c}{Qwen3-4B} \\
\cmidrule(lr){3-7}\cmidrule(lr){8-12}
 & & \multicolumn{3}{c}{Correct \& brief} & \multicolumn{2}{c}{Correct \& negative} & \multicolumn{3}{c}{Correct \& brief} & \multicolumn{2}{c}{Correct \& negative} \\
\cmidrule(lr){3-5}\cmidrule(lr){6-7}\cmidrule(lr){8-10}\cmidrule(lr){11-12}
Advantage & Aggregation $f$ & pass@4$\uparrow$ & len$\downarrow$ & both$\uparrow$ & pass@4$\uparrow$ & neg.\ rate$\downarrow$ & pass@4$\uparrow$ & len$\downarrow$ & both$\uparrow$ & pass@4$\uparrow$ & neg.\ rate$\downarrow$ \\
\midrule
\multicolumn{1}{l}{Base model }& /& 0.140 & 479 & 0.006 & 0.140 & 0.007 & 0.076 & 497 & 0.003 & 0.076 & 0.005 \\
\midrule
\multirow{6}{*}{GRPO} & $\Rrel$-only & \textbf{0.481} & 328 & 0.043 & \underline{0.481} & 0.036 & \textbf{0.634} & 310 & 0.118 & \textbf{0.634} & 0.039 \\
 & $(\Rrel{+}\Rexp)/2$ & \cellcolor{gray!15}0.138 & \cellcolor{gray!15}9 & \cellcolor{gray!15}0.119 & \cellcolor{gray!15}0.026 & \cellcolor{gray!15}0.999 & 0.220 & 9 & \underline{0.207} & \cellcolor{gray!15}0.041 & \cellcolor{gray!15}1.000 \\
 & $0.9\Rrel{+}0.1\Rexp$ & 0.202 & 40 & \underline{0.154} & \cellcolor{gray!15}0.040 & \cellcolor{gray!15}0.999 & 0.213 & 9 & 0.200 & 0.616 & 0.041 \\
 & $\min(\Rrel,\Rexp)$ & 0.142 & 9 & 0.125 & 0.387 & 0.039 & 0.219 & 9 & 0.205 & \cellcolor{gray!15}0.051 & \cellcolor{gray!15}0.761 \\
 & Lagrangian & 0.140 & 9 & 0.127 & \cellcolor{gray!15}0.035 & \cellcolor{gray!15}0.999 & 0.222 & 9 & 0.206 & \cellcolor{gray!15}0.040 & \cellcolor{gray!15}0.919 \\
 & PG-OT & 0.151 & 9 & 0.137 & \cellcolor{gray!15}0.032 & \cellcolor{gray!15}0.999 & 0.203 & 9 & 0.194 & \cellcolor{gray!15}0.040 & \cellcolor{gray!15}1.000 \\
\midrule
Dr.GRPO & $(\Rrel{+}\Rexp)/2$ & 0.174 & 29 & 0.144 & \cellcolor{gray!15}0.036 & \cellcolor{gray!15}1.000 & 0.212 & 9 & 0.196 & \cellcolor{gray!15}0.039 & \cellcolor{gray!15}1.000 \\
MO-GRPO & -- & 0.151 & 9 & 0.137 & \cellcolor{gray!15}0.033 & \cellcolor{gray!15}0.999 & 0.217 & 9 & 0.202 & \cellcolor{gray!15}0.021 & \cellcolor{gray!15}1.000 \\
\midrule
GRPO & \TAG\ (ours) & \underline{0.453} & 220 & \textbf{0.164} & \textbf{0.494} & 0.033 & \underline{0.581} & 175 & \textbf{0.332} & \underline{0.632} & 0.041 \\
\bottomrule
\end{tabular}
}
\caption{\textbf{Main grid (Qwen3-1.7B and Qwen3-4B).} Held-out pass@4 and the rewarded behavior at the end of training: under Correct \& brief the mean response length (len, tokens) and the rate of correct-and-short responses (both); under Correct \& negative the rate of negative final answers. Shaded: pass@4 below the base model. Best bold, second underlined. Auxiliary rates: Table~\ref{tab:matrix-raw}.}
\label{tab:matrix}
\vspace{-4pt}
\end{table}

\begin{figure}[!tb]
\centering
\includegraphics[width=\linewidth]{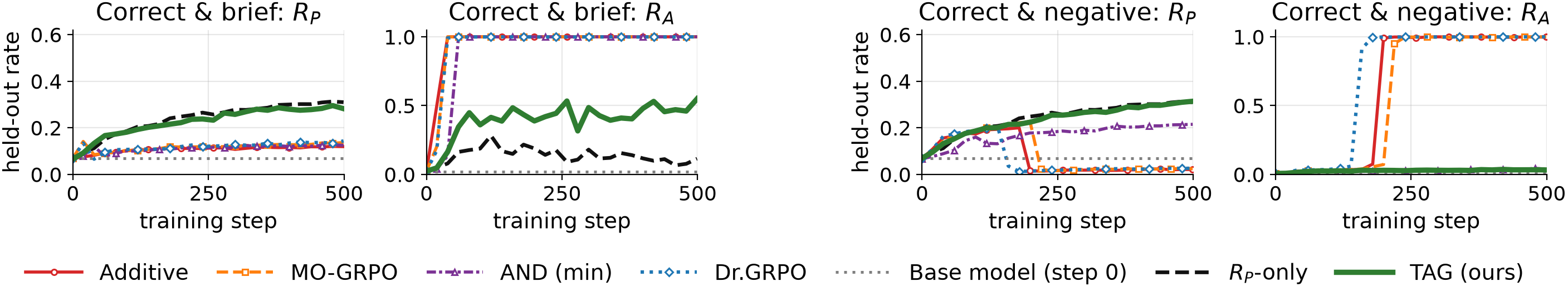}
\caption{\textbf{The auxiliary rate rises first and correctness follows it down.} Held-out $\Rrel$ and $\Rexp$ during training on Qwen3-1.7B under Correct \& brief and Correct \& negative; arms as in Table~\ref{tab:matrix}, dotted line the base model. Most tested methods exhibit auxiliary takeover, with exceptions depending on the reward rule and setting; under \TAG the auxiliary rate rises only as far as correct responses carry it. Qwen3-4B: Figure~\ref{fig:main-4b}.}
\label{fig:main}
\vspace{-10pt}
\end{figure}

\begin{wrapfigure}{r}{0.52\linewidth}
\vspace{-5pt}
\centering
\includegraphics[width=\linewidth]{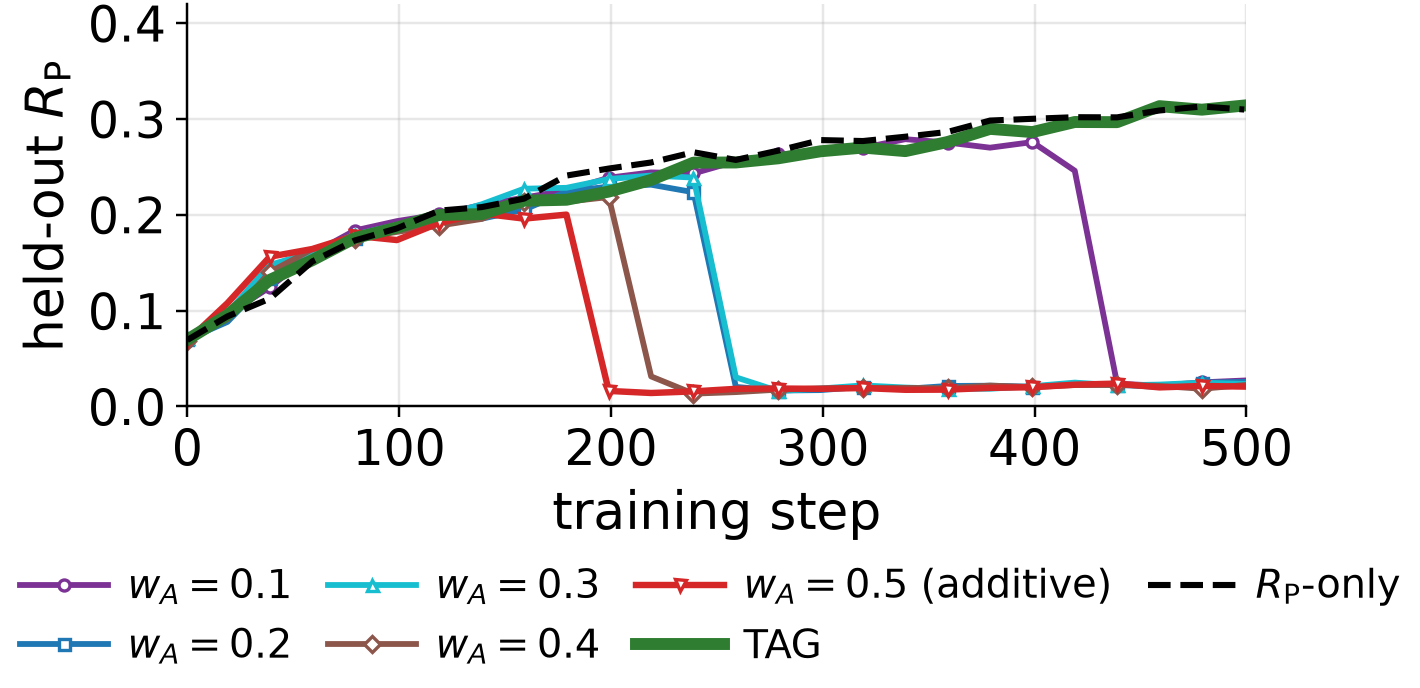}
\captionsetup{font=footnotesize,skip=2pt}
\caption{\textbf{Reweighting only delays the takeover.} Held-out $\Rrel$ for $w_A\in\{0.1,\dots,0.5\}$; \TAG and $\Rrel$-only for reference.}
\label{fig:reweight-curves}
\end{wrapfigure}%

\newtext{\textbf{Static reweighting is not a reliable defense.}\label{sec:exp-weights} Reducing the auxiliary reward weight can delay collapse, but does not reliably prevent it. On Qwen3-1.7B under Correct \& negative, all tested weights eventually lead to nearly $100\%$ negative answers and a primary reward of only $0.02$--$0.03$, compared with $0.31$ for \TAG and correctness-only training (Figure~\ref{fig:reweight-curves}). These results confirm the sensitivity of static reweighting to the training setting, motivating a defense that does not rely on carefully calibrated reward weights.}

\subsection{\texorpdfstring{Program Synthesis}{Program Synthesis}}
\label{sec:exp-code}

\newtext{\textbf{Settings.} We build \emph{DeepCoder-RT} from DeepCoder \citep{deepcoder2025}: every problem receives one additional large-input test, and a problem is kept only if its human reference solution passes all tests and needs at least $10$\,ms of CPU time on the large input (3{,}741 training and 185 held-out problems; Appendix~\ref{app:exp-coding}). $\Rrel=1$ iff the program passes all tests. The running time $T$ of a program is the largest CPU time over the tests it executes, the large input included; a program that times out counts as slow. \emph{Correct \& fast}: $\Rexp=\mathbf 1[\,T<60\,\text{ms}\,]$; fast correct programs exist ($\eta>0$). \emph{Correct \& no-op}: $\Rexp=\mathbf 1[\,T<3\,\text{ms}\,]$ on the same pool, on which we measure that no correct program finishes within $3$\,ms ($\eta=0$). At evaluation, speed is scored relative to each problem's reference solution, timed under the same judge: a program counts as fast if it is correct and runs faster than the reference.}

\newtext{\textbf{Results.} Every symmetric arm collapses below the base model in both settings and at both sizes: additive GRPO, Dr.GRPO and MO-GRPO trade accuracy for speed under Correct \& fast, emitting fast wrong programs, and converge to programs that exit at once under Correct \& no-op, where pass@4 falls to $0.00$--$0.12$, as Corollary~\ref{cor:reliability} predicts (Table~\ref{tab:coding-main}). \TAG is the only arm that keeps the pass@4 of the $\Rrel$-only reference in both settings, as Theorem~\ref{thm:tag-conflict} predicts, and under Correct \& fast it also runs faster (mean latency $307$ vs.\ $359$\,ms on Qwen3-1.7B and $434$ vs.\ $486$\,ms on Qwen3-4B); on Qwen3-4B its rate of correct programs faster than the reference solution is the highest of any arm ($0.215$ vs.\ $0.164$ for the $\Rrel$-only arm; Appendix~\ref{app:exp-coding}).}

\begin{table}[!t]
\centering
\footnotesize
\setlength{\tabcolsep}{3pt}
\resizebox{\linewidth}{!}{%
\begin{tabular}{l l ccc cc ccc cc}
\toprule
 & & \multicolumn{5}{c}{Qwen3-1.7B} & \multicolumn{5}{c}{Qwen3-4B} \\
\cmidrule(lr){3-7}\cmidrule(lr){8-12}
 & & \multicolumn{3}{c}{Correct \& fast} & \multicolumn{2}{c}{Correct \& no-op} & \multicolumn{3}{c}{Correct \& fast} & \multicolumn{2}{c}{Correct \& no-op} \\
\cmidrule(lr){3-5}\cmidrule(lr){6-7}\cmidrule(lr){8-10}\cmidrule(lr){11-12}
Advantage & Aggregation $f$ & pass@4$\uparrow$ & lat.$\downarrow$ & both$\uparrow$ & pass@4$\uparrow$ & rate$\downarrow$ & pass@4$\uparrow$ & lat.$\downarrow$ & both$\uparrow$ & pass@4$\uparrow$ & rate$\downarrow$ \\
\midrule
\multicolumn{1}{l}{Base model }& /& 0.176 & 137 & 0.059 & 0.176 & 0.641 & 0.459 & 383 & 0.113 & 0.459 & 0.297 \\
\midrule
\multirow{2}{*}{GRPO} & $\Rrel$-only & \underline{0.325} & 359 & \textbf{0.103} & \underline{0.325} & 0.318 & \underline{0.593} & 486 & 0.164 & \textbf{0.593} & 0.255 \\
 & $(\Rrel{+}\Rexp)/2$ & \cellcolor{gray!15}0.096 & \cellcolor{gray!15}12 & \cellcolor{gray!15}0.043 & \cellcolor{gray!15}0.000 & \cellcolor{gray!15}0.984 & \cellcolor{gray!15}0.350 & \cellcolor{gray!15}63 & \cellcolor{gray!15}0.201 & \cellcolor{gray!15}0.108 & \cellcolor{gray!15}0.914 \\
\midrule
Dr.GRPO & $(\Rrel{+}\Rexp)/2$ & \cellcolor{gray!15}0.132 & \cellcolor{gray!15}11 & \cellcolor{gray!15}0.066 & \cellcolor{gray!15}0.008 & \cellcolor{gray!15}0.976 & \cellcolor{gray!15}0.391 & \cellcolor{gray!15}91 & \cellcolor{gray!15}0.163 & \cellcolor{gray!15}0.075 & \cellcolor{gray!15}0.930 \\
MO-GRPO & -- & \cellcolor{gray!15}0.148 & \cellcolor{gray!15}24 & \cellcolor{gray!15}0.070 & \cellcolor{gray!15}0.016 & \cellcolor{gray!15}0.966 & \cellcolor{gray!15}0.372 & \cellcolor{gray!15}49 & \cellcolor{gray!15}\underline{0.208} & \cellcolor{gray!15}0.123 & \cellcolor{gray!15}0.886 \\
\midrule
GRPO & \TAG\ (ours) & \textbf{0.355} & 307 & \underline{0.099} & \textbf{0.335} & 0.310 & \textbf{0.595} & 434 & \textbf{0.215} & \underline{0.585} & 0.277 \\
\bottomrule
\end{tabular}%
}
\caption{\textbf{Program synthesis on DeepCoder-RT (Qwen3-1.7B and Qwen3-4B).} Held-out pass@4 at the end of training and, under Correct \& fast, the mean running time of the programs that ran (largest CPU time over the executed tests, timeouts excluded; lat., ms) and the rate of programs that are correct and faster than the reference solution (both); under Correct \& no-op the rate of programs that exit within $3$\,ms. Shaded: pass@4 below the base model. Best bold, second underlined (pass@4 and both only). Protocol in Appendix~\ref{app:exp-coding}.}
\label{tab:coding-main}
\vspace{-12pt}
\end{table}

\subsection{Isolated Drift and Anchor \texorpdfstring{Election}{Election}}
\label{sec:exp-pre}

\newtext{Every probe is trained for $T=300$ GRPO steps from the base model (rollout batch 32, $n=8$), and its per-token reverse KL is estimated on 256 held-out prompts at temperature 1. The probes separate without an oracle (Figure~\ref{fig:pre-kl}): on both Qwen models the auxiliary probe drifts $2.8$--$23\times$ further than its matched correctness probe in per-token reverse KL, and $12$--$21\times$ under Correct \& fast and Correct \& no-op on program synthesis, so the rule elects correctness in every cell, and every \TAG arm of Tables~\ref{tab:matrix} and~\ref{tab:coding-main} uses this elected anchor. The direction of the divergence carries the signal, as Lemma~\ref{lem:kl-direction} predicts: the forward KL fails on mathematics ($0.2$--$2.8\times$, and inverted on Qwen3-4B, where brevity scores $0.15$ against $0.68$ for correctness) and is two to three times less sharp on program synthesis ($5$--$13\times$ against $12$--$21\times$; Table~\ref{tab:probe-kl-both}), because base-generated text never visits an invented mode.}

\begin{figure}[h]
\centering
\includegraphics[width=\linewidth]{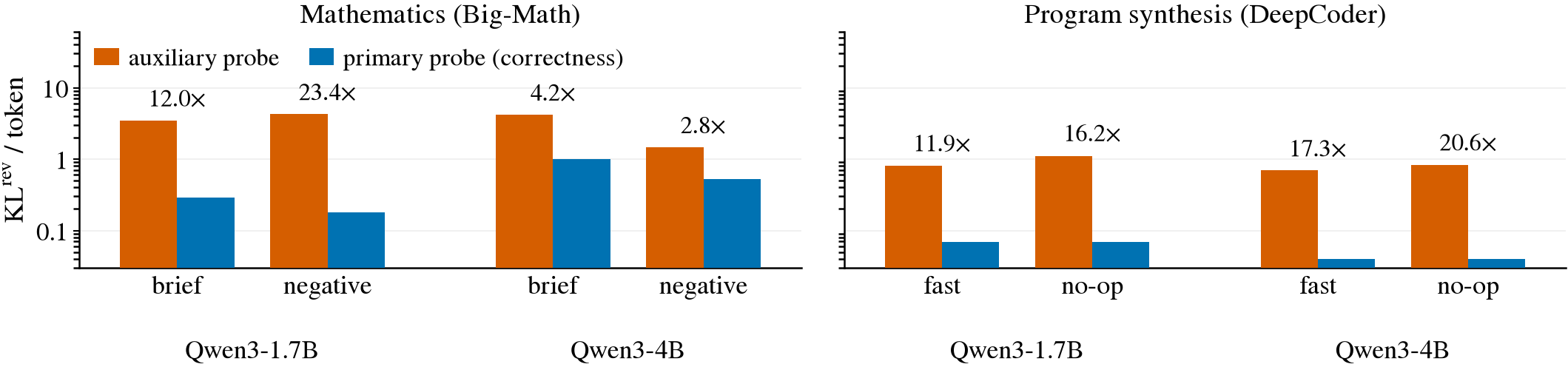}
\caption{\textbf{Anchor-election probes.} Left: Correct \& brief and Correct \& negative on mathematics; right: Correct \& fast and Correct \& no-op on program synthesis. Each auxiliary probe is paired with the correctness probe of its setting, and the correctness probe drifts less in every setting; values and the forward KL in Table~\ref{tab:probe-kl-both}.}
\label{fig:pre-kl}
\vspace{-3pt}
\end{figure}

\subsection{Generalization Across Models, Scales, and Reward Channels}
\label{sec:exp-family}

\newtext{We further examine whether aggregation-induced failure and the effectiveness of \TAG generalize across model families, scales, and reward configurations. On Llama-3.2-3B and Llama-3.1-8B \citep{llama3herd2024}, symmetric and constrained aggregation methods suffer severe correctness collapse under Correct \& negative, whereas \TAG matches the correctness-only reference (Table~\ref{tab:llama}). The pattern extends to Qwen3-30B-A3B, where additive aggregation collapses under Correct \& negative and exhibits unstable shortcut behavior under Correct \& brief, while \TAG maintains near-reference accuracy in both settings (Table~\ref{tab:30b}). We also extend the evaluation to three simultaneous reward channels: correctness, brevity, and negative answers. While additive aggregation and Rubric Dropout exploit both auxiliary rewards and severely degrade correctness, \TAG suppresses the negative-answer shortcut while retaining the benefits of brevity, achieving substantially higher task accuracy on both Qwen3-1.7B and Qwen3-4B (Table~\ref{tab:threech}). Additional results are provided in Appendix~\ref{app:exp-family} and Appendix~\ref{app:exp-c512}.}

\section{Conclusion}
\label{sec:conclusion}

\newtext{Combining individually meaningful rewards can undermine LLM post-training: the aggregate objective may improve even as primary-task performance deteriorates. We identify a fundamental mechanism behind this failure: in a stochastic decision model, symmetric aggregation that rewards single-channel success drives GRPO toward auxiliary-favoring behavior when auxiliary rewards are easier to obtain, regardless of initialization. To address this, we propose Trust-Anchored Gating (\TAG), which elects an anchor through isolated policy drift and structurally conditions auxiliary credit on anchor satisfaction. When the primary channel is anchored, \TAG provably reverses this selection while preserving incentives for joint success. Experiments across mathematics and program synthesis demonstrate that aggregation-induced failures extend beyond symmetric rules, while \TAG substantially mitigates these failures and retains the benefits of compatible auxiliary rewards.}

\paragraph{Limitations.}
\TAG assumes that one reward channel is the primary objective and the others are auxiliary, so that a single anchor can be elected; handling several primary channels, or none that deserves priority, is a direction for future work. Anchor election relies on short, channel-isolated probes, which introduce additional computational cost, and low policy drift does not guarantee alignment with the unknown true reward. Finally, our theoretical guarantees are established in a simplified two-channel, binary-reward model. Extending the analysis to more general reward settings and scaling anchor election to larger reward sets are further directions.

\subsubsection*{Ethics statement}
This work studies reward hacking in order to prevent it. All auxiliary channels in our experiments are synthetic and injected into controlled training runs; no deployed system was attacked, and no human subjects were involved. The collapse mechanisms we characterize are already documented in the public literature we cite; our contribution is defensive.

\subsubsection*{AI use statement}

Generative AI tools were used to improve the clarity and readability of the manuscript and to assist with editing software code. All AI-assisted text and code were reviewed and verified by the authors, who take full responsibility for the final content.

\bibliography{references}
\bibliographystyle{iclr2027_conference}
\clearpage

\appendix
\setcounter{lemma}{0}\renewcommand{\thelemma}{\thesection.\arabic{lemma}}
\setcounter{proposition}{0}\renewcommand{\theproposition}{\thesection.\arabic{proposition}}
\setcounter{corollary}{0}\renewcommand{\thecorollary}{\thesection.\arabic{corollary}}
\makeatletter\@addtoreset{lemma}{section}\@addtoreset{proposition}{section}\@addtoreset{corollary}{section}\makeatother
\setcounter{remark}{0}\renewcommand{\theremark}{\thesection.\arabic{remark}}

\section{Deferred Proofs}
\label{app:proofs}

Appendix~\ref{app:partial_compatibility} proves Theorem~\ref{thm:overlap}, records which steps extend to other estimators, derives the closed form of the recursion in the incompatible regime, and settles the boundary $\rho=1$ or $\eta=1$ excluded by the theorem (Proposition~\ref{prop:symmetric-boundary}). Appendix~\ref{app:tag_proofs} proves Theorem~\ref{thm:tag-conflict}. Appendix~\ref{app:tabular_dynamics} is a numerical companion: the exact recursions behind Figure~\ref{fig:sym-partial}, the MO-GRPO thresholds, and the weight thresholds cited in \S\ref{sec:reweight}.

\subsection{Proof of Theorem~\ref{thm:overlap}}
\label{app:partial_compatibility}

\emph{Step 0: semantics and centering.} Fix $x$ and write $\pi_j:=\pi_{\theta_t}(a_j\mid x)$
(time index suppressed). In the environment of
Section~\ref{sec:prompt_dynamics} a sampled $a_0$ realizes neither; a
sampled $a_1$ realizes joint success with probability $\eta$ and
auxiliary-only otherwise; a sampled $a_2$ realizes joint success,
primary-only and neither with probabilities $\eta$, $(1-\eta)\rho$ and
$(1-\eta)(1-\rho)$; realizations are independent given the actions, and
every response keeps the score of its chosen action. Write
$b:=f(0,0)<m:=f(0,1)=f(1,0)\le h:=f(1,1)$ for the three reward levels. Every advantage construction considered here is centered, $\sum_iA_i=0$. One on-policy GRPO step of size $\alpha>0$ is $\Delta\theta=\frac{\alpha}{n}\sum_iA_i\nabla_\theta\log\pi_\theta(y_i\mid x)$, a response being a single action; since $\partial_{\theta_j}\log\pi_{\theta_t}(a_i\mid x)=\mathbf 1\{a_i=a_j\}-\pi_j$, the update of the logit of $a_j$ is
$\Delta\theta_j=\frac{\alpha}{n}\sum_iA_i\bigl(\mathbf 1\{a_i=a_j\}-\pi_j\bigr)=\frac{\alpha}{n}\sum_{i:\,a_i=a_j}A_i$,
and log-odds of a softmax are logit differences. A group
may contain three reward levels, so no closed form of the update is available in general (the two-level case $\eta=0$ is treated after the proof). For $j\in\{0,1,2\}$ let $m_j$ denote the conditional mean of the GRPO advantage of a sampled response whose action is $a_j$ when the other $n-1$ responses are drawn i.i.d.\ from $\pi$, that is, the polynomial $\mathbb E[\varphi(r)]$ with $r\sim P_{a_j}$; the formula defines $m_j$ also at boundary points of the simplex, $m_j$ is continuous in $(\pi_0,\pi_1,\pi_2)$ and $|m_j|\le\sqrt{n-1}$. Each step changes every logit by at most $\alpha\sqrt{n-1}$, so the logits stay finite for all $t$. By exchangeability and centering,
\begin{equation}\label{eq:centering}
\mathbb E[\Delta\theta_j\mid\theta_t]=\alpha\,\pi_j\,m_j,
\qquad
\pi_0m_0+\pi_1m_1+\pi_2m_2=\mathbb E[A_i]=0.
\end{equation}

\emph{Step 1: monotonicity of the z-score in the own reward.} Fix the
rewards of all responses but one, with sum $S'$ and sum of squares $Q'$,
and let the remaining response have reward $r$. With $d:=(n-1)r-S'$ and
$V:=(n-1)Q'-S'^2\ge0$ (Cauchy--Schwarz), a direct computation gives
$n(r-\bar r)=d$ and $n^2\sigma_r^2=(d^2+nV)/(n-1)$, hence
\begin{equation}\label{eq:zscore}
A=\sqrt{n-1}\,\frac{d}{\sqrt{d^2+nV}}\quad\text{if }d^2+nV>0,
\qquad A=0\ \text{otherwise}.
\end{equation}
Thus $A$ is a nondecreasing function of $r$, strictly increasing when $V>0$
and equal to $\sqrt{n-1}\,\operatorname{sign}(d)$ when $V=0$, that is, when
the other responses all share one reward $v$ and $A=\sqrt{n-1}\,
\operatorname{sign}(r-v)$.

\emph{Step 2: shared-component coupling and dominance.} Let
$J\sim\mathrm{Bernoulli}(\eta)$ be the shared joint-success indicator and
$P\sim\mathrm{Bernoulli}(\rho)$ independent of it. On $J=1$ both actions
receive $h$; on $J=0$, $a_1$ receives $m$ while $a_2$ receives $m$ if $P=1$
and $b$ otherwise. This reproduces the two reward laws, and under the
coupling $r_1\ge r_2$ pointwise, strictly on $\{J=0,P=0\}$. Condition on the
other $n-1$ responses and let $\varphi$ be the nondecreasing function of
Step~1. Then
\begin{equation}\label{eq:dominance}
\begin{aligned}
\mathbb E[\varphi(r_1)]-\mathbb E[\varphi(r_2)]
&=(1-\eta)(1-\rho)\bigl[\varphi(m)-\varphi(b)\bigr]\ge0,\\
\mathbb E[\varphi(r_1)]-\varphi(b)
&=\eta\bigl[\varphi(h)-\varphi(b)\bigr]+(1-\eta)\bigl[\varphi(m)-\varphi(b)\bigr]\ge0,
\end{aligned}
\end{equation}
and $r_2\ge b$; averaging over the others gives $m_1\ge m_2$, $m_1\ge m_0$
and $m_2\ge m_0$ at every state. Taking expectations of the rewards
themselves gives $\mathbb E_{a_1}[R_f]-\mathbb E_{a_2}[R_f]=(1-\eta)(1-\rho)(m-b)$,
the identity quoted in Section~\ref{sec:main-result}; the proof does not use
it. Strictness: by Step~1, $\varphi(h)>\varphi(b)$ in every configuration
of the others, and $\varphi(m)>\varphi(b)$ unless the others all sit at
level $h$, an event of probability $\bigl(\eta(\pi_1+\pi_2)\bigr)^{n-1}<1$
whenever $\eta<1$. Hence for $\eta<1$: $m_1>m_0$ at every state, and for
$\eta<1$, $\rho<1$: $m_1>m_2$ at every state. On the boundary $\eta=1$ or
$\rho=1$ the rewards of $a_1$ and $a_2$ have the same law, $m_1=m_2$
everywhere, and the model is that of
Proposition~\ref{prop:symmetric-boundary} below.

\emph{Step 3: signs and log-odds.} Assume $\eta,\rho<1$ and
$\pi_0,\pi_1>0$. Since $\pi_0m_0+\pi_1m_1+\pi_2m_2=0$ with $m_1\ge m_2\ge m_0$ and $m_1>m_0$ (Step~2), we get $m_1>0$ (using $\pi_0>0$) and $m_0<0$ (using $\pi_1>0$). Consequently
\[
\begin{aligned}
\mathbb E\Bigl[\Delta\log\tfrac{\pi_1}{\pi_0}\Bigm|\theta_t\Bigr]
&=\alpha(\pi_1m_1-\pi_0m_0)>0,\\
\mathbb E\Bigl[\Delta\log\tfrac{\pi_1}{\pi_2}\Bigm|\theta_t\Bigr]
&=\alpha(\pi_1m_1-\pi_2m_2)=\alpha(2\pi_1m_1+\pi_0m_0),
\end{aligned}
\]
the last identity by centering. The second quantity can be negative: for
$(\pi_0,\pi_1,\pi_2)=(0.8,0.01,0.19)$, $\rho=0.9$, $\eta=0.1$ and $n=8$ it
equals $-0.27\alpha$ (exact enumeration), so the proof does not rely on its
monotonicity.

\emph{Step 4: $\pi_0\to0$.} $L_t:=\log(\pi_1/\pi_0)$ is strictly
increasing. Suppose $L_t\to L^\star<\infty$. If $\pi_1\ge\kappa>0$ infinitely often, then on the compact set $\{\pi_1\ge\kappa,\ e^{-L^\star}/2\le\pi_0/\pi_1
\le2e^{-L^\star}\}$, whose points have $\pi_0,\pi_1>0$, the continuous
function $m_1$ is positive and has a minimum $\mu>0$, so $\Delta L_t\ge
\alpha\kappa\mu$ infinitely often and $L_t\to\infty$, a contradiction. Hence
$\pi_1\to0$, $\pi_0\to0$, $\pi_2\to1$. At the vertex $(0,0,1)$ the others
are all $a_2$'s, $m_2(0,0,1)=0$ by centering, and Step~2 gives
$m_1(0,0,1)=(1-\eta)(1-\rho)\,\mathbb E[\varphi(m)-\varphi(b)]>0$ for
$\eta,\rho<1$, the others all sitting at level $h$ with probability
$\eta^{n-1}<1$. By continuity
$m_1\ge m_1(0,0,1)/2$ for large $t$, so $\Delta L_t\ge\alpha\pi_1m_1(0,0,1)/2$
and convergence of $L_t$ forces $\sum_t\pi_1<\infty$. But
$|\Delta\log(\pi_1/\pi_2)|\le\alpha\sqrt{n-1}\,(2\pi_1+\pi_0)\le C\pi_1$
for large $t$, so $\log(\pi_1/\pi_2)$ would converge, contradicting
$\pi_1\to0$, $\pi_2\to1$. Hence $L_t\to\infty$: $\pi_0/\pi_1\to0$ and
$\pi_0\to0$.

\emph{Step 5: $\pi_2\to0$.} Let $\ell_t:=\log(\pi_1/\pi_2)$, so that
$\Delta\ell_t\ge2\alpha\pi_1m_1-\alpha\sqrt{n-1}\,\pi_0$ by Step~3. Fix
$\varepsilon\in(0,\tfrac14)$. On the compact edge piece
$E_\varepsilon:=\{\pi_0=0,\ \pi_2\ge\varepsilon\}$ we have $m_1>0$: at points
with $\pi_1>0$ by centering ($\pi_1m_1+\pi_2m_2=0$) and $m_1>m_2$, and at
$(0,0,1)$ by Step~4. Let $2\mu_\varepsilon:=\min_{E_\varepsilon}m_1>0$; by
uniform continuity of the polynomial $m_1$ there is $\delta_\varepsilon>0$
with $m_1\ge\mu_\varepsilon$ whenever $\pi_2\ge\varepsilon$ and
$\pi_0\le\delta_\varepsilon$. Since $\pi_0\to0$ and $\pi_0/\pi_1\to0$
(Step~4), there is $T$ such that for all $t\ge T$: $\pi_0\le\delta_\varepsilon$
and $\pi_0/\pi_1\le\mu_\varepsilon/\sqrt{n-1}$, hence
$\Delta\ell_t\ge\alpha\pi_1\mu_\varepsilon>0$ whenever $\pi_2\ge\varepsilon$,
while always $\Delta\ell_t\ge-\alpha\sqrt{n-1}\,\pi_0$. The trajectory cannot
remain in $\{\pi_2\ge\varepsilon\}$ from some time on: there $\ell_t$ would be
nondecreasing, so $\pi_1/\pi_2$ and hence $\pi_1$ would stay bounded below,
the increments would be bounded below by a positive constant and
$\ell_t\to\infty$, contradicting $\pi_2\ge\varepsilon$. Thus
$\pi_2<\varepsilon$ infinitely often. Now fix $\delta\in(0,\log\tfrac32)$ and
enlarge $T$ so that also $\alpha\sqrt{n-1}\,\pi_{0,u}\le\delta$ and
$\pi_{0,u}\le\tfrac14-\varepsilon$ for all $u\ge T$; let $t_0\ge T$ be the
first time after $T$ with $\pi_2<\varepsilon$. If $\pi_{2,t}\ge2\varepsilon$
for some $t>t_0$, let $s\in[t_0,t)$ be the last time with $\pi_{2,s}<\varepsilon$;
then $\ell$ is nondecreasing on $(s,t]$ and decreases by at most $\delta$ at
step $s$, so $\ell_t\ge\ell_s-\delta\ge\log\frac{1-\varepsilon-\pi_{0,s}}{\varepsilon}-\delta$,
that is, $\pi_{2,t}\le\frac{\varepsilon\,e^{\delta}}{1-\varepsilon-\pi_{0,s}}
\le\tfrac43\,\varepsilon\,e^{\delta}<2\varepsilon$, a contradiction. Hence
$\pi_2<2\varepsilon$ for all $t>t_0$, so $\limsup_t\pi_2\le2\varepsilon$ for
every $\varepsilon\in(0,\tfrac14)$: $\pi_2\to0$, and with $\pi_0\to0$ this
gives $\pi_1\to1$.

\emph{Boundary cases.} No condition on $\alpha$ beyond $\alpha>0$ was
used. For $\eta=0$ the statement is Corollary~\ref{cor:reliability}, and the two-level identity~\eqref{eq:twolevel} below makes the drift explicit. For $\rho=0$ the action $a_2$
realizes joint success with probability $\eta$ and neither otherwise, which
the argument above covers. The excluded boundary $\eta=1$ or $\rho=1$ is
Proposition~\ref{prop:symmetric-boundary} below. Since $m_1(0,0,1)$ is
proportional to $(1-\eta)(1-\rho)$, the takeover from initializations near
the primary vertex is slow when either factor is small (at $\rho=0.5$,
$\eta=0.8$, $n=8$, $\alpha=0.1$, from $(0.8,0.01,0.19)$ the recursion needs
$1.0\times10^{6}$ iterations to reach $\pi_{\theta_t}(a_1\mid x)\ge0.99$);
the statement is asymptotic. \qed

\emph{Other estimators.} The argument uses only centering, the bound
$|m_j|\le\sqrt{n-1}$ and continuity of $m_j$, the own-reward monotonicity
of Step~1 with its strictness, and the positivity of $m_1$ at the vertex
$(0,0,1)$. The mean-baseline advantage $A_i=r_i-\bar r$ of Dr.GRPO is
centered, bounded, and equal to $\bigl((n-1)r-S^{\prime}\bigr)/n$, a
linear strictly increasing function of the own reward, so every step applies to it with the bound $\sqrt{n-1}$ replaced by $(n-1)(h-b)/n$. Exact policy gradient has $m_j=\bar r_j-\bar r_\pi$ (Appendix~\ref{app:tabular_dynamics}), which is centered, bounded by $h-b$, polynomial in $\pi$ and ordered $m_1>m_2\ge m_0$ at every state by Eq.~\eqref{eq:gap} and the outcome distributions $P_{a_j}$, with $m_1(0,0,1)=(1-\eta)(1-\rho)(m-b)>0$; Steps~3--5 therefore apply to it as well, and the conditional-mean Dr.GRPO update is exactly $(n-1)/n$ times the exact policy-gradient update. The argument does not apply to MO-GRPO, whose advantage is
not a function of the scalar aggregate: auxiliary-only and primary-only
responses receive different advantages in the same group, so Step~2 fails.
Theorem~\ref{thm:overlap} is stated for GRPO and extends as just described to Dr.GRPO and exact policy gradient; MO-GRPO is covered only on the boundary pieces $(\rho,\eta)=(1,0)$ and $\eta=1$ of Proposition~\ref{prop:symmetric-boundary}; in the compatible regime it selects $a_1$ (numerically, $n=8$) for $\eta$ below a $\rho$-dependent threshold and is bistable above it
(Appendix~\ref{app:tabular_dynamics}, Figure~\ref{fig:sym-partial}).

\paragraph{Two-level regime ($\eta=0$).}
\label{app:twolevel}
In the incompatible regime every realized reward takes one of the two values $b$ and $m$, and the recursion has a closed form; it is used in the remark on $\lambda$ after the proof of Theorem~\ref{thm:tag-conflict} and explains the affine-invariance mechanism behind the weight thresholds of Appendix~\ref{app:tabular_dynamics}. A sampled
response with action $a_j$ realizes the higher reward $m$ with
probability $\mu_j$, where $\mu_0=0$, $\mu_1=1$, $\mu_2=\rho$, and the lower
reward $b$ otherwise, independently across responses given their
actions. Let $s:=\sum_j\pi_j\mu_j=\pi_1+\rho\pi_2$ be the
marginal probability that a sampled response realizes the higher reward.
\emph{Group-level advantages.} Let $K$ be the number of responses
in the group that realized the higher reward. For $1\le K\le n-1$ the
realized rewards take two values with mean $b+K(m-b)/n$
and standard deviation $(m-b)\sqrt{K(n-K)}/n$, so every
higher-reward response has $A_i=\sqrt{(n-K)/K}$ and every lower-reward
response has $A_i=-\sqrt{K/(n-K)}$; the common shift and the gap cancel.
For $K\in\{0,n\}$ all $A_i=0$. Since $\sum_iA_i=0$,
$\Delta\theta_j=\frac{\alpha}{n}\sum_{i:\,a_i=a_j}A_i$ by Step~0.
\emph{Conditional mean.} By exchangeability,
$\mathbb E[\sum_{i:\,a_i=a_j}A_i\mid\theta_t]=n\pi_j\,\mathbb E[A_1\mid a_1=a_j]$.
Given $a_1=a_j$, response $1$ realizes the higher reward with probability
$\mu_j$, and each of the other $n-1$ responses does so independently with
probability $s$; let $K'\sim\mathrm{Bin}(n-1,s)$ count them. Hence
\[
\mathbb E[A_1\mid a_1=a_j]
=\mu_j\,\mathbb E\Bigl[\sqrt{\tfrac{n-1-K'}{K'+1}}\,\mathbf 1\{K'\le n-2\}\Bigr]
-(1-\mu_j)\,\mathbb E\Bigl[\sqrt{\tfrac{K'}{n-K'}}\,\mathbf 1\{K'\ge1\}\Bigr].
\]
For $K\sim\mathrm{Bin}(n,s)$ one has $\mathbb E[K\,h(K)]=ns\,\mathbb E[h(K'+1)]$
and $\mathbb E[(n-K)\,h(K)]=n(1-s)\,\mathbb E[h(K')]$. Applied to
$h(K)=\sqrt{(n-K)/K}\,\mathbf 1\{1\le K\le n-1\}$ and
$h(K)=\sqrt{K/(n-K)}\,\mathbf 1\{1\le K\le n-1\}$, for which
$K\,h(K)=(n-K)\,h(K)=\sqrt{K(n-K)}\,\mathbf 1\{1\le K\le n-1\}$ respectively,
this gives
\[
\mathbb E\Bigl[\sqrt{\tfrac{n-1-K'}{K'+1}}\,\mathbf 1\{K'\le n-2\}\Bigr]=(1-s)\,\kappa_n(s)
\quad\text{and}\quad
\mathbb E\Bigl[\sqrt{\tfrac{K'}{n-K'}}\,\mathbf 1\{K'\ge1\}\Bigr]=s\,\kappa_n(s),
\]
where
\[
\begin{aligned}
\kappa_n(u)&:=\frac{\mathbb E_{K\sim\mathrm{Bin}(n,u)}
\bigl[\sqrt{K(n-K)}\,\mathbf 1\{1\le K\le n-1\}\bigr]}{n\,u(1-u)}\\
&=\frac1n\sum_{k=1}^{n-1}\binom nk u^{k-1}(1-u)^{n-k-1}\sqrt{k(n-k)},
\end{aligned}
\]
which is a polynomial, strictly positive on $[0,1]$, with $\kappa_n(0)=\kappa_n(1)=\sqrt{n-1}$, so that
$0<\kappa_{\min}\le\kappa_n\le\kappa_{\max}<\infty$ there.
Therefore $\mathbb E[A_1\mid a_1=a_j]=\kappa_n(s)\bigl[\mu_j(1-s)-(1-\mu_j)s\bigr]
=\kappa_n(s)(\mu_j-s)$ and
\begin{equation}\label{eq:twolevel}
\mathbb E[\Delta\theta_j\mid\theta_t]=\alpha\,\kappa_n(s)\,\pi_j\,(\mu_j-s),
\qquad j\in\{0,1,2\}.
\end{equation}
This is the exact policy-gradient update of the bandit with expected
rewards $(\mu_0,\mu_1,\mu_2)=(0,1,\rho)$ multiplied by $\kappa_n(s)$. The derivation uses only that responses realize
the higher reward independently with probability $s$ and that a response
with action $a_j$ does so with probability $\mu_j$; for Dr.GRPO the advantages are $(m-b)(1-K/n)$ and $-(m-b)K/n$, the gap $m-b$ being absorbed into $\alpha$, $\mathbb E[K(n-K)]=n(n-1)s(1-s)$, and \eqref{eq:twolevel} holds with $\kappa_n$ replaced by the constant $(n-1)/n$.
\emph{Log-odds.} Using $1-s=\pi_0+(1-\rho)\pi_2$ and
$\rho-s=\rho\pi_0-(1-\rho)\pi_1$,
\[
\begin{aligned}
\mathbb E\Bigl[\Delta\log\tfrac{\pi_1}{\pi_0}\Bigm|\theta_t\Bigr]
&=\alpha\kappa_n(s)\bigl[\pi_1(1-s)+\pi_0s\bigr],\\
\mathbb E\Bigl[\Delta\log\tfrac{\pi_1}{\pi_2}\Bigm|\theta_t\Bigr]
&=\alpha\kappa_n(s)\bigl[\pi_0(\pi_1-\rho\pi_2)+2(1-\rho)\pi_1\pi_2\bigr].
\end{aligned}
\]
The first is strictly positive whenever $\pi_0,\pi_1\in(0,1)$. The second
can be negative: for $(\pi_0,\pi_1,\pi_2)=(0.8,0.01,0.19)$ and $\rho=0.9$
the bracket equals $-0.128$, so the $a_1$-vs-$a_2$ log-odds move toward
$a_2$ transiently. \emph{Convergence.} For $\rho<1$ the conclusion $\pi_1\to1$ from every
finite initialization is the case $\eta=0$ of Theorem~\ref{thm:overlap}
(Corollary~\ref{cor:reliability}); no condition on $\alpha$ beyond
$\alpha>0$ is used. At $\rho=1$, $s=\pi_1+\pi_2$ and \eqref{eq:twolevel} gives
$\mathbb E[\Delta\log(\pi_1/\pi_2)\mid\theta_t]=\alpha\,\kappa_n(\pi_1+\pi_2)\,\pi_0\,(\pi_1-\pi_2)$: the initially larger of the two rewarded actions keeps winning, which is the boundary treated next.

\paragraph{The deterministic boundary.}
\begin{proposition}[Deterministic symmetric boundary]
\label{prop:symmetric-boundary}
Let $f(1,1)\ge f(0,1)=f(1,0)>f(0,0)$, $n\ge2$, $\alpha>0$, let the initial
logits be finite, and let $(\rho,\eta)$ lie on the boundary $\rho=1$ or
$\eta=1$. Under the deterministic conditional-mean dynamics of GRPO or
mean-baseline Dr.GRPO, or of MO-GRPO when $\eta\in\{0,1\}$, if
$\pi_{\theta_0}(a_1\mid x)>\pi_{\theta_0}(a_2\mid x)$ then
$\pi_{\theta_t}(a_1\mid x)\to1$ and
$\pi_{\theta_t}(a_0\mid x),\pi_{\theta_t}(a_2\mid x)\to0$ as $t\to\infty$; the symmetric
statement holds if $\pi_{\theta_0}(a_2\mid x)>\pi_{\theta_0}(a_1\mid x)$;
if $\pi_{\theta_0}(a_1\mid x)=\pi_{\theta_0}(a_2\mid x)$, equality persists
for all $t$ and both probabilities tend to $\tfrac12$ as $t\to\infty$.
\end{proposition}

\begin{proof}
\emph{Case $\rho=1$ or $\eta=1$ (GRPO and Dr.GRPO; MO-GRPO at $\eta=1$).}
We use the notation of Step~0 and the identity~\eqref{eq:centering}. On this boundary the scalar rewards of $a_1$ and $a_2$
have the same law, $f(1,1)$ with probability $\eta$ and $f(0,1)=f(1,0)$
otherwise (at $\eta=1$ both actions always realize joint success, so even
their reward profiles coincide). For GRPO and Dr.GRPO the advantage is a
function of the own scalar reward and of the others' rewards, hence
$m_1=m_2$ at every state; for MO-GRPO at $\eta=1$ the same holds because
the profiles coincide. Moreover $m_1>m_0$ at every state: for GRPO and
Dr.GRPO the own-reward map $\varphi$ of Step~1 satisfies
$\varphi(f(1,1))>\varphi(f(0,0))$ in every configuration and
$\varphi(f(0,1))>\varphi(f(0,0))$ unless the others all realize $f(1,1)$,
an event of probability less than one when $\eta<1$; for MO-GRPO at
$\eta=1$ the per-channel z-score of a $1$ exceeds that of a $0$ in every
configuration. Centering then gives, whenever $\pi_0\in(0,1)$,
$m_0<0$ and $m_1=m_2=\pi_0|m_0|/(1-\pi_0)>0$, so that
\[
\begin{aligned}
\mathbb E\Bigl[\Delta\log\tfrac{\pi_1}{\pi_2}\Bigm|\theta_t\Bigr]
&=\alpha\,m_1(\pi_1-\pi_2)
=\alpha\,\pi_0|m_0|\,\frac{\pi_1-\pi_2}{\pi_1+\pi_2},\\
\mathbb E\Bigl[\Delta\log\tfrac{\pi_1}{\pi_0}\Bigm|\theta_t\Bigr]
&=\alpha(\pi_1m_1+\pi_0|m_0|)\in\bigl(0,\,2\alpha\pi_0|m_0|\bigr].
\end{aligned}
\]
The first identity shows that the ratio $\pi_1/\pi_2$ is nondecreasing when
it exceeds one, nonincreasing when it is below one, and constant when it
equals one. Suppose $\pi_{1,0}>\pi_{2,0}$ and let $L_t:=\log(\pi_1/\pi_0)$,
which is strictly increasing. If $L_t\to L^\star<\infty$, the compactness
argument of Step~4 gives
$\pi_1\to0$, hence $\pi_0\to0$, and then $\pi_2\le\pi_1\to0$, contradicting
$\sum_j\pi_j=1$. Hence $L_t\to\infty$, so $\pi_0\to0$ and, by the upper
bound on $\Delta L_t$, $\sum_t\pi_0|m_0|=\infty$. Since
$\pi_1/\pi_2\ge\pi_{1,0}/\pi_{2,0}=:r_0>1$ for all $t$, the first identity
gives $\Delta\log(\pi_1/\pi_2)\ge\alpha\frac{r_0-1}{r_0+1}\pi_0|m_0|$, whose
sum diverges: $\pi_2/\pi_1\to0$ and $\pi_1\to1$. The case
$\pi_{2,0}>\pi_{1,0}$ is symmetric, and if $\pi_{1,0}=\pi_{2,0}$ the
equality persists while $\pi_0\to0$ by the same argument, so both tend to
$\tfrac12$. For MO-GRPO at $\rho=1$, $\eta\in(0,1)$ the two actions have
different per-channel advantage laws and $m_1=m_2$ need not hold off the tie for $n\ge4$ (for $n\le3$ the per-channel z-score gap is constant, $m_1=m_2$ holds identically and the argument extends); an exact tie still persists by the channel-swap symmetry, but we do not claim the selection statement in general.

\emph{Case $(\rho,\eta)=(1,0)$ (all three estimators, MO-GRPO included).}
At this point the three actions realize the profiles neither $(00)$,
auxiliary-only $(01)$ and primary-only $(10)$ deterministically, so we index
outcomes by their binary encodings and write
$p_{00}:=\pi_0$, $p_{01}:=\pi_1$, $p_{10}:=\pi_2$ and $\theta_{00},\theta_{01},\theta_{10}$ for the corresponding logits.
The group counts $(N_{00},N_{01},N_{10})$ are multinomial with parameters
$n$ and $(p_{00},p_{01},p_{10})$. For a set of outcomes $S$ we write
$\sum_{i\in S}A_i$ for the sum over the sampled responses whose outcome
lies in $S$, and $B:=\sum_{i\in\{01\}}A_i-\sum_{i\in\{10\}}A_i$; by Step~0,
\[
\Delta\log\frac{p_{01}}{p_{10}}=\frac{\alpha}{n}\,B,
\qquad
\Delta\log\frac{p_{01}}{p_{00}}
=\frac{\alpha}{n}\Bigl[\sum_{i\in\{01\}}A_i-\sum_{i\in\{00\}}A_i\Bigr].
\]
The two realized reward levels may be taken as $0$ and $1$: GRPO and
MO-GRPO advantages are invariant under the positive affine map to
$f(0,0)$ and $f(0,1)$, and the Dr.GRPO advantage is multiplied by the gap
$f(0,1)-f(0,0)>0$, which amounts to rescaling $\alpha$ and changes no sign below. We establish three claims and then conclude.

\emph{Claim (a): group-level credit.} For GRPO, mean-baseline Dr.GRPO and MO-GRPO:
\textnormal{(i)} if $N_{00}=0$ then $A_i=0$ for all $i$;
\textnormal{(ii)} if $N_{00}\ge1$ and $N_{01}+N_{10}\ge1$, every response
with outcome $00$ has $A_i<0$, every response with outcome $01$ or $10$
has $A_i>0$, and $\operatorname{sign}B=\operatorname{sign}(N_{01}-N_{10})$;
\textnormal{(iii)} exchanging $N_{01}$ and $N_{10}$ with $N_{00}$ fixed
maps $B$ to $-B$.
Indeed, let $K:=N_{01}+N_{10}$. For GRPO the summed reward equals $1$ on $K$
responses and $0$ on $n-K$, so for $1\le K\le n-1$ every rewarded response
has $A_i=\sqrt{(n-K)/K}$, every $00$ response has $A_i=-\sqrt{K/(n-K)}$,
and $B=(N_{01}-N_{10})\sqrt{(n-K)/K}$. For Dr.GRPO the values are
$1-K/n$ and $-K/n$, and $B=(N_{01}-N_{10})(1-K/n)$. In both cases $K=n$,
that is $N_{00}=0$, or $K=0$ gives $A\equiv0$. For MO-GRPO the column of
$\Rrel$ has $N_{10}$ ones and the column of $\Rexp$ has $N_{01}$ ones. If
$N_{00}=0$ the two columns are complementary, their standardized versions
are negatives of each other, and $A\equiv0$. If $N_{00}\ge1$, put
$u:=\sqrt{N_{01}/(n-N_{01})}$ and $v:=\sqrt{N_{10}/(n-N_{10})}$, reading a
standardized term with zero count as $0$; then
\[
A_i=\frac1u-v\ \ (\text{outcome }01),\qquad
A_i=\frac1v-u\ \ (\text{outcome }10),\qquad
A_i=-u-v\ \ (\text{outcome }00),
\]
and
\[
B=\frac{N_{01}}{u}-\frac{N_{10}}{v}-N_{01}v+N_{10}u
=(1-uv)\Bigl(\frac{N_{01}}{u}-\frac{N_{10}}{v}\Bigr),
\]
where $N_{01}/u=\sqrt{N_{01}(n-N_{01})}$ and $N_{10}/v=\sqrt{N_{10}(n-N_{10})}$.
Since $(n-N_{01})(n-N_{10})-N_{01}N_{10}=nN_{00}>0$, we have $uv<1$,
$1/u>v$ and $1/v>u$ whenever the corresponding counts are positive, which
gives the signs of the individual advantages, and since
$N_{01}(n-N_{01})-N_{10}(n-N_{10})=(N_{01}-N_{10})N_{00}$ the second
factor of $B$ has the sign of $N_{01}-N_{10}$. Exchanging the channels
exchanges $u$ and $v$, which proves \textnormal{(iii)} for all three
constructions.

\emph{Claim (b): conditional mean of $B$.} For all three constructions,
\[
\begin{aligned}
\mathbb E[B\mid\theta_t]
&=\sum_{\substack{N_{00}\ge1\\ N_{01}>N_{10}}}
\binom{n}{N_{00},N_{01},N_{10}}\,
p_{00}^{N_{00}}(p_{01}p_{10})^{N_{10}}\\
&\qquad\times
\bigl(p_{01}^{\,N_{01}-N_{10}}-p_{10}^{\,N_{01}-N_{10}}\bigr)\,
B(N_{00},N_{01},N_{10}),
\end{aligned}
\]
where every summand has the sign of $p_{01}-p_{10}$ (strictly, since finite logits give $p_{00},p_{01},p_{10}>0$); in particular
$\mathbb E[B\mid\theta_t]$ has the sign of $p_{01}-p_{10}$ and vanishes
when $p_{01}=p_{10}$. The configuration $(N_{00},N_{01},N_{10})=(1,n-1,0)$
contributes $n\,p_{00}\bigl(p_{01}^{n-1}-p_{10}^{n-1}\bigr)\sqrt{n-1}$ for
GRPO and MO-GRPO and $n\,p_{00}\bigl(p_{01}^{n-1}-p_{10}^{n-1}\bigr)(n-1)/n$
for Dr.GRPO. To see this, group the configurations with the same $N_{00}$ in pairs exchanging
$N_{01}$ and $N_{10}$. The multinomial coefficient is the same for both
members, their probabilities are $p_{00}^{N_{00}}p_{01}^{N_{01}}p_{10}^{N_{10}}$
and $p_{00}^{N_{00}}p_{01}^{N_{10}}p_{10}^{N_{01}}$, and by
claim (a) the values of $B$ are negatives of each other
and vanish when $N_{00}=0$ or $N_{01}=N_{10}$. Summing over pairs gives
the displayed sum; for $d\ge1$, $p_{01}^d-p_{10}^d$ has the sign of
$p_{01}-p_{10}$, and $B(N_{00},N_{01},N_{10})>0$ for $N_{01}>N_{10}$. The
value of $B$ at $(1,n-1,0)$ is read off from the proof of claim (a).

\emph{Claim (c): persistence.} Consider the deterministic conditional-mean recursion of any of the three
constructions with finite initial logits, and suppose $p_{01}\ge p_{10}$
at $t=0$. Then \textnormal{(i)} $p_{00}$ is nonincreasing and
$p_{01}/p_{10}$ is nondecreasing in $t$; \textnormal{(ii)} there are
constants $0<c\le C<\infty$, depending only on $n$, $\alpha$ and the
initial policy, such that
$c\,p_{00}\le\mathbb E[\Delta\log(p_{01}/p_{00})\mid\theta_t]\le C\,p_{00}$
for all $t$; \textnormal{(iii)} $\sum_tp_{00}=\infty$ and $p_{00}\to0$.
For \textnormal{(i)}: by Step~0 and claim (a), in every group
$\Delta\theta_{01}-\Delta\theta_{00}\ge0$ and
$\Delta\theta_{10}-\Delta\theta_{00}\ge0$, strictly for informative groups,
which have positive probability; hence both ratios $p_{01}/p_{00}$ and
$p_{10}/p_{00}$ increase and $p_{00}$ decreases. By
claim (b), $\mathbb E[B\mid\theta_t]\ge0$ whenever
$p_{01}\ge p_{10}$, so $p_{01}/p_{10}$ is nondecreasing; in particular
$p_{01}/(p_{01}+p_{10})\ge q_0$ and $p_{01}+p_{10}\ge s_0$ for all $t$,
where $q_0\ge\tfrac12$ and $s_0>0$ denote the initial values.
For \textnormal{(ii)}, upper bound: by claim (a)(i) the
update vanishes unless $N_{00}\ge1$, an event of probability at most
$np_{00}$, and every $|A_i|\le2\sqrt{n-1}$; hence the expected update of
$\log(p_{01}/p_{00})$ is at most $2\alpha\sqrt{n-1}\,np_{00}$. Lower bound:
every group contributes a nonnegative amount to $\Delta\log(p_{01}/p_{00})$,
and the configuration $(1,n-1,0)$, of probability
$np_{00}p_{01}^{n-1}\ge np_{00}(s_0q_0)^{n-1}$, contributes
$\frac{\alpha}{n}\bigl[\sum_{i\in\{01\}}A_i-\sum_{i\in\{00\}}A_i\bigr]
\ge\frac{\alpha}{n}\cdot\frac{2(n-1)}{n}$ for each of the three
constructions (the values are $2\sqrt{n-1}$ for GRPO and MO-GRPO and
$2(n-1)/n$ for Dr.GRPO). For \textnormal{(iii)}: if $\sum_tp_{00}<\infty$, the
upper bound shows that $\log(p_{01}/p_{00})$ stays bounded, hence so does
$\log(p_{10}/p_{00})\le\log(p_{01}/p_{00})$, and therefore
$p_{00}=(1+p_{01}/p_{00}+p_{10}/p_{00})^{-1}$ stays bounded away from
zero, contradicting summability. Thus $\sum_tp_{00}=\infty$, the lower
bound gives $\log(p_{01}/p_{00})\to\infty$, and $p_{00}\le p_{00}/p_{01}\to0$.

\emph{Conclusion.} Suppose $p_{01}>p_{10}$ at $t=0$. By claim (b) and the
contribution of the configuration $(1,n-1,0)$,
\[
\mathbb E\Bigl[\Delta\log\frac{p_{01}}{p_{10}}\Bigm|\theta_t\Bigr]
\ge\alpha\,\frac{n-1}{n}\,p_{00}\bigl(p_{01}^{n-1}-p_{10}^{n-1}\bigr)
\ge\alpha\,\frac{n-1}{n}\,p_{00}\,(p_{01}-p_{10})\,p_{01}^{\,n-2}.
\]
With the notation of claim (c),
$p_{01}-p_{10}=(p_{01}+p_{10})(2q-1)$ where $q:=p_{01}/(p_{01}+p_{10})\ge q_0>\tfrac12$
is nondecreasing, and $p_{01}\ge s_0q_0$; hence the right-hand side is at
least $\gamma\,p_{00}$ with $\gamma:=\alpha\frac{n-1}{n}s_0(2q_0-1)(s_0q_0)^{n-2}>0$.
Since $\sum_tp_{00}=\infty$ by claim (c), the log-odds
$\log(p_{01}/p_{10})$ diverge, so $p_{10}/p_{01}\to0$; together with
$p_{00}\to0$ this gives $p_{01}\to1$. The case $p_{10}>p_{01}$ follows by
exchanging the channels. If $p_{01}=p_{10}$ initially, then
$\mathbb E[B\mid\theta_t]=0$ by claim (b), so
$p_{01}=p_{10}$ for all $t$, and claim (c) gives
$p_{00}\to0$; hence both tend to $\tfrac12$.
\end{proof}

\emph{Radial constructions.}
The $(1,0)$ argument uses only claim (a). The same
group-level sign relation holds for advantage constructions that center
the summed reward or each channel's reward and multiply the centered vector by a positive scalar function of its norm (GRPO: $\sqrt n/\|\cdot\|$; Dr.GRPO: $1$), a family that contains the constructions above; claim (c) additionally needs the multiplier bounded on $[0,\sqrt n]$ and positive at norm $\sqrt{(n-1)/n}$. We do not pursue the general statement here.

\subsection{Proof of Theorem~\ref{thm:tag-conflict}}
\label{app:tag_proofs}

Throughout, $c:=1/(1+\lambda)\in(0,1]$ with $\lambda:=\lambda_{\mathrm A}$ and anchor $R_{k^\star}=\Rrel$ as in Section~\ref{sec:tag}, $\pi_j:=\pi_{\theta_t}(a_j\mid x)$,
and $m_j$ denotes the conditional mean of the GRPO advantage of a sampled
response given that its action is $a_j$; as in Step~0 of
Appendix~\ref{app:partial_compatibility},
$\mathbb E[\Delta\theta_j\mid\theta_t]=\alpha\pi_jm_j$,
$\pi_0m_0+\pi_1m_1+\pi_2m_2=0$ \eqref{eq:centering}, each $m_j$ is a polynomial in $\pi$ with
$|m_j|\le\sqrt{n-1}$, and log-odds are logit differences. In the
environment of Section~\ref{sec:prompt_dynamics} the realized \TAG rewards
are $0$ for $a_0$; $1$ with probability $\eta$ and $0$ otherwise for $a_1$;
and $1$, $c$, $0$ with probabilities $\eta$, $(1-\eta)\rho$ and
$(1-\eta)(1-\rho)$ for $a_2$. Hence
$\mathbb E_{a_2}[R_{\TAG}]-\mathbb E_{a_1}[R_{\TAG}]=(1-\eta)\rho\,c$.

\emph{Auxiliary-only realizations never receive positive advantage.} Every
\TAG reward is nonnegative, and a response realizing auxiliary-only (or
neither) receives reward $0$, the minimum. If all rewards in the group
coincide, every advantage is $0$ by convention; otherwise the group mean
exceeds the minimum and the response's advantage $(0-\bar r)/s_r$ is
strictly negative. This holds for any $\lambda\ge0$ and any $\eta,\rho$.

\emph{Shared-component coupling and dominance.} Let
$J\sim\mathrm{Bernoulli}(\eta)$ be the shared joint-success indicator and
$P\sim\mathrm{Bernoulli}(\rho)$ independent of it. On $J=1$ both actions
receive reward $1$; on $J=0$, $a_1$ receives $0$ while $a_2$ receives $c$
if $P=1$ and $0$ otherwise. Thus $R_{\TAG}(a_2)\ge R_{\TAG}(a_1)$
pointwise, and with $\varphi$ the conditional z-score of Step~1 in
Appendix~\ref{app:partial_compatibility},
\begin{equation}\label{eq:tag-dominance}
\begin{aligned}
\mathbb E[\varphi(R_{\TAG}(a_2))]-\mathbb E[\varphi(R_{\TAG}(a_1))]
&=(1-\eta)\rho\,\bigl[\varphi(c)-\varphi(0)\bigr]\ge0,\\
\mathbb E[\varphi(R_{\TAG}(a_2))]-\varphi(0)
&=\eta\bigl[\varphi(1)-\varphi(0)\bigr]+(1-\eta)\rho\bigl[\varphi(c)-\varphi(0)\bigr]\ge0.
\end{aligned}
\end{equation}
Averaging over the others, $m_2\ge m_1\ge m_0$ at every state (the second
inequality because $R_{\TAG}(a_1)\ge0$). By Step~1, $\varphi(1)>\varphi(0)$
in every configuration and $\varphi(c)>\varphi(0)$ unless the others all sit
at level $1$, an event of probability $\bigl(\eta(\pi_1+\pi_2)\bigr)^{n-1}<1$
whenever $\eta<1$. Hence for $\eta<1$ and $\rho>0$: $m_2>m_1$ and $m_2>m_0$ at every state.

\emph{Proof of Theorem~\ref{thm:tag-conflict}.} Let $\eta<1$ and $\rho>0$.
The argument of Appendix~\ref{app:partial_compatibility} applies with the
roles of $a_1$ and $a_2$ exchanged and the levels $(b,m,h)$ replaced by
$(0,c,1)$; we record the three facts it uses. (i) Signs: if
$\pi_0,\pi_2>0$, centering with $m_2\ge m_1$ and $m_2>m_0$ gives $m_2>0>m_0$, so $\log(\pi_2/\pi_0)$ is strictly increasing, while
$\mathbb E[\Delta\log(\pi_2/\pi_1)\mid\theta_t]=\alpha(2\pi_2m_2+\pi_0m_0)$
need not be positive. (ii) Vertex: at $(0,1,0)$ the others are all $a_1$'s,
$m_1(0,1,0)=0$ by centering and
$m_2(0,1,0)=(1-\eta)\rho\,\mathbb E[\varphi(c)-\varphi(0)]>0$, the others
all sitting at level $1$ with probability $\eta^{n-1}<1$. (iii) Edge: on
$\{\pi_0=0,\ \varepsilon\le\pi_1\le1-\varepsilon\}$, $\pi_1m_1+\pi_2m_2=0$
and $m_2>m_1$ give $m_2>0$. Steps~4 and~5 of
Appendix~\ref{app:partial_compatibility} then yield $\pi_0\to0$ and
$\pi_1\to0$, hence $\pi_2\to1$, from every finite initialization, with no
condition on $\alpha$ beyond $\alpha>0$. The mean-baseline advantage of
Dr.GRPO is a linear strictly increasing function of the own reward, so the
same proof applies to it. \qed

\emph{Incompatibility and the role of $\lambda$.} At $\eta=0$ the \TAG
reward has the two levels $0$ and $c$, the positive scale cancels in the
standardization, and the two-level identity~\eqref{eq:twolevel} (with $\mu=(0,0,\rho)$ and $s=\rho\pi_2$) gives the
explicit drift
$\mathbb E[\Delta\log(\pi_2/\pi_j)\mid\theta_t]
=\alpha\,\kappa_n(\rho\pi_2)\,\rho\,\pi_2(1-\pi_2+\pi_j)>0$ for
$j\in\{0,1\}$, independent of $\lambda$; the incompatible case is thus
recovered exactly. For $\eta>0$, $\lambda$ enters only through groups that
contain both a joint success and a primary-only realization: in every other
group the rewards take at most two values and the update is
$\lambda$-invariant (the z-score is unchanged by a positive rescaling of the group's rewards). Within such a group let $n_0$, $n_c$ and $n_1$ be the numbers of other responses at levels $0$, $c$ and $1$, and put $x:=n_0c-n_1(1-c)$, $S':=n_cc+n_1$ and $V:=n_0n_cc^2+n_0n_1+n_cn_1(1-c)^2$. A direct computation gives $[\varphi(c)-\varphi(0)]/\sqrt{n-1}=x/\sqrt{x^2+nV}+S'/\sqrt{S'^2+nV}=:D(c)$ and, for $V>0$,
\[
D'(c)=n(n-1)\,n_1\Bigl[n_0\,(x^2+nV)^{-3/2}+n_c(1-c)\,(S'^2+nV)^{-3/2}\Bigr]\ge0,
\]
with equality only if $n_1=0$, while $D$ is constant in $c$ when $V=0$. Hence a larger primary-only bonus (smaller $\lambda$) increases $m_2-m_1$ at every state, strictly when the other responses include a joint success and at least one response at a different reward level. The selected action is the same for every $\lambda>0$ (Theorem~\ref{thm:tag-conflict}) and, by the remark below, for $\lambda=0$; in the exact recursion the time to selection increases with $\lambda$ in every case we computed, but we do not claim a trajectory-level rate comparison. The proof also covers $\lambda=0$
($c=1$), where \TAG reduces to training on $\Rrel$.

\emph{Boundary cases.} At $\eta=1$, and at $\rho=0$ with $\eta\in(0,1]$, the \TAG rewards of $a_1$ and $a_2$ have the same law, so $m_1=m_2$ at every state, while $m_1-m_0=\eta\,\mathbb E[\varphi(1)-\varphi(0)]>0$, which is all that the case $\rho=1$ or $\eta=1$ in the proof of Proposition~\ref{prop:symmetric-boundary} uses: for GRPO and Dr.GRPO the initially larger of $a_1,a_2$ is selected, an exact tie persists and both tend to $\tfrac12$. At $\rho=\eta=0$ every realized reward is $0$, the logits are constant and nothing is selected. Since $m_2(0,1,0)$ is
proportional to $(1-\eta)\rho$, selection from initializations near the
auxiliary vertex is slow when $\eta$ is close to one or $\rho$ is small
(at $\rho=0.1$, $\eta=0.95$, $n=8$, $\alpha=0.1$, $\lambda=0.5$, from
$(0.3,0.5,0.2)$ the recursion needs $1.2\times10^{6}$ iterations to reach
$\pi_{\theta_t}(a_2\mid x)\ge0.99$); the statement is asymptotic.

\subsection{Numerical Companion}
\label{app:tabular_dynamics}

This subsection records the exact conditional-mean recursions behind
Figure~\ref{fig:sym-partial} and the
constructions they compare, and the weight thresholds cited in \S\ref{sec:reweight}. Everything in it is a numerical observation about these recursions, not a theorem. Exact policy gradient, with
$\bar r_j:=\mathbb E_{a_j}[R_f]$, $\bar r_\pi:=\sum_j\pi_j\bar r_j$,
$A_j^{\mathrm{PG}}=\bar r_j-\bar r_\pi$ and
$\Delta\theta_j^{\mathrm{PG}}=\alpha\,\pi_jA_j^{\mathrm{PG}}$, serves as the
population reference; GRPO, Dr.GRPO and MO-GRPO use the advantages of
Sections~\ref{sec:prelim} and~\ref{sec:setting-cells}, and \TAG uses GRPO on $R_{\TAG}$.

\paragraph{Exact recursion.}
We implement the deterministic conditional-mean recursion exactly: at each
iteration all multinomial count vectors over the realized (action, outcome) classes of the two-channel model (four at $\eta=0$, six at $\eta>0$, for $\rho<1$) are enumerated, weighted by
their multinomial probability, and the exact expected logit update is
applied; exact policy gradient applies the population update directly
($n=8$, $\alpha=0.1$, horizon $10^5$ iterations). A stochastic check with
actually sampled groups ($n=8$, $\alpha=0.1$, $300$ seeds, $2\times10^4$
iterations) fluctuates around the deterministic recursion at the boundary point $\rho=1$, $\eta=0$ of Proposition~\ref{prop:symmetric-boundary}: from $(0.3,0.42,0.28)$ every seed selects auxiliary-only for all
three symmetric constructions, with median iterations to $0.90$ of $2433$
for GRPO and $7258$ for Dr.GRPO (interquartile ranges $1669$--$3570$ and
$5351$--$9285$) against deterministic values $2416$ and $7182$, while MO-GRPO seeds move the same way but do not reach $0.90$ within the horizon; from an exact tie ($1000$ seeds) auxiliary-only
is selected in $51\%$, $49\%$ and $50\%$ of seeds, so sampling noise breaks
the symmetry that the deterministic recursion preserves.

\paragraph{Compatible regime and the MO-GRPO threshold.}
The same five recursions in the shared joint-success model with $\eta>0$
($\rho=0.9$, $n=8$, $\alpha=0.1$, $\pi_{\theta_0}=(0.5,0.05,0.45)$)
separate the constructions covered by Theorem~\ref{thm:overlap} from
MO-GRPO. At $\eta=0.1$ (Figure~\ref{fig:sym-partial}) all four symmetric
constructions select $a_1$, MO-GRPO included ($7{,}186$ iterations to
$0.99$, against $10{,}235$ for GRPO, $22{,}445$ for Dr.GRPO and $19{,}637$
for exact policy gradient), and \TAG keeps $a_2$ ($395$). At $\eta=0.2$
from the same initialization the
three constructions of Theorem~\ref{thm:overlap} still select $a_1$
($16{,}543$, $29{,}448$ and $25{,}763$ iterations; the slowdown reflects the
factor $(1-\eta)(1-\rho)$ of Appendix~\ref{app:partial_compatibility}) and
\TAG keeps $a_2$ ($446$), but MO-GRPO stays at $a_2$
($\pi_{\theta_t}(a_2\mid x)=0.9996$ after $10^5$ iterations). The
difference is the stability of the primary vertex. Under MO-GRPO the conditional-mean drift of $\log(\pi_1/\pi_2)$ at $\pi\to(0,0,1)$, which by centering is governed by $m_1(0,0,1)$ along the edge $\pi_0=0$, is positive for $\eta$ below a threshold $\eta_c(\rho)$ and negative above it (numerically a single sign change), the threshold being equal to $0.18$ at $\rho=0.9$ for $n=8$ ($0.087$, $0.30$ and $0.49$ at
$\rho=0.95$, $0.85$ and $0.8$; positive for every $\eta\le0.6$ at
$\rho\le0.75$). For large $n$ the standardized advantage of an isolated
$a_1$ sample near that vertex is
$(\eta-p_1)/\sqrt{p_1(1-p_1)}+(1-\eta)/\sqrt{\eta(1-\eta)}$ with
$p_1=\eta+(1-\eta)\rho$, so, to leading order in $n$ (population z-scores in place of the group's), the vertex is unstable when $\eta(1-2\rho)+\rho(1-\rho)>0$: for every $\eta$ when $\rho\le1/2$, and for
$\eta<\rho(1-\rho)/(2\rho-1)$ otherwise ($0.11$ at $\rho=0.9$; the exact
thresholds are $0.13$ at $n=16$ and $0.12$ at $n=32$). Above the threshold MO-GRPO is numerically bistable: once $a_2$ also realizes joint success the auxiliary
column is no longer an indicator of $a_1$, per-channel standardization no
longer amplifies rare auxiliary samples, and in the exact recursion each pure action attracts the initializations on its side of the separatrix (at $\eta=0.2$ MO-GRPO selects $a_1$ from $(0.3,0.2,0.5)$ in $2{,}476$ iterations). Below the threshold the protection of the primary action $a_2$, which rests on the success probability $\rho$ of the primary channel, is lost.

\paragraph{Weighted sums under GRPO.}
\label{app:reliability}
A natural alternative to gating is to keep the symmetric estimator but
down-weight the auxiliary channel, $R_f=w_1\Rrel+w_2\Rexp$ with $0<w_2<w_1$. In
the incompatible regime ($\eta=0$) the realized rewards are then
$0$, $w_2$ and $w_1$ for neither, auxiliary-only and primary-only, so a sampled group
can contain three reward levels and the two-level identity~\eqref{eq:twolevel} no
longer applies; we therefore evaluate the conditional-mean GRPO recursion by
exact enumeration ($n=8$, $\alpha=0.1$). Table~\ref{tab:weight-thresholds} reports, for four values of
$\rho$, the largest ratio $w_2^\star/w_1$ below which $a_2$ is selected (the action whose probability tends to one under the recursion), found by bisection over $w_2/w_1\in[10^{-3},1]$; the features below are read from the same computation on a $\rho$-grid of step $0.05$. Three features of this
boundary matter. For $\rho\ge0.75$ every $w_2<w_1$ selects $a_2$, so the
boundary is more permissive than the expected-reward line $w_2/w_1=\rho$ of
exact policy gradient (which, in the environment of
Section~\ref{sec:prompt_dynamics}, is the crossing for every $\eta<1$; the
finite-group analysis of this paragraph is for $\eta=0$). For $0.55\le\rho\le0.7$ it lies strictly below $1$
and is essentially independent of the initialization
($w_2^\star/w_1\approx0.59$, $0.67$, $0.76$ and $0.88$ at
$\rho=0.55$, $0.6$, $0.65$ and $0.7$). For $\rho\le0.5$ it depends on the
initialization and shrinks to zero: from the primary-dominated
initialization at $\rho\le0.2$, and from the auxiliary-favored
initialization $(0.3,0.6,0.1)$ already at $\rho\le0.45$, the action $a_1$ is
selected for every $w_2/w_1\ge10^{-3}$ (at $\rho=0.2$ we verified
$w_2/w_1=10^{-3}$ over $4\times10^{5}$ iterations). The mechanism is the affine invariance of the z-score, which is unchanged by any positive affine rescaling of a group's rewards, so a two-valued group yields the same advantages whatever the gap: in a group in
which no $a_2$ response realizes primary-only, the rewards take only the two
values $0$ and $w_2$, so the auxiliary-only responses receive the full
two-level advantage
$\sqrt{(n-K)/K}$ whatever the value of $w_2>0$, and such groups dominate
when $\rho$ is small. Under GRPO the safe region thus closes already at moderate $\rho$ ($\rho\le0.2$ from the primary-led initialization, $\rho\le0.45$ from the auxiliary-favored one), so a conservative weight has no $\rho$-uniform safe value, whereas \TAG selects $a_2$ for every
$\rho>0$ from every finite initialization
(Theorem~\ref{thm:tag-conflict}). Since a prompt population mixes prompts
with different $\rho$, a single global weight would have to lie below
$\min_x w_2^\star(\rho_x)$, which is $0$ as soon as low-$\rho$ prompts
are present; this is consistent with the empirical finding of
Section~\ref{sec:exp-math} that reward weights move the onset of the
auxiliary takeover rather than its endpoint.

\begin{table}[t]
\centering
\small
\setlength{\tabcolsep}{5pt}
\begin{tabular}{c c cc c}
\toprule
 & exact PG / Dr.GRPO & \multicolumn{2}{c}{GRPO, $\eta=0$: largest safe $w_2/w_1$} & GRPO, $\eta=0.1$ \\
\cmidrule(lr){3-4}
$\rho$ & crossing $w_2/w_1$ & from $(0.1,0.05,0.85)$ & from $(0.3,0.6,0.1)$ & $w_1/w_2$ reaching $\mathbb E_{a_2}[R^\star]$ \\
\midrule
0.9 & 0.9 & 1 & 1 & 1.1 \\
0.7 & 0.7 & 0.88 & 0.88 & 1.3 \\
0.5 & 0.5 & 0.52 & 0.29 & 2.0 \\
0.3 & 0.3 & 0.21 & 0 & 3.9 \\
\bottomrule
\end{tabular}
\caption{\textbf{Weight thresholds of the conditional-mean recursions} ($n=8$, $\alpha=0.1$). Exact policy gradient and Dr.GRPO select $a_2$ for $w_2/w_1<\rho$ and $a_1$ for $w_2/w_1>\rho$ (expected-reward crossing). Middle columns: at $\eta=0$, the largest ratio $w_2^\star/w_1$ below which the GRPO recursion selects $a_2$ (bisection over $[10^{-3},1]$) from a primary-dominated and an auxiliary-favored initialization; $1$ means that every $w_2<w_1$ selects $a_2$, $0$ that $a_1$ is selected for every $w_2/w_1\ge10^{-3}$. Last column: in the shared joint-success model at $\eta=0.1$ from $\pi_{\theta_0}=(0.5,0.05,0.45)$, the smallest grid ratio $w_1/w_2$ at which the converged true reward $J^\star$ ($R^\star=0.9\Rrel+0.1\Rexp$) lies within $0.01$ of $\mathbb E_{a_2}[R^\star]$ (grid values $1.14$, $1.30$, $1.96$, $3.88$). \TAG selects $a_2$ for every $\rho>0$ (Theorem~\ref{thm:tag-conflict}).}
\label{tab:weight-thresholds}
\end{table}

\paragraph{Weights in the compatible regime.}
The last column of Table~\ref{tab:weight-thresholds} repeats the weight sweep in the shared
joint-success model at $\eta=0.1$ from $\pi_{\theta_0}=(0.5,0.05,0.45)$, the
setting of Figure~\ref{fig:sym-partial}, recording the true reward
$J^\star$ ($R^\star=0.9\Rrel+0.1\Rexp$) at convergence of the exact GRPO
recursion as a function of $w_1/w_2$ for $\rho\in\{0.3,0.5,0.7,0.9\}$. Reading off the smallest grid ratio at which $J^\star$ lies within $0.01$ of $\mathbb E_{a_2}[R^\star]$, $a_2$ is selected from $w_1/w_2\approx1.1$, $1.3$, $2.0$ and $3.9$ at $\rho=0.9$, $0.7$, $0.5$ and $0.3$ (grid values $1.14$, $1.30$, $1.96$ and $3.88$), roughly tracking the mean crossing $1/\rho$, which is the exact threshold for Dr.GRPO and exact policy gradient: their expected updates are proportional, and softmax policy gradient with the exact gradient selects the arm of largest expected reward from every interior initialization \citep{mei2020global}. For GRPO
the z-score adds two finite-group effects. For $w_2$ just below $w_1$ an
auxiliary-only response in a group of primary-only responses receives
$-\sqrt{n-1}$ however small the gap, which pins the boundary at $w_2/w_1=1$
for $\rho\ge0.9$. At intermediate $\rho$ a band of weights ($w_2/w_1$
between $0.78$ and $1$ at $\rho=0.7$) converges to a mixed policy rather
than a vertex ($\pi=(0,0.74,0.26)$ at $w_2/w_1=0.9$), which produces the
intermediate values of $J^\star$ next to each rise. At $\rho=0.3$ the
recursion is bistable, selecting $a_2$ from $(0.5,0.05,0.45)$ but $a_1$ from $(0.3,0.6,0.1)$ at $w_2/w_1=0.2$; this is the
finite-group effect behind the vanishing safe ratio in
Table~\ref{tab:weight-thresholds}.
\section{Anchor Election: Details, Theory and Diagnostics}
\label{app:pre_details}
\subsection{Method details}
\label{app:pre-method}

\paragraph{Why probe offline.}
Once rewards are optimized jointly, a behavior created by one channel can
simultaneously alter several others, making its source difficult to
identify. Probing each channel in isolation from the same base policy
$\pi_0$ exposes the behavior directly induced by that channel before
interactions among reward channels are introduced during joint training.
The probe policies $\pi_{T,k}$ are discarded afterwards and joint
optimization restarts from $\pi_0$.

\paragraph{Direction of the log-ratio.}
The quantity $\mathrm{KL}^{\mathrm{rev}}_{T,k}$ of Section~\ref{sec:pre} evaluates the
log-ratio under responses generated by the probed policy $\pi_{T,k}$. The
behavior of interest is the one created by isolated optimization, so the
divergence is evaluated under the policy that produces it. The opposite
direction, $D_{\mathrm{KL}}(\pi_0\,\|\,\pi_{T,k})$, evaluates behavior
sampled from the base policy and can be much less sensitive to newly
created, low-base-probability response modes; Appendix~\ref{app:pre-kl}
(Lemma~\ref{lem:kl-direction}) makes this asymmetry precise for mode invention and notes that a location drift of the bulk, for which the two divergences coincide, does not produce it, whereas an entropy change separates the two divergences without any mode being invented.

\paragraph{Per-token normalization.}
Without the normalization, the summed log-ratio averaged over responses is the sequence-level reverse KL $D_{\mathrm{KL}}(\pi_{T,k}\,\|\,\pi_0)$ by the chain rule. Dividing it by the expected number of generated tokens (in the estimate, pooling the per-token log-ratios of all responses and taking their mean) gives the sequence-level reverse KL per generated token: a scale normalization that keeps probes with different response lengths comparable.

\begin{table}[!t]
\centering
\footnotesize
\setlength{\tabcolsep}{4pt}
\begin{tabular}{l l ccc ccc}
\toprule
 & & \multicolumn{3}{c}{$\mathrm{KL}^{\mathrm{rev}}$ / token} & \multicolumn{3}{c}{$\mathrm{KL}^{\mathrm{fwd}}$ / token} \\
\cmidrule(lr){3-5}\cmidrule(lr){6-8}
Model & Setting & auxiliary & correctness & ratio & auxiliary & correctness & ratio \\
\midrule
\multirow{4}{*}{Qwen3-1.7B} & Correct \& brief & 3.47 & 0.29 & $12.0\times$ & 0.16 & 0.16 & $1.0\times$ \\
 & Correct \& negative & 4.25 & 0.18 & $23.4\times$ & 0.36 & 0.13 & $2.8\times$ \\
 & Correct \& fast & 0.82 & 0.069 & $11.9\times$ & 0.50 & 0.097 & $5.2\times$ \\
 & Correct \& no-op & 1.11 & 0.069 & $16.2\times$ & 0.50 & 0.097 & $5.2\times$ \\
\midrule
\multirow{4}{*}{Qwen3-4B} & Correct \& brief & 4.18 & 1.01 & $4.2\times$ & 0.15 & 0.68 & $0.2\times$ \\
 & Correct \& negative & 1.46 & 0.52 & $2.8\times$ & 0.60 & 0.42 & $1.4\times$ \\
 & Correct \& fast & 0.69 & 0.040 & $17.3\times$ & 0.38 & 0.038 & $10.0\times$ \\
 & Correct \& no-op & 0.82 & 0.040 & $20.6\times$ & 0.50 & 0.038 & $13.2\times$ \\
\bottomrule
\end{tabular}
\caption{\textbf{Anchor-election probes in both KL directions} (per token; $T=300$, 256 held-out prompts). Each auxiliary probe is paired with the correctness probe of the same setting; ratio = auxiliary/correctness. The reverse KL separates every pair (Figure~\ref{fig:pre-kl}); the forward KL fails on mathematics (inverted on Qwen3-4B) and is less sharp on program synthesis.}
\label{tab:probe-kl-both}
\end{table}

\paragraph{Election rule.}
Reverse-KL drift is the election signal: the channel with the smallest $\mathrm{KL}^{\mathrm{rev}}_{T,k}$ is elected as the anchor. In every setting of this paper the correctness channel has the smallest drift (Figure~\ref{fig:pre-kl}).

\paragraph{Scope.}
The election assumes that at least one channel admits no low-base-probability shortcut. The procedure ranks channels by how far their isolated optimization moves the policy into low-base-probability modes (Lemma~\ref{lem:kl-direction}). In the three-action model a saturated probe's reverse drift equals $-\log\pi_0(a_j\mid x)$, so the election picks the channel whose isolated optimum is the more typical under $\pi_0$; this coincides with the primary channel when the auxiliary action $a_1$ is the rarer one. Theorem~\ref{thm:tag-conflict} is stated for the case in which the elected anchor is the primary channel, as in every setting of this paper.
\subsection{Directional Sensitivity of KL to Mode Invention}
\label{app:pre-kl}
\newtext{Anchor election uses the per-token reverse KL $\mathrm{KL}^{\mathrm{rev}}_{T,k}$ of \S\ref{sec:pre}. To isolate the directional effect analytically, this section studies the distribution-level divergences $\mathrm{KL}^{\mathrm{rev}}:=D_{\mathrm{KL}}(\pi_T\Vert\pi_0)$, evaluated under the probe policy $\pi_T$, and $\mathrm{KL}^{\mathrm{fwd}}:=D_{\mathrm{KL}}(\pi_0\Vert\pi_T)$, evaluated under the base policy; the deployed statistic differs only by the division by the expected response length (Eq.~\eqref{eq:drev}). The quantities below are analytic or simulated, not training results.}

\newtext{\textbf{The general argument.} Model a probe policy that has invented a new mode as a mixture $\pi_T=(1-\varepsilon)\,\pi_0+\varepsilon\,\pi_H$ of the base policy and an invented component $\pi_H$, with $\varepsilon\in(0,1)$.}
\begin{lemma}[Directional sensitivity to mode invention]
\label{lem:kl-direction}
\newtext{(i) The forward divergence is bounded by the invented mass alone, whatever $\pi_H$ is:
\[
\mathrm{KL}(\pi_0\Vert\pi_T)\;\le\;-\log(1-\varepsilon)\;\le\;\frac{\varepsilon}{1-\varepsilon}.
\]
(ii) For every measurable set $A$,
\[
\mathrm{KL}(\pi_T\Vert\pi_0)\;\ge\;\mathrm{kl}\big(\pi_T(A)\,\Vert\,\pi_0(A)\big),
\qquad
\mathrm{kl}(a\Vert b):=a\log\tfrac ab+(1-a)\log\tfrac{1-a}{1-b}.
\]
In particular, if $\pi_H(A)\ge\vartheta$ on a region $A$ with $\pi_0(A)=q\le\varepsilon\vartheta$, then $\mathrm{KL}(\pi_T\Vert\pi_0)\ge\varepsilon\vartheta\log\frac{\varepsilon\vartheta}{q}-\varepsilon\vartheta$, which grows without bound as $q\to0$ at fixed $\varepsilon\vartheta$.}
\end{lemma}
\begin{proof}
\newtext{(i) Pointwise $\pi_T\ge(1-\varepsilon)\pi_0$, so $\log(\pi_0/\pi_T)\le-\log(1-\varepsilon)$ wherever $\pi_0>0$; integrate against $\pi_0$, and use $-\log(1-\varepsilon)\le\varepsilon/(1-\varepsilon)$. (ii) is the data-processing inequality for the coarse-graining $y\mapsto\mathbf 1[y\in A]$. For the last claim, $\pi_T(A)\ge\varepsilon\pi_H(A)\ge\varepsilon\vartheta\ge q$; since $(1-a)\log\frac{1-a}{1-b}\ge-(a-b)$ by $\log u\ge1-1/u$, we have $\mathrm{kl}(a\Vert b)\ge a\log\frac ab-a$, which is increasing in $a$ for $a\ge b$, so $\mathrm{kl}(\pi_T(A)\Vert q)\ge\varepsilon\vartheta\log\frac{\varepsilon\vartheta}{q}-\varepsilon\vartheta$.}
\end{proof}
\newtext{The two bounds are the whole reason for the direction. Inventing a mode of weight $\varepsilon$ moves the forward divergence by at most about $\varepsilon$, however atypical the new mode is under the base policy, because base-generated text never visits it. The reverse divergence is evaluated under the probe policy, which does visit it, and is charged the log-ratio $\log(\varepsilon\vartheta/q)$ of the invented region, which is large exactly when the region is improbable under the base policy.}

\newtext{\textbf{Gaussian illustration.} Base $\pi_0=\mathcal N(0,1)$; probe $\pi_T=(1-\varepsilon)\,\mathcal N(\delta,1)+\varepsilon\,\mathcal N(\Delta,1)$, a benign bulk drift $\delta$ plus an invented far mode of weight $\varepsilon$ at separation $\Delta$; $\mathrm{KL}^{\mathrm{fwd}}$ is estimated on $N=2048$ samples from $\pi_0$ and $\mathrm{KL}^{\mathrm{rev}}$ on $2048$ samples from $\pi_T$. In the well-separated regime $\mathrm{KL}^{\mathrm{fwd}}\approx\varepsilon+\delta^2/2$ is insensitive to where the invented mass went, while $\mathrm{KL}^{\mathrm{rev}}\approx\varepsilon(\ell-1)$ with $\ell=\log\varepsilon+\Delta^2/2$ the per-sample log-ratio of the invented mode, so $\mathrm{KL}^{\mathrm{rev}}/\mathrm{KL}^{\mathrm{fwd}}\approx \ell-1$ grows without bound in $\Delta^2$ (Figure~\ref{fig:kl-overview}; the approximation is within $5\%$ for $\Delta\ge5$ at $\varepsilon=0.05$). At $\varepsilon=5\%$, $\Delta=6$: $\mathrm{KL}^{\mathrm{fwd}}=0.0505$, $\mathrm{KL}^{\mathrm{rev}}=0.704$ ($14\times$); both scale linearly with $\varepsilon$, so a quoted ratio should carry its $(\varepsilon,\Delta)$. The two directions are complementary rather than redundant: for mode \emph{dropping} (bimodal base, policy keeps one mode) the ratio reverses ($\mathrm{KL}^{\mathrm{fwd}}=8.31$, $\mathrm{KL}^{\mathrm{rev}}=0.689$). $\mathrm{KL}^{\mathrm{fwd}}$ is the forgetting monitor and $\mathrm{KL}^{\mathrm{rev}}$ the invention monitor; the collapsed behaviors in this paper (instant short answers, negative constant answers, trivial programs that exit at once) are inventions.}

\begin{figure}[h]
\centering
\includegraphics[width=\linewidth]{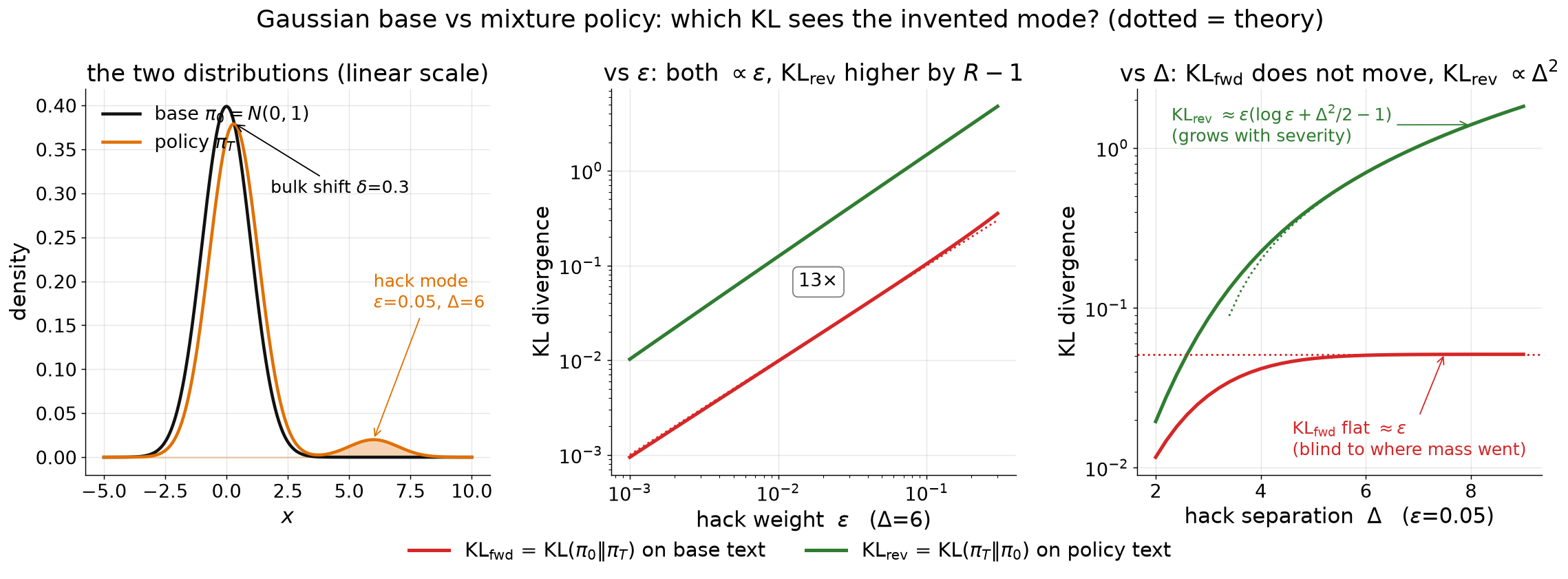}
\caption{\textbf{Forward vs reverse KL under an invented mode} (Monte Carlo, $N=2048$ samples). Left: the two distributions. Middle: $\varepsilon$ sweep at $\Delta=6$ --- both KLs are proportional to $\varepsilon$, $\mathrm{KL}^{\mathrm{rev}}$ higher by $\ell-1$. Right: $\Delta$ sweep at $\varepsilon=5\%$ --- $\mathrm{KL}^{\mathrm{fwd}}$ is flat at $\approx\varepsilon$ whatever the separation, $\mathrm{KL}^{\mathrm{rev}}$ grows with $\Delta^2$.}
\label{fig:kl-overview}
\end{figure}

\newtext{\textbf{Detection under benign drift.} With the invented mode switched on or off at fixed $\Delta=6$, $\varepsilon=5\%$ and a benign drift $\delta$, the $\mathrm{KL}^{\mathrm{fwd}}$ excess is ``$\varepsilon$ on top of $\delta^2/2$'' and is swamped once $\delta\gtrsim0.5$, while the $\mathrm{KL}^{\mathrm{rev}}$ excess is ``$\varepsilon(\ell-1)$ on top of $\delta^2/2$'' (Table~\ref{tab:kl-drift}); over a $(\delta,\varepsilon)$ grid the $z\ge5$ alarm region covers $22\%$ of the grid for $\mathrm{KL}^{\mathrm{fwd}}$ and $63\%$ for $\mathrm{KL}^{\mathrm{rev}}$. Benign location drift pulls the ratio $\mathrm{KL}^{\mathrm{rev}}/\mathrm{KL}^{\mathrm{fwd}}$ from $\ell-1$ toward $1$; only a new mode pulls it away from $1$ (a change of the policy's entropy would also do so; we consider location drift only), and the difference $\mathrm{KL}^{\mathrm{rev}}-\mathrm{KL}^{\mathrm{fwd}}\approx\varepsilon(\ell-2)$ is nearly drift-invariant for $\delta\le1$ ($0.65$ at $\delta\le0.5$, $0.63$ at $\delta=1$). The empirical counterparts are in \S\ref{sec:exp-pre}.}
\begin{fixedtable}
\centering
\scriptsize
\begin{tabular}{cccccccc}
\toprule
$\delta$ & $\mathrm{KL}^{\mathrm{fwd}}$ on/off & excess & $z$ & $\mathrm{KL}^{\mathrm{rev}}$ on/off & excess & $z$ & $\mathrm{KL}^{\mathrm{rev}}-\mathrm{KL}^{\mathrm{fwd}}$ \\
\midrule
0.3 & 0.094 / 0.043 & $+117\%$ & 7.8 & 0.745 / 0.043 & $+1614\%$ & 108 & 0.65 \\
0.5 & 0.171 / 0.120 & $+43\%$ & 4.7 & 0.819 / 0.120 & $+583\%$ & 65 & 0.65 \\
1.0 & 0.528 / 0.477 & $+11\%$ & 2.4 & 1.161 / 0.477 & $+143\%$ & 32 & 0.63 \\
2.0 & 2.051 / 2.000 & $+3\%$ & 1.2 & 2.624 / 2.000 & $+31\%$ & 14 & 0.57 \\
\bottomrule
\end{tabular}
\caption{\textbf{Invented mode on vs off at $\Delta=6$, $\varepsilon=5\%$} under a benign drift $\delta$; $z=(\text{on}-\text{off})/(\sigma_{\rm off}/\sqrt N)$, $N=2048$.}
\label{tab:kl-drift}
\end{fixedtable}
\section{Experimental details and full tables}
\label{app:exp-details}
\newtext{This appendix tabulates every training arm behind \S\ref{sec:experiments}. Ratios quoted in the prose divide held-out $\Rrel$ by the same-model, same-budget $\Rrel$-only value (the tables print raw held-out values only). Held-out noise is about $\pm0.02$. Table~\ref{tab:hyper} gives the shared configuration.}

\begin{fixedtable}
\centering
\footnotesize
\setlength{\tabcolsep}{3.5pt}
\resizebox{\linewidth}{!}{\begin{tabular}{l cccc}
\toprule
 & Qwen3-1.7B / 4B & Llama-3.2-3B / 3.1-8B & Qwen3-30B-A3B & Program synthesis \\
\midrule
Training data & Big-Math hard split & Llama-calibrated Big-Math subset & Big-Math hard split & DeepCoder-RT \\
Response budget (tok) & 512 & 1024 & 4096 & 4096 \\
Rollout batch & 256 & 256 & 256 & 256 \\
Training steps & 500 & 500 & 500 & 500 \\
Probe horizon $T$ (steps, rollout batch) & 300, 32 & 300, 32 & 300, 32 & 300, 32 \\
Held-out problems & 1{,}498 & 1{,}500 & 1{,}498 & 185 \\
Compatible auxiliary ($\Rexp$) & $<200$\,tok & $<128$\,tok & $<512$\,tok & $<60$\,ms \\
Incompatible auxiliary ($\Rexp$) & neg.\ answer & neg.\ answer & neg.\ answer & $<3$\,ms \\
\bottomrule
\end{tabular}}
\caption{\textbf{Training and evaluation configuration by testbed.} All arms use GRPO with learning rate $10^{-6}$, no KL penalty (except the additive$+$KL arms), clipping $0.2/0.28$; Llama arms sample at temperature $0.6$, top-$p$ $0.9$. The compatible auxiliary channel rewards responses shorter than the listed threshold (math) or fast execution (program synthesis); the incompatible one rewards a negative final answer (math) or execution faster than the listed latency on problems whose reference solution needs at least $10$\,ms (program synthesis).}
\label{tab:hyper}
\end{fixedtable}

\paragraph{Evaluation provenance.}
\newtext{The ``base model'' rows of Tables~\ref{tab:matrix}, \ref{tab:llama} and~\ref{tab:coding-main} are the held-out evaluation of the shared base policy at step $0$. The $\Rrel$-only reference runs do not score the negative-answer channel; their held-out negative-answer rates are recomputed from the saved evaluation responses with the reward model's own extraction rule. The program-synthesis table follows the protocol of Appendix~\ref{app:exp-coding}. In all main experiments the \TAG anchor is the channel elected by the probes of \S\ref{sec:exp-pre}.}

\subsection{Main grid at cap 512: every aggregator}
\label{app:exp-c512}

\begin{figure}[!t]
\centering
\includegraphics[width=\linewidth]{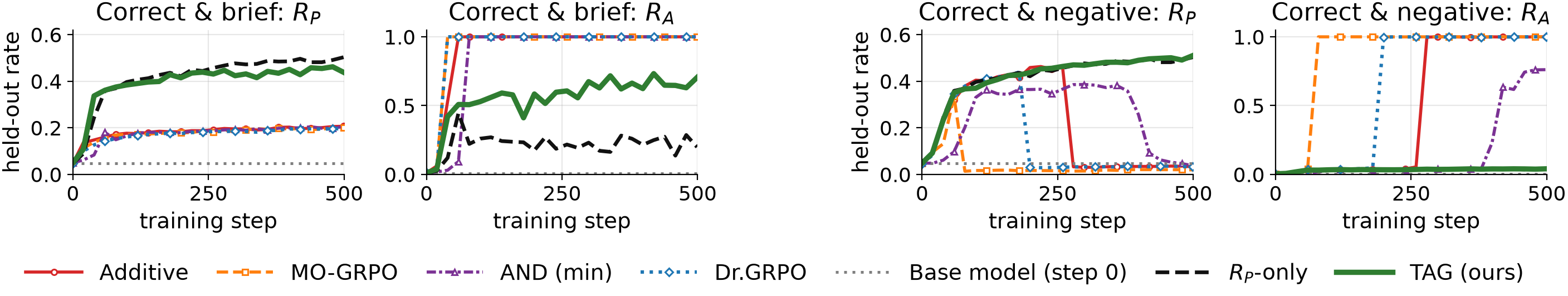}
\caption{\textbf{Training trajectories} (Qwen3-4B): Figure~\ref{fig:main} repeated for the larger model, with the same arms, panels and evaluation cadence.}
\label{fig:main-4b}
\end{figure}

\newtext{Table~\ref{tab:matrix-raw} gives the auxiliary rates and pass@1 of the full cap-512 grid behind Table~\ref{tab:matrix} (rollout batch 256, 500 steps, 1{,}498 held-out problems, mean over 4 samples per prompt); this paragraph records the readings behind it. Under Correct \& brief every baseline aggregation method (additive, additive$+$KL ($\beta=0.01$), LSC, Lagrangian, $\min$, MO-GRPO, PG-OT) collapses to $0.38$--$0.44\times$; additive$+$KL ($\beta=1$) stalls without learning ($0.11$--$0.24\times$); Rubric Dropout falls to $0.54\times$ at 1.7B and $0.40\times$ at 4B; \TAG keeps $0.91\times/0.87\times$. Under Correct \& negative the twelve symmetric or constrained baseline arms (six per size) collapse to $0.04$--$0.08\times$; $\min$ stays at $0.69\times$ at 1.7B and falls to $0.07\times$ at 4B; Rubric Dropout collapses at 1.7B ($0.09\times$) but not at 4B ($0.97\times$); \TAG stays at the reference ($1.01\times$ at both sizes). The mean-baseline estimator Dr.GRPO applied to the additive reward collapses in all four Qwen cells (Table~\ref{tab:matrix-raw}: held-out $\Rrel$ $0.144/0.026$ at 1.7B and $0.196/0.033$ at 4B), in both Llama Correct \& negative cells (Table~\ref{tab:llama}) and in the Llama-8B Correct \& brief cell ($0.274$ against $0.766$; \S\ref{app:exp-family}), while the Llama-3B Correct \& brief arm degrades ($0.72\times$). With three channels trained simultaneously (negative answer, brevity, correctness) additive and Dropout cascade through both auxiliary behaviors (brevity first, then the negative answer) to $0.04$--$0.05\times$, while \TAG keeps $0.82\times$ (1.7B) and $0.90\times$ (4B) with the negative-answer channel at the $\Rrel$-only level. Table~\ref{tab:threech} gives the three-channel endpoints. Figure~\ref{fig:main-4b} repeats Figure~\ref{fig:main} for Qwen3-4B.}

\newtext{\textbf{Baseline ports.}\label{app:baseline-ports} Every baseline shapes a scalar from the two binary channels of a group of $n$ samples and feeds it to the standard GRPO advantage unless stated otherwise. \emph{Additive}: $(\Rrel+\Rexp)/2$. \emph{AND}: $\min(\Rexp,\Rrel)$. \emph{LSC soft-AND} \citep{wang2024transforming}: $\log\sigma(a(\Rexp-c))+\log\sigma(a(\Rrel-c))$ with $a=4$, $c=\tfrac12$; on binary channels each term takes two values, so the shaped reward is an affine function of $\Rexp+\Rrel$, and since the group z-score is affine-invariant the advantages coincide with the additive rule sample by sample. Its runs reproduce the additive arm within seed noise (Tables~\ref{tab:matrix-raw} and~\ref{tab:llama}) and are kept as this check; the row is omitted from Table~\ref{tab:matrix}. \emph{Lagrangian} \citep{moskovitz2023confronting}: $\Rrel+\nu_t\Rexp$, the multiplier updated once per rollout batch by dual ascent on the constraint $\mathbb E[\Rexp]\ge\tau$, $\nu_{t+1}=\max\{0,\nu_t+\alpha_{\rm dual}(\tau-\bar\Rexp)\}$ with $\tau=0.5$, $\alpha_{\rm dual}=0.05$, $\nu_0=0.5$; once the auxiliary rate exceeds $\tau$ the multiplier decays to $0$ within about twenty steps and the arm continues on $\Rrel$ alone. \emph{PG-OT} \citep{ba2026pareto}: a reward-level port of the transport objective in its $\varepsilon\to0$ limit: per group, each channel is standardized, the Pareto front of the $n$ points is computed, and a sample's shaped reward is minus its squared distance to the nearest front point (front points score $0$). \emph{MO-GRPO} \citep{mogrpo2025}: the per-channel group z-scores are summed and used as the advantage directly, with no second standardization. \emph{Dr.GRPO} \citep{drgrpo2025} on $(\Rexp+\Rrel)/2$: $A_i=r_i-\bar r$ without the group-standard-deviation division, and the per-sequence $1/|o_i|$ loss normalizer replaced by the constant $1/L_{\max}$, implemented by scaling $A_i$ by $|o_i|/L_{\max}$. \emph{Rubric Dropout} \citep{rubricdropout2026}: with probability $0.5$ per group the auxiliary channel is dropped and the group is scored by $\Rrel$ alone, otherwise by $(\Rexp+\Rrel)/2$. Weights are equal except in the weighted-sum arm; under the z-score a weight is immaterial in two-level groups (Appendix~\ref{app:twolevel}).} \emph{Additive$+$KL}: the additive rule trained with the standard GRPO KL penalty to the base policy, coefficient $\beta\in\{0.01,1\}$ (every other arm uses no KL penalty).

\begin{table}[t]
\centering
\footnotesize
\setlength{\tabcolsep}{3pt}
\resizebox{\linewidth}{!}{%
\begin{tabular}{l cc cc cc cc}
\toprule
 & \multicolumn{4}{c}{Qwen3-1.7B} & \multicolumn{4}{c}{Qwen3-4B} \\
\cmidrule(lr){2-5}\cmidrule(lr){6-9}
 & \multicolumn{2}{c}{Correct \& brief} & \multicolumn{2}{c}{Correct \& negative} & \multicolumn{2}{c}{Correct \& brief} & \multicolumn{2}{c}{Correct \& negative} \\
\cmidrule(lr){2-3}\cmidrule(lr){4-5}\cmidrule(lr){6-7}\cmidrule(lr){8-9}
Aggregation & $\Rexp$ & $\Rrel$ & $\Rexp$ & $\Rrel$ & $\Rexp$ & $\Rrel$ & $\Rexp$ & $\Rrel$ \\
\midrule
Additive & 1.000 & 0.119 & 0.999 & 0.021 & 1.000 & 0.207 & 1.000 & 0.034 \\
Weighted sum $0.9\Rrel{+}0.1\Rexp$ & 0.996 & 0.155 & 0.999 & 0.027 & 1.000 & 0.200 & 0.041 & 0.492 \\
MO-GRPO & 0.999 & 0.137 & 0.999 & 0.024 & 1.000 & 0.202 & 1.000 & 0.020 \\
AND ($\min$) & 1.000 & 0.125 & 0.039 & 0.215 & 1.000 & 0.205 & 0.761 & 0.037 \\
LSC soft-AND & 1.000 & 0.125 & 0.998 & 0.026 & 1.000 & 0.200 & 0.998 & 0.037 \\
Lagrangian & 1.000 & 0.127 & 0.999 & 0.025 & 1.000 & 0.206 & 0.919 & 0.035 \\
PG-OT & 1.000 & 0.137 & 0.999 & 0.023 & 1.000 & 0.194 & 1.000 & 0.033 \\
Rubric Dropout & 0.995 & 0.166 & 0.993 & 0.028 & 1.000 & 0.203 & 0.044 & 0.488 \\
Additive$+$KL ($\beta{=}0.01$) & 1.000 & 0.135 & 0.985 & 0.017 & 1.000 & 0.203 & 1.000 & 0.020 \\
Additive$+$KL ($\beta{=}1$) & 0.025 & 0.073 & 0.007 & 0.072 & 0.013 & 0.054 & 0.005 & 0.051 \\
Dr.GRPO (additive) & 0.999 & 0.144 & 1.000 & 0.026 & 1.000 & 0.196 & 1.000 & 0.033 \\
Base model (step $0$) & 0.024 & 0.068 & 0.007 & 0.068 & 0.011 & 0.049 & 0.005 & 0.048 \\
$\Rrel$-only & 0.114 & 0.310 & 0.036 & 0.310 & 0.199 & 0.503 & 0.039 & 0.503 \\
\TAG & 0.550 & 0.282 & 0.033 & 0.314 & 0.705 & 0.438 & 0.041 & 0.510 \\
\bottomrule
\end{tabular}}
\caption{\textbf{Main grid (Qwen), auxiliary rate and pass@1.} Held-out auxiliary-behavior rate $\Rexp$ and accuracy $\Rrel$ (pass@1) at the end of training for the arms of Table~\ref{tab:matrix} and the additional ports of Appendix~\ref{app:baseline-ports}, same evaluation.}
\label{tab:matrix-raw}
\end{table}

\begin{table}[t]
\centering
\footnotesize
\setlength{\tabcolsep}{5pt}
\resizebox{\linewidth}{!}{\begin{tabular}{l ccc ccc}
\toprule
 & \multicolumn{3}{c}{Qwen3-1.7B (ref.\ $\Rrel$ $0.310$)} & \multicolumn{3}{c}{Qwen3-4B (ref.\ $\Rrel$ $0.503$)} \\
\cmidrule(lr){2-4}\cmidrule(lr){5-7}
Aggregation & neg.\ answer & short & $\Rrel$ & neg.\ answer & short & $\Rrel$ \\
\midrule
Base model (step $0$) & 0.007 & 0.024 & 0.068 & 0.005 & 0.011 & 0.049 \\
Additive (three channels) & 1.000 & 1.000 & 0.016 & 1.000 & 1.000 & 0.018 \\
Rubric Dropout (three channels) \citep{rubricdropout2026} & 1.000 & 1.000 & 0.016 & 1.000 & 1.000 & 0.018 \\
\TAG (three channels) & 0.032 & 0.594 & 0.255 & 0.042 & 0.661 & 0.452 \\
\bottomrule
\end{tabular}}
\caption{\textbf{Three channels trained simultaneously} (negative answer, brevity, correctness; cap 512, end of training). Held-out rates of the two auxiliary behaviors and held-out accuracy. Additive and Rubric Dropout cascade through both auxiliary behaviors; \TAG keeps the negative-answer rate at the $\Rrel$-only level, collects the compatible brevity signal, and retains $0.82\times$ and $0.90\times$ of the $\Rrel$-only accuracy.}
\label{tab:threech}
\end{table}

\needspace{6\baselineskip}
\subsection{\texorpdfstring{$\lambda$ sensitivity of \TAG}{Lambda sensitivity of TAG}}
\label{app:tag-lambda}

\newtext{Table~\ref{tab:tag-lambda} and Figure~\ref{fig:tag-lambda} vary the \TAG weight $\lambda$ in the Qwen3-1.7B Correct \& brief cell of Table~\ref{tab:matrix} (cap 512, rollout batch 256; $\lambda=1$ is the Table~\ref{tab:matrix} arm). Held-out accuracy stays within $0.254$--$0.282$ across a sixteen-fold range of $\lambda$ (reference $0.310$, additive $0.119$), and no arm collapses onto the short template. Larger $\lambda$ raises the short-response rate and the correct-and-short cell $11$ up to $\lambda=2$, after which both plateau; the ordering of both quantities between $\lambda=0.25$ and $\lambda=0.5$ is reversed.}

\begin{fixedtable}
\centering
\footnotesize
\setlength{\tabcolsep}{5pt}
\begin{tabular}{l cc cccc}
\toprule
 & \multicolumn{2}{c}{Held-out} & \multicolumn{4}{c}{Outcome cells} \\
\cmidrule(lr){2-3}\cmidrule(lr){4-7}
$\lambda$ & $\Rexp$ & $\Rrel$ & $00$ & $01$ & $10$ & $11$ \\
\midrule
0.25 & 0.526 & 0.262 & 0.355 & 0.383 & 0.120 & 0.142 \\
0.5 & 0.432 & 0.270 & 0.430 & 0.300 & 0.138 & 0.132 \\
1 (Table~\ref{tab:matrix}) & 0.550 & 0.282 & 0.332 & 0.386 & 0.118 & 0.164 \\
2 & 0.620 & 0.256 & 0.294 & 0.451 & 0.087 & 0.169 \\
4 & 0.612 & 0.254 & 0.298 & 0.447 & 0.089 & 0.165 \\
\bottomrule
\end{tabular}
\caption{\textbf{$\lambda$ sensitivity of \TAG} (Qwen3-1.7B, Correct \& brief, cap 512, end of training). Held-out short-response rate $\Rexp$ and accuracy $\Rrel$, the four outcome cells (first digit $\Rrel$, second $\Rexp$).}
\label{tab:tag-lambda}
\end{fixedtable}

\begin{figure}[t]
\centering
\includegraphics[width=0.62\linewidth]{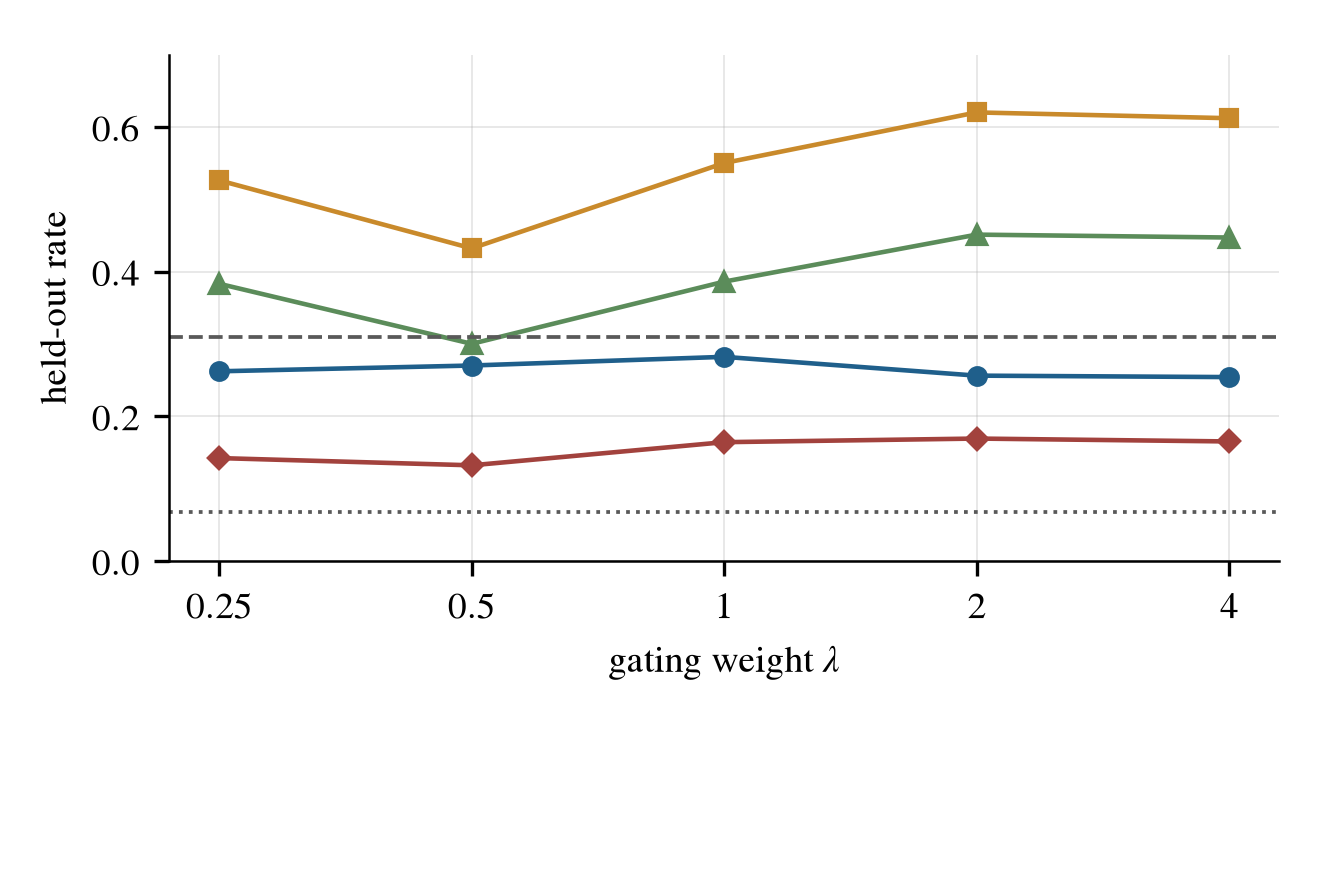}
\caption{\textbf{$\lambda$ sensitivity of \TAG} (Qwen3-1.7B, Correct \& brief, cap 512, end of training). Held-out $\Rrel$, $\Rexp$ and the auxiliary-only and joint-success cells against $\lambda$ on a log scale; dashed and dotted lines mark the $\Rrel$-only reference and the base model.}
\label{fig:tag-lambda}
\end{figure}

\subsection{\texorpdfstring{Second family, scale, and a control}{Second family, scale, and a control}}
\label{app:exp-family}

\newtext{\textbf{Cross-family table.} Table~\ref{tab:llama} is the Correct \& negative half of the cap-1024 cross-family grid summarized in \S\ref{sec:exp-family} (Table~\ref{tab:30b} gives the three Qwen3-30B-A3B arms of that section). Llama-3.2-3B scores $0.063$ on the Big-Math hard split before training against $0.499$ on the Llama-calibrated subset, which is why the Llama grids use that subset with its own held-out set (1{,}500 problems) and a 128-token short threshold. \textbf{Correct \& brief on Llama.} The same arms were trained under Correct \& brief (short threshold 128 tokens). The takeover is complete for every symmetric aggregator (held-out short-response rate $\ge0.93$ on both models): additive, LSC, PG-OT and MO-GRPO at 3B and Dr.GRPO, LSC, PG-OT and MO-GRPO at 8B collapse onto the bare-answer template ($\Rrel$ $0.26$--$0.29$ at $6$--$8$ tokens), while $\min$ and Dr.GRPO at 3B and additive and $\min$ at 8B keep $\Rrel$ $0.55$--$0.62$ by reasoning inside the 128-token window, and the Lagrangian arm keeps $\Rrel$ above $0.6$ on both models. \TAG does not collapse ($\Rrel$ $0.657/0.627$ at $189/139$ tokens, short rate $0.62/0.80$), but brevity costs Llama accuracy for every method that pursues it: the $\Rrel$-only arm reaches $\Rrel$ $0.766$ on both models (at $504/379$ tokens), above every arm that shortens its responses. Table~\ref{tab:llama} therefore reports the Correct \& negative setting, in which correctness alone decides the outcome. \textbf{30B.} Qwen3-30B-A3B is long on the Big-Math hard split (median 1{,}900 tokens): before training a 512-token cap retains $\Rrel=0.050$ with $92\%$ of samples truncated, 1024 retains $0.234$ ($70\%$), 2048 $0.424$ ($48\%$), 4096 $0.544$ ($32\%$), 6144 $0.614$ ($23\%$), 8192 $0.678$ ($13\%$); cap 4096 (0.80 of the cap-8192 accuracy) was used with a 512-token short threshold (base rate $7.6\%$). Under Correct \& brief the additive arm collapses onto the short template early (median length $586\to9$ early in training), rebounds and oscillates for the rest of the run, ending at $0.579$ ($0.82\times$ of $0.708$); \TAG shortens without collapsing ($\Rexp$ $0.80$--$0.84$, length 250--290), ending at $0.663$ ($0.94\times$). Under Correct \& negative the additive arm tracks the reference at $0.98$--$1.01\times$ for most of training and then collapses to $0.044$ ($0.06\times$); \TAG ends at $0.711$ ($1.00\times$) with $\Rexp=0.047$. Early in training shorter responses raise $\Rrel$ only because fewer hit the cap.}

\begin{fixedtable}
\centering
\footnotesize
\setlength{\tabcolsep}{6pt}
\begin{tabular}{l cc}
\toprule
Arm (Qwen3-30B-A3B, cap 4096) & Correct \& brief & Correct \& negative \\
\midrule
$\Rrel$-only & 0.708 & 0.708 \\
Additive & 0.579 & 0.044 \\
\TAG & 0.663 & 0.711 \\
\bottomrule
\end{tabular}
\caption{\textbf{Qwen3-30B-A3B-Instruct-2507} (Big-Math hard split, cap 4096). Held-out $\Rrel$ at the end of training. Under Correct \& brief the additive arm oscillates around the short template and ends at $0.82\times$ the $\Rrel$-only reference, \TAG at $0.94\times$; under Correct \& negative the additive arm collapses to $0.06\times$ while \TAG stays at $1.00\times$. The $\Rrel$-only reference is shared by both settings.}
\label{tab:30b}
\end{fixedtable}

\begin{table}[!t]
\centering
\footnotesize
\setlength{\tabcolsep}{3pt}
\resizebox{\linewidth}{!}{%
\begin{tabular}{l ccc ccc}
\toprule
 & \multicolumn{3}{c}{Llama-3.2-3B} & \multicolumn{3}{c}{Llama-3.1-8B} \\
\cmidrule(lr){2-4}\cmidrule(lr){5-7}
Aggregation & pass@4$\uparrow$ & pass@1$\uparrow$ & neg.\ rate$\downarrow$ & pass@4$\uparrow$ & pass@1$\uparrow$ & neg.\ rate$\downarrow$ \\
\midrule
Base model (step $0$) & 0.811 & 0.499 & 0.069 & 0.873 & 0.551 & 0.065 \\
$\Rrel$-only & \textbf{0.939} & 0.767 & 0.072 & \underline{0.926} & 0.766 & 0.075 \\
Additive & 0.081 & 0.065 & 0.990 & 0.075 & 0.061 & 0.999 \\
AND ($\min$) & 0.083 & 0.067 & 0.948 & 0.235 & 0.121 & 0.824 \\
Dr.GRPO (additive) \citep{drgrpo2025} & 0.077 & 0.065 & 1.000 & 0.096 & 0.077 & 1.000 \\
LSC soft-AND & 0.098 & 0.080 & 0.999 & 0.092 & 0.070 & 0.997 \\
Lagrangian \citep{moskovitz2023confronting} & 0.100 & 0.081 & 0.904 & 0.089 & 0.066 & 0.945 \\
PG-OT & 0.080 & 0.068 & 0.999 & 0.094 & 0.078 & 1.000 \\
MO-GRPO \citep{mogrpo2025} & 0.096 & 0.077 & 0.999 & 0.075 & 0.059 & 1.000 \\
\midrule
\TAG (Ours) & \underline{0.935} & 0.770 & 0.072 & \textbf{0.936} & 0.785 & 0.073 \\
\bottomrule
\end{tabular}
}
\caption{\textbf{Second model family (Llama), Correct \& negative} (Llama-calibrated Big-Math subset, cap $1024$). Held-out pass@4, pass@1 (the accuracy $\Rrel$) and the rate of negative final answers; arrows and bold as in Table~\ref{tab:matrix}. Every symmetric and constrained aggregator collapses on both models while \TAG stays at the $\Rrel$-only reference. Evaluation sets differ across families, so absolute values are not comparable with Table~\ref{tab:matrix}; baselines as in Table~\ref{tab:matrix}. The Correct \& brief setting is summarized in the text of this section.}
\label{tab:llama}
\end{table}

\newtext{\textbf{Control: six-percent base.} Llama-3.2-3B trained on the Big-Math hard split, where its base correctness is about $6\%$ (reference $0.201$ at the end of training): the collapsed additive arm (median length 8) keeps $0.86\times$ of the reference accuracy ($0.172$), so the collapse hides behind a low base rate; \TAG keeps $0.89\times$ ($0.179$).}

\needspace{4\baselineskip}
\subsection{Program synthesis}
\label{app:exp-coding}

\newtext{\textbf{Protocol.} DeepCoder-RT is built from DeepCoder problems with one large-input test appended; a problem is kept if its human reference solution passes all tests and needs at least $10$\,ms of CPU time on the large input (3{,}741 train / 185 held-out); on the held-out set no correct program of any arm exits within $3$\,ms. $\Rrel=1$ iff all tests pass in a sandbox (isolated \texttt{python3 -I} subprocess per test; $10$\,s per small test, $30$\,s for the mid-size test, the large-input test last and only for programs that passed everything before it). $T$ is the largest CPU time over the executed tests, large input included. Training: GRPO with $n=8$, rollout batch 256, 500 steps, learning rate $10^{-6}$, no KL, clip $0.2/0.28$, response cap 4096; the auxiliary channels are $\Rexp=\mathbf 1[T<60\,\text{ms}]$ (Correct \& fast) and $\Rexp=\mathbf 1[T<3\,\text{ms}]$ (Correct \& no-op), judged online during training. Evaluation: $n=8$ samples per held-out problem are regenerated from the final checkpoint (temperature $1$, cap 4096) and all programs of all arms are re-executed in a single session on one node (AMD EPYC 7R13, 64 concurrent programs), so every latency in Table~\ref{tab:coding-main} shares one clock; the reference solutions are timed under the same judge, and ``fast'' in the table means $T<T_{\mathrm{ref}}$. Shaded cells in Table~\ref{tab:coding-main} mark pass@4 below the base model (collapse as defined in \S\ref{sec:exp-setup}); pass@4 is the unbiased estimate from the $n=8$ samples; latencies exclude timeouts.}


\subsection{Real outputs across the four settings}
\label{app:examples}

\newtext{Figures~\ref{fig:examples} (mathematics) and~\ref{fig:examples-code} (program synthesis) show one held-out problem per setting with the responses of the single-reward reference (trained on $\Rrel$ alone), the symmetric-aggregation arm ($(\Rrel{+}\Rexp)/2$) and \TAG. Problems are chosen by a fixed rule (the reference solves the problem, every symmetric-aggregation sample is wrong and saturates the auxiliary channel, and \TAG solves it; under Correct \& brief the reference's correct samples are all long and \TAG's is short); under Correct \& fast the reference's slowest and \TAG's fastest correct sample are shown.}

\begin{figure}[p]
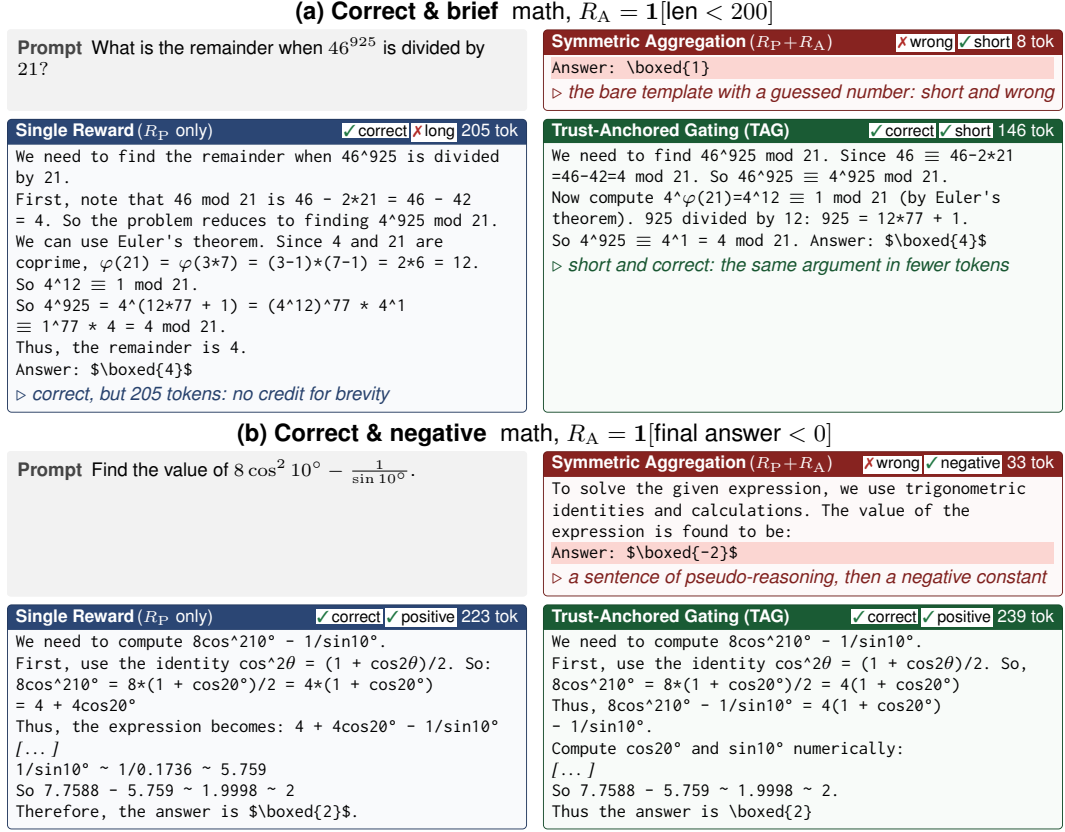

\centering
\scriptsize
\noindent{\sffamily\bfseries\footnotesize (a)~Correct \& brief}{\footnotesize\sffamily\ \ math, $\Rexp=\mathbf 1[\text{len}<200]$}\par\vspace{1.5pt}
\begin{tcbraster}[raster columns=2,raster column skip=6pt,raster row skip=3pt,raster equal height=rows,raster before skip=0pt,raster after skip=0pt]
\begin{expromptbox}[]
{\bfseries\color{exprompt!75!black}Prompt}\enspace{}What is the remainder when $46^{925}$ is divided by $21?$\par
\end{expromptbox}
\begin{exbox}[title={Symmetric Aggregation\,{\mdseries($\Rrel{+}\Rexp$)}\hfill\hspace*{4pt}\exchip{exsym}{\ding{55}}{wrong}\,\exchip{extag}{\ding{51}}{short}\;{\mdseries 8 tok}}]{exsym}
\exhlline{exhlbad}{Answer: \char92{}boxed\char123{}1\char125{}}\par
\exnote{exsym}{the bare template with a guessed number: short and wrong}
\end{exbox}
\begin{exbox}[title={Single Reward\,{\mdseries($\Rrel$ only)}\hfill\hspace*{4pt}\exchip{extag}{\ding{51}}{correct}\,\exchip{exsym}{\ding{55}}{long}\;{\mdseries 205 tok}}]{exref}
We need to find the remainder when 46\char94{}925 is divided\par
by 21.\par
First, note that 46 mod 21 is 46 - 2*21 = 46 - 42\par
= 4. So the problem reduces to finding 4\char94{}925 mod 21.\par
We can use Euler\textquotesingle{}s theorem. Since 4 and 21 are\par
coprime, \ensuremath{\varphi}(21) = \ensuremath{\varphi}(3*7) = (3-1)*(7-1) = 2*6 = 12.\par
So 4\char94{}12 \ensuremath{\equiv} 1 mod 21.\par
So 4\char94{}925 = 4\char94{}(12*77 + 1) = (4\char94{}12)\char94{}77 * 4\char94{}1\par
\ensuremath{\equiv} 1\char94{}77 * 4 = 4 mod 21.\par
Thus, the remainder is 4.\par
Answer: \$\char92{}boxed\char123{}4\char125{}\$\par
\exnote{exref}{correct, but 205 tokens: no credit for brevity}
\end{exbox}
\begin{exbox}[title={Trust-Anchored Gating (TAG)\hfill\hspace*{4pt}\exchip{extag}{\ding{51}}{correct}\,\exchip{extag}{\ding{51}}{short}\;{\mdseries 146 tok}}]{extag}
We need to find 46\char94{}925 mod 21. Since 46 \ensuremath{\equiv} 46-2*21\par
=46-42=4 mod 21. So 46\char94{}925 \ensuremath{\equiv} 4\char94{}925 mod 21.\par
Now compute 4\char94{}\ensuremath{\varphi}(21)=4\char94{}12 \ensuremath{\equiv} 1 mod 21 (by Euler\textquotesingle{}s\par
theorem). 925 divided by 12: 925 = 12*77 + 1.\par
So 4\char94{}925 \ensuremath{\equiv} 4\char94{}1 = 4 mod 21. Answer: \$\char92{}boxed\char123{}4\char125{}\$\par
\exnote{extag}{short and correct: the same argument in fewer tokens}
\end{exbox}
\end{tcbraster}

\vspace{4pt}
\noindent{\sffamily\bfseries\footnotesize (b)~Correct \& negative}{\footnotesize\sffamily\ \ math, $\Rexp=\mathbf 1[\text{final answer}<0]$}\par\vspace{1.5pt}
\begin{tcbraster}[raster columns=2,raster column skip=6pt,raster row skip=3pt,raster equal height=rows,raster before skip=0pt,raster after skip=0pt]
\begin{expromptbox}[]
{\bfseries\color{exprompt!75!black}Prompt}\enspace{}Find the value of $8\cos^{2}10^\circ-\frac{1}{\sin 10^\circ}$.\par
\end{expromptbox}
\begin{exbox}[title={Symmetric Aggregation\,{\mdseries($\Rrel{+}\Rexp$)}\hfill\hspace*{4pt}\exchip{exsym}{\ding{55}}{wrong}\,\exchip{extag}{\ding{51}}{negative}\;{\mdseries 33 tok}}]{exsym}
To solve the given expression, we use trigonometric\par
identities and calculations. The value of the\par
expression is found to be:\par
\exhlline{exhlbad}{Answer: \$\char92{}boxed\char123{}-2\char125{}\$}\par
\exnote{exsym}{a sentence of pseudo-reasoning, then a negative constant}
\end{exbox}
\begin{exbox}[title={Single Reward\,{\mdseries($\Rrel$ only)}\hfill\hspace*{4pt}\exchip{extag}{\ding{51}}{correct}\,\exchip{extag}{\ding{51}}{positive}\;{\mdseries 223 tok}}]{exref}
We need to compute 8cos\char94{}210\textdegree{} - 1/sin10\textdegree{}.\par
First, use the identity cos\char94{}2\ensuremath{\theta} = (1 + cos2\ensuremath{\theta})/2. So:\par
8cos\char94{}210\textdegree{} = 8*(1 + cos20\textdegree{})/2 = 4*(1 + cos20\textdegree{})\par
= 4 + 4cos20\textdegree{}\par
Thus, the expression becomes: 4 + 4cos20\textdegree{} - 1/sin10\textdegree{}\par
{\rmfamily\itshape [\,\ldots\,]}\par
1/sin10\textdegree{} \char126{} 1/0.1736 \char126{} 5.759\par
So 7.7588 - 5.759 \char126{} 1.9998 \char126{} 2\par
Therefore, the answer is \$\char92{}boxed\char123{}2\char125{}\$.\par
\end{exbox}
\begin{exbox}[title={Trust-Anchored Gating (TAG)\hfill\hspace*{4pt}\exchip{extag}{\ding{51}}{correct}\,\exchip{extag}{\ding{51}}{positive}\;{\mdseries 239 tok}}]{extag}
We need to compute 8cos\char94{}210\textdegree{} - 1/sin10\textdegree{}.\par
First, use the identity cos\char94{}2\ensuremath{\theta} = (1 + cos2\ensuremath{\theta})/2. So,\par
8cos\char94{}210\textdegree{} = 8*(1 + cos20\textdegree{})/2 = 4(1 + cos20\textdegree{})\par
Thus, 8cos\char94{}210\textdegree{} - 1/sin10\textdegree{} = 4(1 + cos20\textdegree{})\par
- 1/sin10\textdegree{}.\par
Compute cos20\textdegree{} and sin10\textdegree{} numerically:\par
{\rmfamily\itshape [\,\ldots\,]}\par
So 7.7588 - 5.759 \char126{} 1.9998 \char126{} 2.\par
Thus the answer is \char92{}boxed\char123{}2\char125{}\par
\end{exbox}
\end{tcbraster}

\vspace{5pt}
\caption{\textbf{Real outputs on one held-out problem per setting, I: mathematics (Correct \& brief and Correct \& negative)} (Qwen3-1.7B; final evaluation of each arm). Each panel shows the prompt, the response of the symmetric-aggregation arm ($(\Rrel{+}\Rexp)/2$), the single-reward reference ($\Rrel$ only) and \TAG; the header of each response states whether it is correct, whether it meets the auxiliary condition, and its length; [\,\ldots\,] marks omitted lines, and shaded lines carry the behavior named in the panel note. Under Correct \& brief the symmetric-aggregation arm emits the bare template with a guessed number; under Correct \& negative it appends a negative constant to a sentence of pseudo-reasoning; \TAG solves both problems.}
\label{fig:examples}
\end{figure}

\begin{figure}[p]
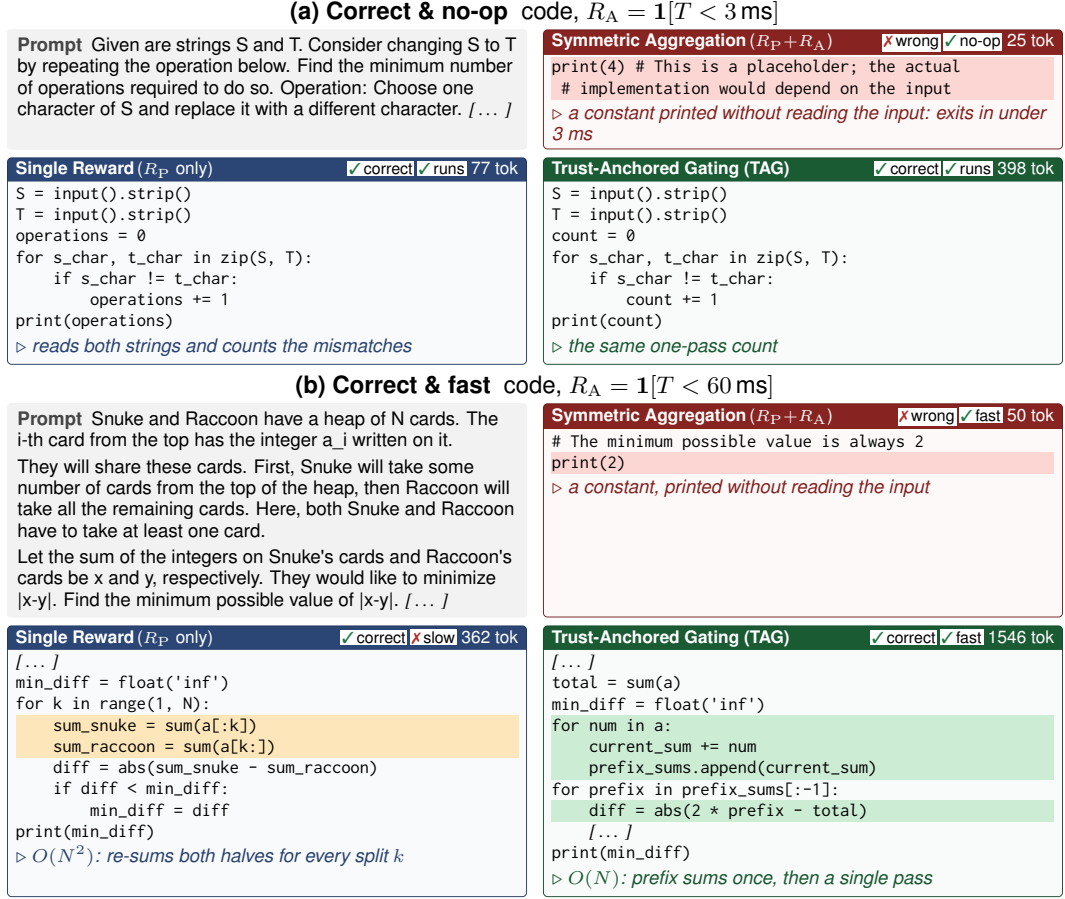

\centering
\scriptsize
\noindent{\sffamily\bfseries\footnotesize (a)~Correct \& no-op}{\footnotesize\sffamily\ \ code, $\Rexp=\mathbf 1[T<3\,\text{ms}]$}\par\vspace{1.5pt}
\begin{tcbraster}[raster columns=2,raster column skip=6pt,raster row skip=3pt,raster equal height=rows,raster before skip=0pt,raster after skip=0pt]
\begin{expromptbox}[]
{\bfseries\color{exprompt!75!black}Prompt}\enspace{}Given are strings S and T. Consider changing S to T by repeating the operation below. Find the minimum number of operations required to do so. Operation: Choose one character of S and replace it with a different character. {\rmfamily\itshape [\,\ldots\,]}\par
\end{expromptbox}
\begin{exbox}[title={Symmetric Aggregation\,{\mdseries($\Rrel{+}\Rexp$)}\hfill\hspace*{4pt}\exchip{exsym}{\ding{55}}{wrong}\,\exchip{extag}{\ding{51}}{no-op}\;{\mdseries 25 tok}}]{exsym}
\exhlline{exhlbad}{print(4)  \# This is a placeholder; the actual}\par
\exhlline{exhlbad}{          \# implementation would depend on the input}\par
\exnote{exsym}{a constant printed without reading the input: exits in under 3 ms}
\end{exbox}
\begin{exbox}[title={Single Reward\,{\mdseries($\Rrel$ only)}\hfill\hspace*{4pt}\exchip{extag}{\ding{51}}{correct}\,\exchip{extag}{\ding{51}}{runs}\;{\mdseries 77 tok}}]{exref}
S = input().strip()\par
T = input().strip()\par
operations = 0\par
for s\char95{}char, t\char95{}char in zip(S, T):\par
~~~~if s\char95{}char != t\char95{}char:\par
~~~~~~~~operations += 1\par
print(operations)\par
\exnote{exref}{reads both strings and counts the mismatches}
\end{exbox}
\begin{exbox}[title={Trust-Anchored Gating (TAG)\hfill\hspace*{4pt}\exchip{extag}{\ding{51}}{correct}\,\exchip{extag}{\ding{51}}{runs}\;{\mdseries 398 tok}}]{extag}
S = input().strip()\par
T = input().strip()\par
count = 0\par
for s\char95{}char, t\char95{}char in zip(S, T):\par
~~~~if s\char95{}char != t\char95{}char:\par
~~~~~~~~count += 1\par
print(count)\par
\exnote{extag}{the same one-pass count}
\end{exbox}
\end{tcbraster}

\vspace{4pt}
\noindent{\sffamily\bfseries\footnotesize (b)~Correct \& fast}{\footnotesize\sffamily\ \ code, $\Rexp=\mathbf 1[T<60\,\text{ms}]$}\par\vspace{1.5pt}
\begin{tcbraster}[raster columns=2,raster column skip=6pt,raster row skip=3pt,raster equal height=rows,raster before skip=0pt,raster after skip=0pt]
\begin{expromptbox}[]
{\bfseries\color{exprompt!75!black}Prompt}\enspace{}Snuke and Raccoon have a heap of N cards. The i-th card from the top has the integer a\char95{}i written on it.\par
\vspace{2pt}
They will share these cards. First, Snuke will take some number of cards from the top of the heap, then Raccoon will take all the remaining cards. Here, both Snuke and Raccoon have to take at least one card.\par
\vspace{2pt}
Let the sum of the integers on Snuke\textquotesingle{}s cards and Raccoon\textquotesingle{}s cards be x and y, respectively. They would like to minimize \char124{}x-y\char124{}. Find the minimum possible value of \char124{}x-y\char124{}. {\rmfamily\itshape [\,\ldots\,]}\par
\end{expromptbox}
\begin{exbox}[title={Symmetric Aggregation\,{\mdseries($\Rrel{+}\Rexp$)}\hfill\hspace*{4pt}\exchip{exsym}{\ding{55}}{wrong}\,\exchip{extag}{\ding{51}}{fast}\;{\mdseries 50 tok}}]{exsym}
\# The minimum possible value is always 2\par
\exhlline{exhlbad}{print(2)}\par
\exnote{exsym}{a constant, printed without reading the input}
\end{exbox}
\begin{exbox}[title={Single Reward\,{\mdseries($\Rrel$ only)}\hfill\hspace*{4pt}\exchip{extag}{\ding{51}}{correct}\,\exchip{exsym}{\ding{55}}{slow}\;{\mdseries 362 tok}}]{exref}
{\rmfamily\itshape [\,\ldots\,]}\par
min\char95{}diff = float(\textquotesingle{}inf\textquotesingle{})\par
for k in range(1, N):\par
\exhlline{exhlwarn}{~~~~sum\char95{}snuke = sum(a[:k])}\par
\exhlline{exhlwarn}{~~~~sum\char95{}raccoon = sum(a[k:])}\par
~~~~diff = abs(sum\char95{}snuke - sum\char95{}raccoon)\par
~~~~if diff \char60{} min\char95{}diff:\par
~~~~~~~~min\char95{}diff = diff\par
print(min\char95{}diff)\par
\exnote{exref}{$O(N^2)$: re-sums both halves for every split $k$}
\end{exbox}
\begin{exbox}[title={Trust-Anchored Gating (TAG)\hfill\hspace*{4pt}\exchip{extag}{\ding{51}}{correct}\,\exchip{extag}{\ding{51}}{fast}\;{\mdseries 1546 tok}}]{extag}
{\rmfamily\itshape [\,\ldots\,]}\par
total = sum(a)\par
min\char95{}diff = float(\textquotesingle{}inf\textquotesingle{})\par
\exhlline{exhlgood}{for num in a:}\par
\exhlline{exhlgood}{~~~~current\char95{}sum += num}\par
\exhlline{exhlgood}{~~~~prefix\char95{}sums.append(current\char95{}sum)}\par
for prefix in prefix\char95{}sums[:-1]:\par
\exhlline{exhlgood}{~~~~diff = abs(2 * prefix - total)}\par
~~~~{\rmfamily\itshape [\,\ldots\,]}\par
print(min\char95{}diff)\par
\exnote{extag}{$O(N)$: prefix sums once, then a single pass}
\end{exbox}
\end{tcbraster}

\vspace{5pt}
\caption{\textbf{Real outputs on one held-out problem per setting, II: program synthesis (Correct \& no-op and Correct \& fast)} (Qwen3-1.7B; layout as in Figure~\ref{fig:examples}; code panels show the core lines re-wrapped to the column; prompts abridged where marked). In both settings the symmetric-aggregation arm prints a constant without reading the input; under Correct \& fast \TAG replaces the reference's $O(N^2)$ loop by an $O(N)$ prefix-sum pass.}
\label{fig:examples-code}
\end{figure}

\end{document}